\documentclass[11pt]{article} 
\usepackage[utf8]{inputenc} 

\usepackage[margin=1.1in]{geometry} 
\usepackage{graphicx} 

\usepackage{booktabs} 
\usepackage{array} 
\usepackage{paralist} 
\usepackage{verbatim} 
\usepackage{mathrsfs}
\usepackage{amssymb}
\usepackage{amsthm}
\usepackage{amsmath,amsfonts,amssymb}
\usepackage{esint}
\usepackage{graphics}
\usepackage{enumerate}
\usepackage{mathtools}
\usepackage{xfrac}
\usepackage{subcaption}
\usepackage{stmaryrd}
 \usepackage{mathabx}
\usepackage{algorithm}
\usepackage{algpseudocode}

\usepackage{float} 

\newfloat{routine}{tbp}{lor}
\floatname{routine}{Routine}

\usepackage[dvipsnames]{xcolor}
\usepackage[colorlinks=true, pdfstartview=FitV, linkcolor=blue, citecolor=blue, urlcolor=blue]{hyperref}
\usepackage[normalem]{ulem}

\usepackage{tikz}
\usetikzlibrary{calc}
\usepackage{pgf}
\usetikzlibrary{external}
\numberwithin{equation}{section}
\numberwithin{figure}{section}

\newtheorem{theorem}{Theorem}[section]

\newtheorem{assumption}[theorem]{Assumption}
\newtheorem{corollary}[theorem]{Corollary}
\newtheorem{proposition}[theorem]{Proposition}
\newtheorem{lemma}[theorem]{Lemma}
\theoremstyle{definition}
\newtheorem{definition}[theorem]{Definition}

\newtheorem{remark}[theorem]{Remark}

\newtheorem{condition}[theorem]{Condition}

\newcommand*{\supp}{\ensuremath{\mathrm{supp\,}}}

\newcommand*{\R}{\ensuremath{\mathbb{R}}}

\newcommand{\eps}{\varepsilon}

\renewcommand*{\tilde}{\widetilde}

\renewcommand{\S}{\mathbf{S}}
\newcommand{\Pc}{\mathcal{P}}

\DeclareMathOperator{\MMD}{MMD}

\usepackage{accents}

\newcommand{\ux}{X}
\newcommand{\uy}{Y}

\def\Xint#1{\mathchoice
{\XXint\displaystyle\textstyle{#1}}%
{\XXint\textstyle\scriptstyle{#1}}%
{\XXint\scriptstyle\scriptscriptstyle{#1}}%
{\XXint\scriptscriptstyle\scriptscriptstyle{#1}}%
\!\int}
\def\XXint#1#2#3{{\setbox0=\hbox{$#1{#2#3}{\int}$}
\vcenter{\hbox{$#2#3$}}\kern-.5\wd0}}

\def\fint{\Xint-}

\let\originalleft\left
\let\originalright\right
\renewcommand{\left}{\mathopen{}\mathclose\bgroup\originalleft}
\renewcommand{\right}{\aftergroup\egroup\originalright}

\usepackage[hang,flushmargin]{footmisc}
\newcommand{\emptyfootnote}[1]{%
    \renewcommand{\thefootnote}{}%
    \footnotetext{#1}%
    \renewcommand{\thefootnote}{\arabic{footnote}}%
}

\newcommand{\E}{\mathbb{E}}

\newcommand{\RD}{\mathbb{D}}
\renewcommand{\S}{{{\mathbb{S}}}}
\newcommand{\diff}{\mathrm{diff}}

\renewcommand{\hat}{\widehat}

\makeatletter
\pgfmathdeclarefunction{erf}{1}{%
  \begingroup
    \pgfmathparse{#1 > 0 ? 1 : -1}%
    \edef\sign{\pgfmathresult}%
    \pgfmathparse{abs(#1)}%
    \edef\x{\pgfmathresult}%
    \pgfmathparse{1/(1+0.3275911*\x)}%
    \edef\t{\pgfmathresult}%
    \pgfmathparse{%
      1 - (((((1.061405429*\t -1.453152027)*\t) + 1.421413741)*\t 
      -0.284496736)*\t + 0.254829592)*\t*exp(-(\x*\x))}%
    \edef\y{\pgfmathresult}%
    \pgfmathparse{(\sign)*\y}%
    \pgfmath@smuggleone\pgfmathresult%
  \endgroup
}

\newenvironment{keywords}{%
  \begin{quote}\noindent
  \textbf{Keywords: }\ignorespaces
  \hangindent=3.2em\hangafter=1 
}{%
  \end{quote}
}

\newenvironment{MSC}{%
  \begin{quote}\noindent
  \textbf{AMS MCS 2020 Classes: }\ignorespaces
  \hangindent=3.2em\hangafter=1 
}{%
  \end{quote}
}

\usepackage{titlesec}

\makeatletter
\newif\ifinappendix
\let\oldappendix\appendix
\renewcommand{\appendix}{\oldappendix\inappendixtrue}

\renewcommand{\@seccntformat}[1]{%
  \ifinappendix
    \ifnum\pdfstrcmp{#1}{section}=0
      Appendix~\thesection\quad
    \else
      \csname the#1\endcsname\quad
    \fi
  \else
    \csname the#1\endcsname\quad
  \fi
}

\makeatother

\title{Radon--Wasserstein Gradient Flows for Interacting-Particle Sampling in High Dimensions}

\author{
Elias Hess-Childs
\thanks{Department of Mathematical Sciences, Carnegie Mellon University
{\footnotesize \href{mailto:ehesschi@andrew.cmu.edu}{ehesschi@andrew.cmu.edu}.}
}
\and 
Dejan Slep\v{c}ev
\thanks{Department of Mathematical Sciences, Carnegie Mellon University
{\footnotesize \href{mailto:slepcev@andrew.cmu.edu}{slepcev@andrew.cmu.edu}
}
}
\and
Lantian Xu
\thanks{Department of Mathematical Sciences, Carnegie Mellon University
{\footnotesize \href{mailto:lxu2@alumni.cmu.edu}{lxu2@alumni.cmu.edu}
}
}
}
\date{ }

\usepackage[nottoc,notlot,notlof]{tocbibind}

\makeatletter
\renewcommand{\@oddhead}{\hfil Radon--Wasserstein Gradient Flows \hfil}
\makeatother

\begin{document}

\maketitle

\begin{abstract}
Gradient flows of the Kullback–Leibler (KL) divergence, such as the Fokker–Planck equation and Stein Variational Gradient Descent, evolve a distribution toward a target density known only up to a normalizing constant.\ We introduce new gradient flows of the KL divergence with a remarkable combination of properties: they admit accurate interacting-particle approximations in high dimensions, and the per-step cost scales linearly in both the number of particles and the dimension.\ These gradient flows are based on new transportation-based Riemannian geometries on the space of probability measures: the Radon–Wasserstein geometry and the related Regularized Radon–Wasserstein (RRW) geometry.\ We define these geometries using the Radon transform so that the gradient-flow velocities depend only on one-dimensional projections.\ This yields interacting-particle-based algorithms whose per-step cost follows from efficient Fast Fourier Transform-based evaluation of the required 1D convolutions. We additionally provide numerical experiments that study the performance of the proposed algorithms and compare convergence behavior and quantization. Finally, we prove some theoretical results including well-posedness of the flows and long-time convergence guarantees for the RRW flow. 
\end{abstract}


\begin{MSC}
 65M75, 35Q62, 35Q68, 35Q70, 62-08, 82C22, 49Q22
\end{MSC}


\section{Introduction}

\emptyfootnote{\quad\:\, Companion code: \url{https://github.com/slepcev/Radon-Wasserstein-Gradient-Flow}}%

Sampling high-dimensional distributions is an essential task in the sciences \cite{liu08MC,Leimkuhler15MD,rubinstein2016simulation}, Bayesian inference \cite{Gelman14Bayes,Stuart10Bayes}, and other domains \cite{brooks2011handbook,chib2001markov,luengo2020survey,Rasmussen06GP}. The most widely used approaches for sampling in high dimensions are based on Markov processes with invariant measure equal to the target, 
such as Metropolis--Hastings (MH), Langevin Monte Carlo (LMC),  Hamiltonian Monte Carlo (HMC), and variants such as the Metropolis Adjusted Langevin Algorithm (MALA).

While these approaches are applicable to high-dimensional problems, for complex energy landscapes the convergence can be very slow and may be difficult or impossible to ensure in practice. In recent years alternative approaches that take into account global information have been developed. Variational  Inference~\cite{blei2017variational,zhang2018advances,dhaka2021challenges,graves2011practical,lambert2022variational} refers to approaches in which approximating the target measure is achieved by minimizing a functional over a parameterized family of measures, such as Gaussians or product measures. Variational Inference typically offers faster convergence than MCMC methods, at the cost of introducing bias due to the restriction to a parameterized family.

Another line of approaches that takes global information into account is based on interacting particles. Some of these approaches are extensions of MCMC that benefit from taking a global point of view: for example, the affine-invariant ensemble sampler of~\cite{GoodmanWeare10} adapts to the geometry of the landscape, birth–death methods~\cite{lindsey2022ensemble,lu2019accelerating} enable nonlocal transport to overcome multimodality, and annealing and Sequential Monte Carlo (SMC)~\cite{Chopin20} address multimodality by introducing families of interpolating measures. Other interacting-particle approaches discretize flows that evolve toward the target distribution. One example of such a flow is the Fokker--Planck equation---the gradient flow of the Kullback--Leibler (KL) divergence with respect to the Wasserstein geometry on the space of probability measures. To approximate this flow with interacting particles, authors have developed the ``blob method,'' which modifies the functional~\cite{carrillo2019blob}, as well as probability-flow methods that estimate the score (needed to compute the velocity) via score matching with deep neural networks~\cite{maoutsa20,boffi2023probability}. A further direction modifies the geometry so that the resulting KL gradient flow can be more directly approximated by interacting particles, potentially in high dimensions. In particular, Stein Variational Gradient Descent (SVGD) \cite{liu2016stein,liu2016kernelized,lu2019scaling} considers a more restrictive geometry than the Wasserstein geometry that penalizes velocities in a Reproducing Kernel Hilbert Space (RKHS). This allows the flow to be approximated by interacting particles whose positions satisfy an ODE system. A downside of SVGD is that the complexity of each step is $O(n^2 d)$ where $n$ is the number of particles and $d$ is the dimension. Additionally, convergence can be slow, as the kernels typically need to be wide, and the approximation of $d$-dimensional convolutional quantities from $n$ particles can degrade when $n$ is not at least comparable to $d$.

 The need to develop approaches whose complexity is subquadratic in $n$ and that apply to high-dimensional problems motivated us to introduce a new geometry on the space of probability measures, the Radon--Wasserstein (RW) geometry. This geometry only allows velocities that can be decomposed as an average over ``one-dimensional'' movements. Considering such``simple" velocities enables for accurate approximation in high dimensions. We describe the geometry precisely in Section~\ref{sec:RWintro}. Here we highlight the gradient flow of the KL divergence in the RW geometry, which is the basis for the algorithms we develop. Letting $R^\theta$ denote the Radon transform,
\[R^\theta f(p):=\int_{\{y^\theta\in\R^d:\ \theta\cdot y^\theta=0\}}f(p\theta+y^\theta)\,dy^\theta\qquad \theta\in\S^{d-1},\; p\in\R,\]
the Radon--Wasserstein gradient flow of the KL divergence with target measure $\pi\,\propto\, e^{-U}$ is given by the continuity equation 
\[\partial_t\rho_t + \nabla\cdot(\rho_t v) = 0,
     \qquad v(t,x) = -\fint_{\S^{d-1}} \theta \left(\frac{ \frac{\partial}{\partial p} ( R^\theta\rho_t) + R^\theta(\rho_t\nabla_\theta U)}{R^\theta\rho_t}\right)(x\cdot\theta)\,d\theta,\]
where $p$ is the one-dimensional variable in the Radon transform domain and $\nabla_\theta$ denotes the directional derivative. We also introduce a variant of this geometry, the Regularized Radon--Wasserstein (RRW) geometry, and derive the corresponding KL gradient flow (see Proposition~\ref{prop:RRWgf} and Figure~\ref{fig:trajectories_plot}).

The two key properties of these approaches---namely, that the velocities can be well approximated in high dimensions and that each step of our scheme can be performed in $O(nd)$ operations---both stem from the fact that velocities can be computed from one-dimensional projections of the evolving distribution. To build these methods and justify them, we analyze the underlying equations, introduce interacting-particle discretizations, design projection-based algorithms, and prove convergence. Below we provide a detailed outline of our work.

\begin{figure}[t]
\centering
\includegraphics[width=0.72\textwidth,
  clip]{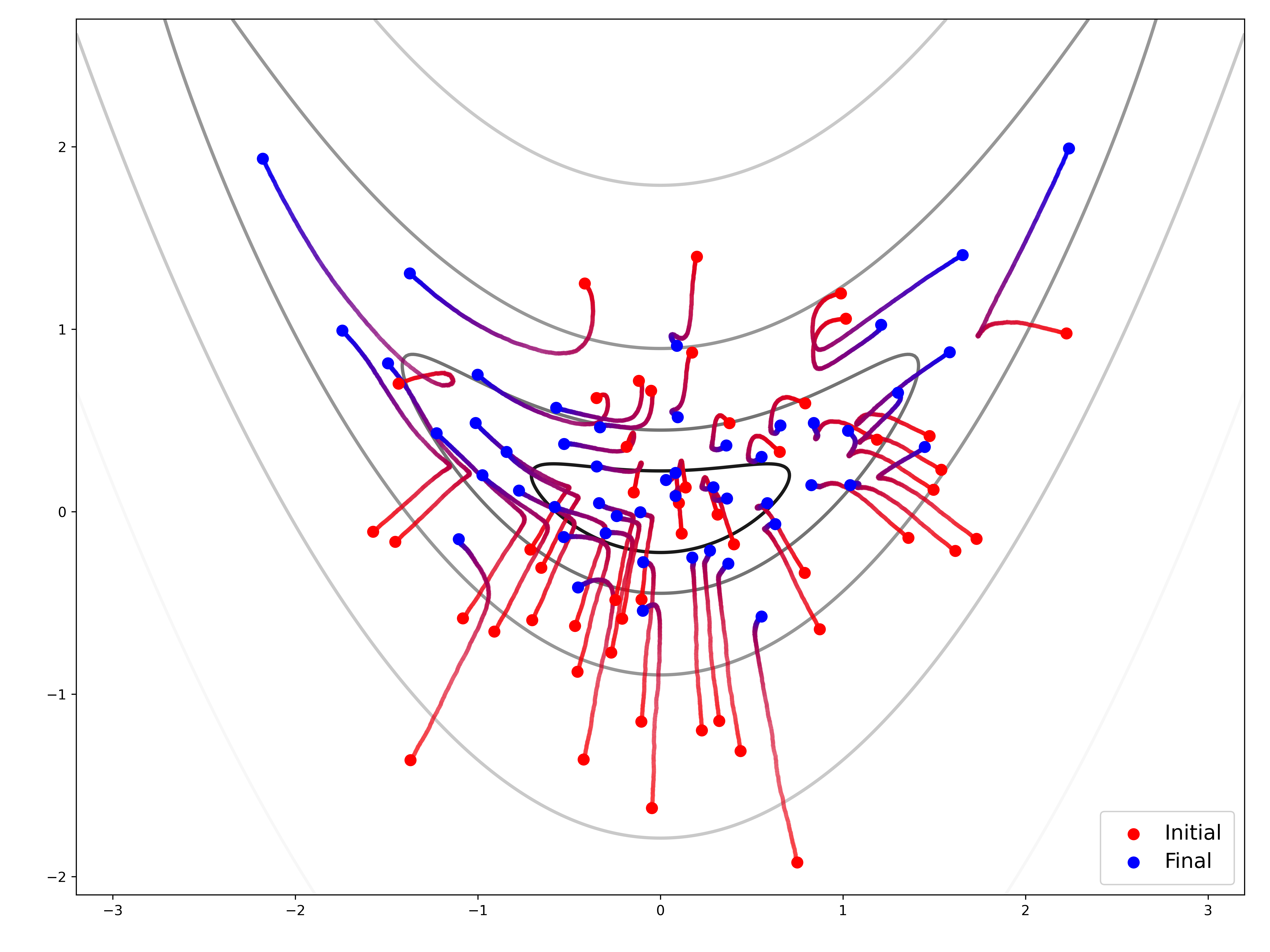}
     \caption{Trajectories of the Regularized Radon--Wasserstein (RRW) gradient flow of the KL divergence~\eqref{eq:RRWgf} for a $2$-dimensional Rosenbrock ``banana" target distribution with potential $U(x)=\frac12 ( x_1^2 +10(x_2 - 0.4x_1^2)^2)$. Trajectories were generated using Algorithm \ref{algo:skeleton} with Routine~\ref{rout:KDRW_fft} and $n=50$ particles. The initial particles were i.i.d.\ sampled from a centered Gaussian with covariance matrix $0.8I_2$.}
\label{fig:trajectories_plot}
\end{figure}

\subsection{Outline} 
In the subsequent subsections, we review related work (Subsection~\ref{sec:related}) and introduce the notation used throughout the paper (Subsection~\ref{sec:notation}).

After reviewing the Radon transform in Subsection~\ref{sec:Radon} we introduce the Radon--Wasserstein metric tensor in Subsection~\ref{sec:RWintro} and the Regularized Radon--Wasserstein metric tensor in Subsection \ref{sec:RRWintro}. The associated gradient flows of the KL divergence are introduced in Subsections~\ref{sec:RWgf} and \ref{sec:RRWgf}, as well as a kernel density–regularized version of the Radon--Wasserstein flow---the Kernel-Density Radon-Wasserstein (KDRW) flow---in Subsection~\ref{sec:KDRWgf}. The details of the derivation of the gradient flows are presented in Subsection~\ref{subsec:GFDeriv}.

In Section~\ref{sec:IPS} we present the interacting-particle approximations of the gradient flows and the associated algorithms. More precisely, in Subsection~\ref{sec:IPSKDRW} we introduce an interacting particle approximation of the  Kernel-Density Radon--Wasserstein system, as well as an efficient way to compute the velocities using the Fast Fourier Transform (FFT). The interacting-particle approximation for the Regularized Radon--Wasserstein flow and the associated routines for computing the velocity are given in Subsection~\ref{sec:IPSRRW}. Subsection~\ref{subsec:comp} is devoted to a discussion of the computational complexity of the algorithms. Finally, in Subsection~\ref{sec:IPSextensions} we discuss an alternative fast algorithm for computing velocities when the regularizing kernel is a Laplace kernel that does not use the FFT. 

Section~\ref{sec:experiments} is devoted to investigating the properties and performance of the proposed algorithms. This includes performance across different kernel bandwidths in Subsection~\ref{sec:bandwidths}, comparison of convergence over time in Subsection~\ref{sec:convergence}, observed computational complexity of steps in Subsection~\ref{subsec:obs_comp}, and the approximation error of the final states over different numbers of particles in Subsection~\ref{sec:quantization}.

Section~\ref{sec:theory} is devoted to establishing theoretical results regarding the introduced flows and algorithms. In particular, Theorem~\ref{thm:specific_existence} establishes existence and uniqueness of KDRW and RRW flows. Theorem~\ref{thm:specific_stability} shows that the solutions are stable with respect to Wasserstein distance. Since the stability estimate also applies to the particle system, it yields mean-field convergence of the interacting-particle system to the continuum solution as $n\rightarrow \infty$. Theorem~\ref{thm:specific_sgd_conv} establishes the convergence of the stochastic gradient descent algorithms introduced in Section~\ref{sec:IPS} and used in the experiments in Section~\ref{sec:experiments}. Finally, Theorem~\ref{thm:long_time_convergence} establishes the qualitative convergence of solutions to the RRW gradient flow equation to the target density as $t \to \infty$. 

In Appendix~\ref{appendix:Otto} we recall some of the background on geometry of optimal transport and Otto calculus. Appendix~\ref{ap:continuity} proves an entropy balance identity for continuity equations, and Appendix~\ref{app:theory} collects technical background and proof details for some of the results in Section~\ref{sec:theory}.

\subsection{Related work} \label{sec:related}

Our work falls in the large class of approaches that are based on recasting sampling as optimization over the space of probability measures. There are several related approaches in this direction. Variational Inference \cite{blei2017variational, wainwright2008graphical, zhang2018advances} considers optimization over a parameterized family of measures such as product measures, Gaussians, and mixtures of Gaussians. 
Variational Inference scales well with dimension and offers fast performance, at the expense of the bias that minimizing over a restrictive class of measures introduces. Rigorous analyses of bias and convergence are available in specific regimes, notably for log-concave targets. In particular,  \cite{arnese2024convergenceVI, bhattacharya2025CAVI, jiang2025algorithms} and \cite{lavenant2024convergence} establish convergence of Coordinate Ascent Variational Inference (CAVI) and related algorithms to the minimizer within the class for log-concave targets.

More closely aligned to this paper are works that consider gradient flows of KL divergence or other functionals and their interacting particle approximations. In seminal work by Jordan, Kinderlehrer, and Otto \cite{jordan1998variational} the Fokker--Planck equation was shown to be the Wasserstein gradient flow of the KL divergence. We note that the Fokker--Planck equation describes the mean-field behavior of particles following (overdamped) Langevin dynamics used in the LMC sampling algorithm. Although LMC is typically presented via discretizations of Langevin dynamics, the underlying Wasserstein gradient-flow viewpoint has been exploited to design variants and to prove convergence results; see Wibisono \cite{wibisono2018sampling} and the book by Chewi \cite{chewi2023book}. Other works more directly exploit the gradient flow structure to introduce interacting-particle systems. The ``blob method'' of  Carrillo, Craig, and Patacchini, \cite{carrillo2019blob} is based on regularizing the KL divergence so it has a finite value for discrete measures. The resulting Wasserstein gradient flow is then a system of interacting particles for discrete initial data. We note that this regularization by a convolution introduces bias that is proportional to the bandwidth of the convolution. This necessitates using convolutions with narrow bandwidth which suffers from the curse of dimensionality.

Stein Variational Gradient Descent (SVGD), introduced by Liu and Wang~ \cite{liu2016stein,liu2017stein}, offers a different route to obtaining an interacting particle system for a gradient flow of the KL divergence. Namely, SVGD is the gradient flow of the KL divergence in a geometry on the space of measures where---instead of the kinetic energy as for the Wasserstein geometry---one considers an RKHS norm on the space of velocities. This yields an integral equation where the velocities are more regular and are meaningful for discrete measures. This gives rise to the SVGD interacting-particle system, and the resulting geometry on the space of measures is now called the Stein geometry. Lu, Lu and Nolen \cite{lu2019scaling} provide important theoretical results for SVGD that include well-posedness, mean-field convergence of interacting-particle systems towards the continuum equation and long-time asymptotic results. The subsequent work of Korba et al.~\cite{korba2020non} provides discrete-time descent properties for KL under suitable stepsizes and stronger long-time convergence results for the continuum equation. The Stein geometry has been further studied  by Duncan, N\"usken, and Szpruc \cite{duncan2023geometry}, who show that due to the integral nature of the equation, the linearized system does not have a spectral gap when the kernel is regular, hence one cannot expect exponential convergence. Finally, exponential convergence for continuous-time SVGD was obtained for certain singular, exponentially growing kernels via a Stein log-Sobolev framework by  Carrillo, Skrzeczkowski, and Warnett \cite{carrillo2024stein}.

Turning to finite-$N$ behavior, Shi and Mackey \cite{ShiMackey2023finite} established finite-particle convergence rates for SVGD, later improved by Banerjee, Balasubramanian, and  Ghosal \cite{Banerjee2024Improved}. In terms of the non-asymptotic analysis of SVGD for a limited number of particles,
Ba et al. \cite{ba2021understanding} shows that for Gaussian targets, unless the number of particles is substantially higher than the dimension, the variance of the samples obtained by SVGD will be substantially lower than the variance of the target distribution. This limits the applicability of SVGD in high dimensional problems.

He,  Balasubramanian, Sriperumbudur, and Lu \cite{HeBalasub25} recently introduced gradient flows and corresponding discretizations that interpolate between SVGD and the Fokker--Planck equation. 
The connection between SVGD and stochastic processes was investigated by N\"usken and Renger \cite{Nusken23}.
In a different extension, Chen, et al. 
\cite{chen2023gradient} develop affine-invariant interacting-particle variations, and Gaussian approximations, of Langevin dynamics and of SVGD.
A further direction towards building interacting particle samplers includes 
consensus-based sampling, Carrillo, Hoffmann, Stuart, and Vaes \cite{Carrillo2022CBS}.

We note that there are several works that, similar to this paper, exploit projections to make the computation and approximation more manageable
 in high dimensions. In particular, Gong, Li and Hern\'andez-Lobato \cite{gong2020sliced} introduce the Sliced Kernelized Stein Discrepancy and use a related geometry to compute flows. Liu, Zhu, Ton, Wynne and Duncan \cite{liu2022grassmann} extend this approach to projections on arbitrary dimensional subspaces. While these flows alleviate some of the issues SVGD encounters, the per-step complexity is $O(n d(n+ d))$ or higher. The works of Chen and Ghattas
\cite{chen2020projected} and of Wang, Chen, and Li \cite{wang2022projected} approximate the flows based on projections to suitable subspaces, which allows for computational efficiency at the cost of introducing bias. In this sense this is a different approach to tackling the difficulties of high dimensional sampling.

There are also exciting recent works that develop interacting-particle methods for high-dimensional Fokker--Planck equations \cite{boffi2023probability,maoutsa20,Reich2021,shen2022self}. Fokker--Planck equations are solved using particles in the works of Maoutsa,  Reich, and  Opper~\cite{maoutsa20} and Reich and Weissman~\cite{Reich2021}. The term involving the score is estimated using variational score matching. 
The papers consider kernel based families of functions and works well in low and moderate dimensions. Boffi and Vanden-Eijnden \cite{boffi2023probability} and Shen et al. \cite{shen2022self} take a similar point of view but use deep neural networks for estimating the score based on particle positions. This approach allows for  high-dimensions, as demonstrated by Boffi and Vanden-Eijnden \cite{boffi2024active}. We remark that, although this family of approaches applies to high-dimensional problems, it implicitly relies on strong assumptions about the score’s structure, since the score must be learned in high dimensions from a limited number of particles. Moreover, training deep networks to approximate the score can require substantial computational resources.

\subsection{Notation} \label{sec:notation}
\begin{itemize} \addtolength{\itemsep}{-4pt}
    \item $\Pc(\R^d)$ denotes the space of probability measures on $\R^d$, and $\Pc_2(\R^d)$ denotes the subset of measures with finite second moment. 
    \item Given $\mu\in\Pc_2(\R^d)$ we let
    \[\|\mu\|_{\Pc_2}:=\left(\int_{\R^d} |x|^2\,d\mu(x)\right)^{1/2}.\]
    \item $\mathcal{W}(\mu,\nu)$ denotes the Wasserstein distance between $\mu,\nu\in\Pc_2(\R^d)$. That is,
    \[\mathcal{W}(\mu,\nu):=\inf_{\gamma\in\Gamma(\mu,\nu)}\left(\int_{(\R^d)^2}|x-y|^2\,d\gamma(x,y)\right)^{1/2}\]
    where $\Gamma(\mu,\nu)\subset \Pc((\R^d)^2)$ is the set of all couplings of $\mu$ and $\nu$.
    \item $C^k_b(\R^d)$ denotes the space of functions with bounded derivatives up to order $k$ with norm
    \[\|f\|_{C^k_b}=\sup_{1\leq j\leq k} \|D^j f\|_{\infty}.\]
    \item $T\#\nu$ refers to the pushforward of a finite signed measure $\nu$ under a measurable map $T$. That is, $T\#\nu(B)= \mu(T^{-1}(B))$ for all measurable $B$.
    \item For $\theta\in\S^{d-1}$, $\nabla_\theta$ denotes the directional derivative $\theta\cdot\nabla$.
    \item Given a function $f$ on $\S^{d-1}$, 
    \[\fint_{\S^{d-1}} f(\theta)\,d\theta\]
    denotes integration over the normalized $d-1$-dimensional volume form.
    \item $\RD^d$ denotes the $d$-dimensional Radon domain
    \[\RD^d:=(\S^{d-1}\times \R)/\sim\]
    where $\sim$ is the equivalence relation $(\theta,p)\sim(-\theta,-p).$ This is endowed with the measure inherited from $\S^{d-1}\times \R$. We thus abuse notation and denote integration over $\mathbb{D}$ by $d\theta dp$.
    \item $\ux^n$ denotes a vector $(x^1,\dotsc,x^n)\in (\R^d)^n$.
\end{itemize}

\section{Radon--Wasserstein metric tensor and gradient flows}\label{sec:defs}

In this section, we recall the definition of the Radon transform and its dual, define the Radon--Wasserstein (RW) metric tensor and its variants, and formally derive the gradient flow equations of the KL divergence with respect to these metric tensors. 

\subsection{The Radon transform and its dual} \label{sec:Radon}

We first recall the Radon transform, which maps measures (and functions) on $\R^d$ to measures (and functions) on the Radon domain $\mathbb{D}^d$. Throughout, we associate $\mathbb{D}^d$ with $\S^{d-1}\times \R$ and functions on $\mathbb{D}^d$ with functions on $\S^{d-1}\times \R$ that are invariant under the equivalence relation $(\theta,p)\sim(-\theta,-p)$. That is, if $g$ is a measurable function on $\mathbb{D}^d$, then we implicitly associate it with an even function $g$ on $\S^{d-1}\times \R$ and let
\[\int_{\RD^d} g(\theta,p)\,d\theta dp=\int_{\R} \fint_{\S^{d-1}} g(\theta,p)\,d\theta dp.\]
For a comprehensive review of the Radon transform and its properties, see~\cite{helgason2011integral}. 

\begin{definition} Letting $\Pi_\theta$ denote the projection $\Pi_\theta(x)=x\cdot\theta$, for any finite signed measure $\mu$, the \textit{Radon transform} of $\mu$ is defined by $R^\theta\mu:=\Pi_\theta\#\mu$. $R\mu$ thus denotes the finite signed measure on $\mathbb{D}^d$ defined by
\[R\mu(\Theta \times B)=\fint_{\Theta} R^\theta\mu(B)\,d\theta,\]
where $\Theta\subset \S^{d-1}$ and $B\subset \R$ are Borel measurable sets.
\end{definition}

More specifically, if $f\in L^1(\R^d)$, we use the convention that
\[Rf(\theta,p):=R^\theta f(p):=\int_{\theta^\perp} f(p\theta+y^\theta)\,dy^\theta, \quad \theta\in\S^{d-1},p\in\R,\]
where $\theta^\perp:=\{y\in\R^{d}:y\cdot\theta=0\}$ and $dy^\theta$ is the $(d-1)$-dimensional Lebesgue measure on $\theta^\perp$. This is consistent with the definition above since if $d\mu(x)= f(x)\,dx$, then $d(R\mu)(\theta,p)=Rf(\theta,p)\,d\theta dp$.

The dual Radon transform is then identified as follows.
\begin{definition}
For $g\in L^1_{loc}(\RD
^d)$, $R^* g$ denotes the \textit{dual Radon transform}
\[R^*g(x):=\fint_{\S^{d-1}}g(\theta,x\cdot\theta)\,d\theta,\quad x\in\R^d,\]
where $d\theta$ is the $(d-1)$-dimensional Hausdorff measure on $\S^{d-1}$.
\end{definition}

The main property of the Radon transform and its dual, which is easily established by Fubini's theorem, is the following adjoint property.

\begin{proposition}\label{prop:adjoint}
If $f\in L^1(\R^d)$ and $g\in L^\infty(\mathbb{D}^d)$, then
\begin{equation}\label{eq:adjoint}
\int_{\RD^d} Rf(\theta,p) g(\theta,p)\,d\theta dp=\int_{\R^d} f(x)R^*g(x)\,dx.
\end{equation}
\end{proposition}

\subsection{Radon--Wasserstein metric tensor} \label{sec:RWintro}

This discussion is motivated by the Otto Calculus~\cite{otto2001geometry}, where $\mathcal{P}_2(\R^d)$ endowed with the Wasserstein geometry can be formally viewed as a Riemannian manifold with metric tensor given by
\begin{equation} \label{eq:Wass_tensor}
\overline g_\rho(s,s):=\inf_{v \::\: s=-\nabla\cdot(\rho v)}g_\rho(v,v) \quad \text{ with }\: g_\rho(v,v) = \int_{\R^d} |v(x)|^2\,d\rho(x).
\end{equation}
Above, $s$ is a mass-preserving variation of the density and $v$ provides its Lagrangian description: the change of density $s$ is achieved by infinitesimally transporting the density along $v$. The quadratic form $g_\rho(v,v)$ provides a metric tensor on velocity fields and induces the metric on density variations via the minimization above. This can be viewed formally as an isometric submersion.

One can additionally associate to $s$ the minimizing (tangent) velocity field $v$ and identify the tangent space $T_\rho\mathcal{P}_2(\R^d)$ with $\overline{\{\nabla\phi:\phi\in C^\infty_c(\R^d)\}}^{L^2(\rho)}$; see \cite{ambrosio2005gradient} for rigorous discussion of the tangent structure of the spaces of probability measures endowed with Wasserstein metric. 

In our discussions we introduce new metric tensors on velocity fields, which induce corresponding metrics on spaces of probability measures as in~\eqref{eq:Wass_tensor}. These replace the Wasserstein choice $g_\rho(v,v)=\|v\|_{L^2(\rho)}^2$ with metric tensors involving the Radon transform and its dual. Our choices of $g_\rho$ are motivated by the following consideration: to accurately approximate velocities—particularly those arising in gradient flows—in high dimensions with a limited number of particles, the admissible velocity fields must have low information-theoretic complexity. Accordingly, we restrict the class of allowable velocities to those that decompose into simple uni-directional movements.

We first introduce an auxiliary metric tensor.

\begin{definition}\label{def:sliced_metric}
Given a direction $\theta\in \S^{d-1}$, the \textit{direction-$\theta$ Wasserstein metric tensor} at $\rho\in\mathcal{P}_2(\R^d)$ is defined by
\begin{equation}
    g_{\rho}^\theta(v^\theta,v^\theta) :=
    \begin{cases}
    \displaystyle{\int_{\R^d} u^\theta(x\cdot \theta)^2d\rho(x)} & \text{if } \exists u^\theta: \R \to\R\, \text{ such that }  v^\theta(x) = \theta u^\theta(x\cdot \theta)\text{ on } \supp(\rho) \\
     \infty & \text{otherwise}.
    \end{cases}
\end{equation}
Above, $u^\theta$ is assumed to be in $L^2(R^\theta\rho)$.
\end{definition}

That is, $g^\theta$ is the Wasserstein metric tensor restricted to velocity fields that only move mass in direction $\theta$ and only depend on the projection of the position to direction $\theta$. We can thus naturally define a metric tensor over velocity fields that can be decomposed into unidirectional velocity fields via
\[g_\rho(v,v):=\fint_{\S^{d-1}} g_\rho^\theta(v^\theta,v^\theta)\,d\theta\quad\text{ if}\quad v= \fint_{\S^{d-1}}  v^\theta \,d\theta \;\;\text{ on }\supp(\rho).\]
More precisely, expanding out the definition of $g_\rho^\theta$ and using the definition of the dual Radon transform $R^*$, we define the \textit{Radon--Wasserstein metric tensor} below.

\begin{definition}\label{def:RW} The \textit{Radon--Wasserstein metric tensor} at $\rho\in\mathcal{P}_2(\R^d)$ is defined by
\[g_\rho(v,v):=\inf_{\substack{v=R^*(\theta u)\\\text{on } \supp(\rho)}}\left\{\fint_{\S^{d-1}} \int_{\R^d} u(\theta,x\cdot\theta)^2 \,d\rho(x)\,d\theta\right\},\]
where the infimum is over all $u\in L^2(R\rho)$ such that $v=R^*(\theta u)$ on $\supp(\rho)$.
\end{definition}

We note that since $R^*$ is injective~\cite[Chapter 1 $\mathsection$3]{helgason2011integral}, if $\supp(\rho)=\R^d$ then there is at most one $u$ such that $v=R^*(\theta u)$. Thus, in this case, we have the simplified expression
\[g_\rho(v,v)=\begin{cases}\int_{\RD^d} u^2 d R\rho &\text{if }v=R^*(\theta u)\\
\infty&\text{otherwise.}
\end{cases}
\]
Additionally, Jensen's inequality and Proposition~\ref{prop:adjoint} imply that
\[\int_{\R^d}\bigg|\fint_{\S^{d-1}}\theta u(\theta,x\cdot\theta)\,d\theta\bigg|^2\,d\rho\leq \int_{\R^d} \fint_{\S^{d-1}}u(\theta,x\cdot\theta)^2\,d\theta\,d\rho=\int_{\R^d} R^*(u^2)\,d\rho=\int_{\RD^d} u^2 \,d(R\rho).\]
Thus if $u\in L^2(R\rho)$, then $R^*(\theta u)\in L^2(\rho)$ and $\|v\|_{L^2(\rho)}^2\leq g_\rho(v,v)$.

We make the following immediate remarks.

\begin{remark}
The metric-tensor $g_\rho$ formally defines a distance on $\mathcal{P}_2(\R^d)$ via the Benamou--Brenier-like formula
\[\mathcal{W}_R(\mu,\nu):=\inf\left\{ \left(\int_0^1 g_{\rho_t}(v_t,v_t)\,dt\right)^{1/2}:\partial_t\rho_t+\nabla\cdot(v_t\rho_t)=0,\rho_0=\mu,\rho_1=\nu\right\},\]
see Appendix~\ref{appendix:Otto}. Then, since $\|v\|_{L^2(\rho)}\leq g_\rho(v,v)$, we have that $\mathcal{W}(\mu,\nu)\leq \mathcal{W}_R(\mu,\nu)$ for all $\mu,\nu\in \mathcal{P}_2(\R^d)$.
\end{remark}

\begin{remark}
Suppose that $\rho=\sum_{i=1}^n m_i \delta_{x^i}$ where $m_i>0$ sum to $1$ and the points $x^i\in \R^d$ are distinct. For $i\neq j$, the set of directions $\theta\in \S^{d-1}$ such that
$x^i\cdot\theta = x^j\cdot\theta$ has spherical measure $0$, hence $x^i\cdot\theta$ are distinct for a.e.\ $\theta$. After straightforward optimization, this implies that for any $v \in L^2(\rho)$
\[g_\rho(v,v)=d\sum_{i=1}^n m_i |v(x^i)|^2=d\|v\|_{L^2(\rho)}^2.\]
Thus, up to a dimensional normalization constant, at discrete measures the Radon--Wasserstein metric tensor is the same as the $L^2(\rho)$ inner product, and the tangent space is the same as that of the Wasserstein geometry.

In contrast, if $\rho(dx)=\rho(x)dx$ is such that $\rho(x)$ is continuous and nonnegative, then so is $R\rho$. This implies that if $g_\rho(v,v)<\infty$, then there must exist $u\in L^2_{loc}(\mathbb{D}^d)$ such that $v=R^*(\theta u)$. Due to the property that $R^*: L^2_{loc}(\mathbb{D}^{d-1})\rightarrow H_{loc}^{(d-1)/2}(\R^d)$~\cite[Lemma 2]{Boman2006StableInversionHalfData}, this implies that $v\in H_{loc}^{(d-1)/2}(\R^d)$. Thus at smooth and nondegenerate measures, only velocity fields that are locally much smoother than $L^2$ are admissible.

Together these examples show that the RW metric tensor has a very different nature at measures with full support than at discrete measures. At discrete measures it behaves like the Wasserstein metric tensor, while at measures with continuous density it requires the velocities to be more regular. This behavior is dual to that of the Sliced Wasserstein (SW) metric as observed in \cite{park2023geometry}: the SW metric tensor agrees with the Wasserstein metric tensor at discrete measures but allows tangent velocity fields that live in negative regularity Sobolev spaces at continuous measures. 
\end{remark}

\subsection{Regularized Radon--Wasserstein metric tensor} \label{sec:RRWintro}

We also introduce a regularization of the Radon--Wasserstein geometry. This will result in gradient flows with smoother velocities. This has both theoretical and computational advantages. It will depend on some regularizing kernel $k:\R\rightarrow [0,\infty)$, and a small parameter $\eps>0$. Again, we first introduce a (unidirectional) auxiliary metric tensor.

\begin{definition}\label{def:regularized_sliced_metric}
Given a direction $\theta\in \S^{d-1}$, a kernel $k:\R\rightarrow [0,\infty)$, and $\eps>0$, the \textit{direction-$\theta$ Regularized Radon--Wasserstein metric tensor} is given by
\begin{align*}
    &g_{\rho}^{\theta,k,\eps}(v^\theta,v^\theta)
    \\&\ :=
    \begin{cases}
    \displaystyle{\int_{\R} (u^\theta)^2(k*R^\theta\rho+\eps)\,dp}& \text{if } \exists u^\theta: \R \to\R\, \text{ such that }  v^\theta(x) = \theta k*u^\theta(x\cdot \theta)\text{ on } \supp(\rho)\\
     \infty & \text{otherwise}.
    \end{cases}
\end{align*}
Above, $u^\theta$ is assumed to be in $L^2(\R)$.
\end{definition}

If $\eps=0$ and $k(p)=\delta(p)$ where $\delta(p)$ denotes the Dirac delta on $\R$, then the above definition concurs with Definition~\ref{def:sliced_metric} since
\[\int_{\R} (u^\theta)^2 \,d(R^\theta\rho)=\int_{\R^d} u^\theta(x\cdot\theta)^2d\rho(x)\]
by definition of $R^\theta\rho$.

The Regularized Radon--Wasserstein metric tensor is then defined in an analogous way as to Definition~\ref{def:RW}.

\begin{definition}\label{def:RRWM} For a kernel $k:\R\rightarrow [0,\infty)$ and $\eps>0$, the \textit{Regularized Radon--Wasserstein metric tensor} is defined by
\begin{equation*}
g^{k,\eps}_\rho(v,v):=\inf_{\substack{v=R^*(\theta k*u)\\\text{on }\supp(\rho)}} \int_{\RD^d} u^2(k*R\rho+\eps)\,d\theta dp
\end{equation*}
where the infimum is over all $u\in L^2(\RD^d)$ such that $v=R^*(\theta k*u)$ on $\supp(\rho)$.
\end{definition}

\begin{remark}
If $k$ is a smooth mollifier and $\eps>0$, then $g_\rho^{k,\eps}(v,v)<\infty$ implies that $u\in L^2(\mathbb{D}^d)$, and thus $k*u$ and $v$ are smooth. This implies that if $\mu$ and $\nu$ are connected by a curve $\rho$ solving $\partial_t\rho+\nabla(v\rho)=0$ so that $\rho_0=\mu$, $\rho_1=\nu$, and
\[ \int_0^1 g^{k,\eps}_{\rho_t}(v,v)\,dt<\infty\]
then the supports of $\mu$ and $\nu$ must have the same number of connected components. This demonstrates that $\mathcal{P}_2(\R^d)$ would be foliated under the geometry defined by the RRW metric tensor.
\end{remark}

\subsection{Radon--Wasserstein gradient flow} \label{sec:RWgf}

Here we state the (formal) equations for gradient flows with respect to Radon--Wasserstein metric tensor. In particular, we consider a target probability measure with density $\pi\ \propto\ e^{-U}$ for a potential $U$. The Kullback--Leibler divergence or relative entropy is then defined below.

\begin{definition}
$\mathcal{F}(\rho)$ denotes the Kullback--Leibler divergence of a probability measure $\rho$ with respect to probability measure $\pi$
\[\mathcal{F}(\rho):=
\begin{cases}
\displaystyle{\int_{\R^d} \log\left(\frac{\rho}{\pi}\right)\, \rho\,dx} \quad & \text{if } \rho\ll\pi ,
\\\infty&\text{otherwise}
\end{cases}\]
\end{definition}

We can now identify the  gradient flow equation of the Kullback--Leibler divergence with respect to the Radon--Wasserstein metric tensor. By this we mean the following: at a measure $\rho$, the gradient $\text{grad}_g E(\rho)$ of a differentiable function $E$ with respect to the metric tensor $g_\rho$ is the unique tangent velocity field such that for all velocities $w$
\begin{equation}\label{eq:grad} 
g_\rho(\text{grad}_gE(\rho),w) = \diff|_\rho E(w) = \frac{\delta E}{\delta \rho}(- \nabla \cdot (\rho w)).
\end{equation}
The associated gradient flow equation is then $\partial_t \rho + \nabla \cdot (\rho v) = 0$ with $v=-\text{grad}_g E(\rho)$.

\begin{proposition}\label{prop:ori_pgf}
The gradient flow of $\mathcal{F}$ with respect to the Radon--Wasserstein metric tensor $g$ is given by the equation
\begin{align}
\label{eq:ori_pgf}
    &\partial_t\rho + \nabla\cdot(\rho v) = 0, \quad  \text{with}
     \quad v(t,x) = -\fint_{\S^{d-1}} \theta \left(\frac{\theta\cdot (R^\theta(\nabla\rho_t + \rho_t\nabla U))}{R^\theta\rho_t}\right)(x\cdot\theta)\,d\theta.   
\end{align}
\end{proposition}

We verify the claim above in Subsection~\ref{subsec:GFDeriv}.
We note that the velocity in the equation above can also be written as
\[v(t,x)=-\fint_{\S^{d-1}} \theta \left(\frac{\partial}{\partial p}(\log R^\theta\rho_t)+\frac{R^\theta(\rho_t\nabla_\theta U)}{R^\theta\rho_t}\right)(x\cdot\theta)\,d\theta.\]
The first term in the integral on the right-hand side is the derivative of the logarithm of projected density, which can be calculated based on the projection itself; and the second term is the average of $\nabla U$ over the hyperplane that projects to a given point.

\subsection{Kernel-Density Radon--Wasserstein (KDRW) Flow} \label{sec:KDRWgf}

The velocity in \eqref{eq:ori_pgf} is not defined if $\rho$ is a discrete measure. In particular, the first term only makes sense for measures whose Radon transform is differentiable in $p$. For this reason, to create particle approximations of the Radon--Wasserstein gradient flow, we introduce a kernel-regularized version of the PDE that we call the \textit{Kernel-Density Radon--Wasserstein} (KDRW) flow.

To approximate the flow \eqref{eq:ori_pgf}, we again consider a regularizing kernel $k : \R \to [0, \infty)$ and a small parameter $\eps>0$. The kernel-density flow is then defined as follows.

\begin{definition}
The \textit{Kernel-Density Radon--Wasserstein flow} is given by the equation,
\begin{align}
\label{eq:KDRWgf}
    &\partial_t\rho + \nabla\cdot(\rho v) = 0, \quad 
     \quad v(t,x) = -\fint_{\S^{d-1}}\theta \left(\frac{k'*R^\theta\rho_t + k*R^\theta(\rho_t\nabla_\theta U)}{k*R^\theta\rho_t+\eps} \right)(x\cdot\theta)\,d\theta.
\end{align}
\end{definition}

One should think of the KDRW flow as an approximation of the RW flow~\eqref{eq:ori_pgf}. Formally, \eqref{eq:KDRWgf} converges to \eqref{eq:ori_pgf} as the kernel concentrates (i.e., as its bandwidth tends to zero) and $\eps \to 0$. In practice, however, if one starts from an $n$-particle approximation of $\rho_0$, the choice of kernel must depend on $n$ in order to balance regularization bias with finite-$n$ approximation error. In this sense, KDRW performs kernel density estimation of the one-dimensional projected densities that define the RW velocity field.

\subsection{Regularized Radon--Wasserstein (RRW) Gradient Flow} \label{sec:RRWgf}

We now define the formal gradient flow of the Kullback--Leibler divergence with respect to the Regularized Radon--Wasserstein metric tensor as defined in Definition~\ref{def:RRWM}. 
In contrast to Radon--Wasserstein gradient flow, the RRW gradient flow is well-defined for discrete measures, and is thus amenable to approximation by particles. 

For a regularizing kernel $k:\R\rightarrow[0,\infty)$ and a small parameter $\eps>0$, the \emph{Regularized Radon--Wasserstein} gradient flow of $\mathcal{F}$ is given below.

\begin{proposition}\label{prop:RRWgf} The gradient flow of $\mathcal{F}$ with respect to the Regularized Radon--Wasserstein metric tensor $g^{k,\eps}$ is given by
\begin{equation}
\label{eq:RRWgf}
\partial_t\rho+\nabla\cdot(\rho v)=0,\qquad v(t,x)=-\fint_{\S^{d-1}}\theta k*\left(\frac{k'*R^\theta\rho_t+k*R^\theta(\rho_t\nabla_\theta U)}{k*R^\theta\rho_t+\eps}\right)\,d\theta. 
\end{equation}   
\end{proposition}

Again, this derivation is formal, and we provide details in the next subsection. The RRW velocity~\eqref{eq:RRWgf} is closely related to the KDRW velocity~\eqref{eq:KDRWgf}, but is smoother due to the additional convolution with $k$. In contrast to the KDRW flow, the regularization kernel enters through the metric tensor, rather than as an ad hoc smoothing introduced to accommodate discrete measures, and thus does not introduce an additional regularization bias of the type present in the KDRW flow. Consequently, the role of the bandwidth of $k$ is more subtle: it changes the geometry (and hence the induced dynamics), rather than serving only as a tuning parameter.

\subsection{Derivation of the gradient flows}\label{subsec:GFDeriv}

To compute the gradient of the Kullback--Leibler divergence $\mathcal F$ with respect to the
(Regularized) Radon--Wasserstein geometry, we use the Rayleigh functional
\[
\mathcal R_\rho(v)\;:=\;\frac12\,g_\rho(v,v)\;+\;\diff\big|_\rho \mathcal F(v).
\]
As explained in Appendix~\ref{appendix:Otto}, we have that
\[
-\text{grad}_g\mathcal F(\rho)\;=\;\arg\min_{v}\,\mathcal R_\rho(v).
\]
Thus, to derive the velocity of the gradient flow, it suffices to compute the minimizer of $\mathcal R_\rho$ for the appropriate metric tensor $g_\rho$. In the Wasserstein case $g_\rho(v,v)=\|v\|_{L^2(\rho)}^2$, this recovers the Fokker--Planck equation; again
see Appendix~\ref{appendix:Otto}.

\medskip

\begin{proof}[Derivation of Proposition~\ref{prop:ori_pgf} and~\ref{prop:RRWgf}] We will only derive~\eqref{eq:RRWgf} as the derivation of~\eqref{eq:KDRWgf} then follows by setting $\eps=0$ and $k(p)=\delta(p)$.

First, we note that the first variation of $\mathcal{F}$ is given by
\[\text{diff}|_{\rho}\mathcal{F}(v)=\frac{d}{dt}{\Big|}_{t=0} \mathcal{F}(\rho_t)=\int_{\R^d} (\nabla\rho+\rho\nabla U)\cdot v\,dx\]
where $\rho_t$ satisfies the continuity equation $\partial_t\rho_t+\nabla(v\rho_t)=0$ at time $0$. Formally viewing $(\Pc_2(\R^d),g^{k,\eps})$ as a Riemannian manifold, we thus need to  minimize the Rayleigh functional
\[\mathcal{R}(v)=\frac{1}{2}g_\rho^{k,\eps}(v,v)+ \text{diff}|_{\rho}\mathcal{F}(v)=\inf_{v=R^*(v^\theta)}\fint_{\S^{d-1}}\left(\frac{1}{2}g^{\theta,k,\eps}(v^\theta,v^\theta)+\text{diff}|_{\rho}\mathcal{F}(v^\theta)\right)\,d\theta\]
where in the second equality we have expanded out the definition of $g^{k,\eps}$ and used the linearity of $\text{diff}|_{\rho}\mathcal{F}(v)$ in $v$. We thus have that 
\[\mathcal{R}(v)=\inf_{v=R^*(v^\theta)} \fint_{\S^{d-1}} \mathcal{R}^\theta(v^\theta)\,d\theta,\quad\text{where}\quad \mathcal{R}^\theta(v^\theta):=\frac{1}{2}g_\rho^{\theta,k,\eps}(v^\theta,v^\theta)+\diff|_\rho\mathcal{F}(v^\theta).\]
This implies that
\[\text{argmin}\ \mathcal{R}(v)=\fint_{\S^{d-1}} \text{argmin}\  \mathcal{R}(v^\theta)\,d\theta,\]
thus we must only compute the $\text{argmin}$ of $\  \mathcal{R}^\theta(v^\theta)$ for all $\theta$.

For admissible $v^\theta(x)=\theta k*u^\theta(x\cdot\theta)$, we can abuse notation and write $\mathcal{R}^\theta(v^\theta)$ as a function of $u^\theta$
\begin{align*}
 \mathcal{R}^\theta(u^\theta)&=\frac{1}{2}\int_{\R} (u^\theta)^2(k*R^\theta\rho+\eps)\,dp+\int_{\R^d} \theta\cdot(\nabla\rho(x)+\rho(x) \nabla U(x)) k*u^\theta(x\cdot\theta)\,dx
 \\&=\frac{1}{2}\int_{\R} (u^\theta)^2(k*R^\theta\rho+\eps)+u^\theta \theta\cdot (k*R^\theta(\nabla\rho+\rho \nabla U))\,dp 
\end{align*}
where we have used the definition of $R^\theta$ in the second line.

Perturbing $u^\theta$ by $w^\theta$, the first variation of $\mathcal{R}^\theta$ is equal to
\[\frac{\partial\mathcal{R}^\theta}{\partial u^\theta}(w^\theta)=\int_{\R} \big(u^\theta(k*R^\theta\rho+\eps)+\theta\cdot(k*R^\theta(\nabla\rho+\rho\nabla U))\big)w^\theta,\]
which implies that
\[\text{argmin}_{v^\theta}\ \mathcal{R}^\theta(v^\theta)=-\theta k*\left(\frac{ \theta\cdot (k*R^\theta(\nabla\rho+\rho\nabla U))}{k*R^\theta\rho+\eps}\right).\]
Thus, using that $\theta\cdot (k*R^\theta(\nabla\rho+\rho\nabla U))=k'*R^\theta\rho+k*R^\theta(\rho\nabla_\theta U)$, in total we find that
\[\text{argmin}_{v}\ \mathcal{R}(v)=- \fint_{\S^{d-1}} \theta k*\left(\frac{k'*R^\theta\rho+k*R^\theta(\rho\nabla_\theta U)}{k*R^\theta\rho+\eps}\right)\,d\theta,\]
as claimed.
\end{proof}

\section{Interacting-particle approximations} \label{sec:IPS}

In this section we introduce interacting particle approximations for the KDRW and RRW flows, discuss their algorithmic complexity, and describe some extensions. We note that the velocity fields $v$ in~\eqref{eq:KDRWgf} or~\eqref{eq:RRWgf} are well-defined when $\rho=\frac{1}{n}\sum_{i=1}^n\delta_{x^i}$, and, if $k$ is differentiable and the initial conditions are discrete, then the equations define a coupled system of ODEs (see Corollary~\ref{cor:specific_ode_wp}) that can be approximated by a forward Euler scheme.

As evaluating these velocities requires integrating over $\S^{d-1}$, we approximate the integral by a single Monte Carlo sample at each time step. Specifically, letting
\begin{equation}\label{eq:spec_integral_kern}
u(\theta,p,\mu):=-\theta\left(\frac{k'*R^\theta\mu+k*R^\theta(\mu\nabla_\theta U)}{k*R^\theta\mu+\eps}\right)(p),
\end{equation}
we sample $\theta$ uniformly from $\S^{d-1}$ and update the particle positions using the velocity field
$u(\theta,x\cdot\theta,\rho)$ for the KDRW flow, or $k*u(\theta,x\cdot\theta,\rho)$ for the RRW flow.
In particular, if $v(t,x)$ is respectively defined by~\eqref{eq:KDRWgf} or~\eqref{eq:RRWgf}, then
\[v(t,x)=\E_\theta[u(\theta,x\cdot\theta,\rho_t)]\quad\text{or}\quad v(t,x)=\E_\theta[k*u(\theta,x\cdot\theta,\rho_t)],\]
so the single-direction updates are unbiased estimators of the sphere-averaged velocity. This thus amounts to a form of stochastic gradient descent~\cite{RobMon51}.

The general skeleton to these algorithms is given by the following pseudocode, with the scalar velocity $v_m^i$ to be specified by the routines in the subsequent subsections.

\begin{algorithm}[h]
\caption{Radon--Wasserstein Algorithm Skeleton}\label{algo:skeleton}
\textbf{Input:} Initial
particles $\{x_0^i\}_{i=1}^n$, number of iterations $T$, time step $\tau$, and score  S = $-\nabla U$\\
\textbf{Output:} An $n$-particle approximation of $\pi$
\begin{algorithmic}[1]
\For{\texttt{m = 0:T}}
    \State Choose random direction $\theta_m$
    \State Compute the particle projection vector $p_m$ with entries $p_m^i := \theta_m \cdot x_m^i$
    \State  Compute projected score vector $s_m$ with entries $s_m^i:=\theta_m\cdot S(x_m^i)$
    \State For each $i$, compute $v_m^i$ (as a function of $p_m$ and $s_m$)
    \State Update $x_{m+1}^i = x_m^i - \tau \theta_m v_m^i$
\EndFor
\end{algorithmic}
\end{algorithm}

\begin{remark}
In the way the algorithm is written, it requires $n$ evaluations of the score function $S$ in the step 4, each of which has $d$ coordinates. The score is then dotted with $\theta_m$ to compute directional derivatives. Computing the scores and the dot products is a major part of the overall complexity of the algorithm. We note that instead of computing the scores one can easily approximate the directional derivatives by using finite differences as follows: for small $\triangle p >0$ let $\triangle x = \theta_m^i \triangle p$. Then $s_m^i \approx \frac{1}{2\triangle p} (U(x_m^i + \triangle x) - U(x_m^i - \triangle x))$. This calls for $2n$ evaluations of the potential, which, depending on $U$, may be much less demanding than computing the scores. 
\end{remark}

\subsection{Interacting particle approximation of KDRW flow} \label{sec:IPSKDRW}
 Given time step $\tau$ and an initial configuration, consider the forward discretization of the KDRW flow~\eqref{eq:KDRWgf} with a signle, random direction per step:
\begin{equation}\label{eq:KDRW_sgd}
    x_{m+1}^i = x_m^i - \tau\theta_m \! \left(\frac{\sum_{j=1}^n k'(p_m^i -p_m^j) + \sum_{j=1}^n k(p_m^i -p_m^j)\nabla_{\theta_m} U(x_m^j)}{\sum_{j=1}^n k( p_m^i -p_m^j)+n\eps}\right),
\end{equation}
where $i= 1,\dotsc,n$, $\theta_m$ is sampled uniformly on $\S^{d-1}$, and $p_m^i = x_m^i \cdot \theta_m$. This clearly corresponds to Algorithm~\ref{algo:skeleton} with $v_m^i$ defined by
\begin{equation}\label{eq:KDRW_velocity}
    v_m^i:=\frac{\sum_{j=1}^n k'(p_m^i -p_m^j) + \sum_{j=1}^n k(p_m^i -p_m^j)\nabla_{\theta_m} U(x_m^j)}{\sum_{j=1}^n k( p_m^i -p_m^j)+n\eps}.
    \end{equation}
This can be efficiently computed using vector operations. The following pseudo code thus replaces line 5 in Algorithm~\ref{algo:skeleton}.

\begin{routine}[ht]
\caption{KDRW scalar velocity}\label{rout:KDRW}
\textbf{Input:} Projected particles $\{p^i_m\}_{i=1}^n$, projected score vector $\{s^i_m\}_{i=1}^n$\\
\textbf{Output:} Scalar velocity $v_m$
\begin{algorithmic}[1]
 \State Compute ``kernel matrix'' $K_{m}$, whose entries are $K_m^{i,j} = k (p_m^i-p_m^j)$ and ``derivative kernel matrix''  $K'_m$, whose entries are $(K_m')^{i,j} = k'(p_m^i-p_m^j)$
\State Compute row sums of $K_m$ and $K_m'$ and denote them by $\kappa_m$ and $\kappa'_m$ \smallskip
\State Compute $v_m=(\kappa_m' -K_m s_m)\oslash(\kappa_m+n\eps)$, where $\oslash$ denotes element-wise (or Hadamard) division of vectors
\end{algorithmic}
\end{routine}

We note that if $k$ has finite support and $p_m^i$ has no neighbors within the support of $k$, then $x_{m+1}^i = x_m^i - \tau  \theta_m \nabla U(x_m^i) \cdot \theta_m$. That is, on average, the particles follow the velocity field $-\nabla U$ until they become sufficiently close to other particles. Additionally, if $k(0)>0$ then $(\kappa_m^i+n \eps)\geq k(0)$ for all $i$. We thus expect that changing the regularization parameter $\eps$ does little to change the output of the algorithm when $n\eps\lesssim k(0)$.

Alternatively we can approximate the one-dimensional convolutions in~\eqref{eq:KDRW_velocity} by using the Fast Fourier Transform (FFT)~\cite{cooley1965fft} to compute convolutions on a uniform grid. Specifically,  suppose that the kernel is specified by some parametrized family of the form \[k_b(x)\,\propto\, k_1(x/b),\quad b>0\]
where $b$ sets the spatial bandwidth (for example,\ $k_b(x)\propto \exp(-x^2/2b^2)$). Then we can truncate the interaction to a radius $R = Lb$ for some cutoff $L>0$, restrict computations to a padded interval containing the projected particles, and discretize this interval with grid spacing $h=b/M$ for some parameter $M>0$ (so that bins are small relative to the kernel scale). We treat the FFT convolution as periodic on the domain and pad the interval by $R$ to eliminate wrap-around error. The following routine thus replaces the kernel-matrix and row-sum computations in Routine~\ref{rout:KDRW} by gridding, FFT convolution, and interpolation.

\begin{routine}[h]
\caption{KDRW FFT scalar velocity}\label{rout:KDRW_fft}
\textbf{Input:} Projected particles $\{p^i_m\}_{i=1}^n$, projected score vector $\{s^i_m\}_{i=1}^n$, cutoff $L$, and discretization parameter $M$\\
\textbf{Output:} Scalar velocity $v_m$
\begin{algorithmic}[1]
    \State Set $R := Lb$, $h := b/M$, and $[a,c] := [\min_i p^i_m - R,\ \max_i p^i_m + R]$
    \State Build a uniform grid on $[a,c]$ with spacing $h$, and deposit $\{p^i_m\}$ and $\{s^i_m\}$ to obtain grid arrays
    $\rho_m\approx \sum_{i=1}^n \delta_{p^i_m}$ and $\sigma_m\approx \sum_{i=1}^n s^i_m\,\delta_{p^i_m}$
    \State Discretize the kernel on the same grid to obtain $k_{\mathrm{grid}} \approx k\,\mathbf{1}_{|\cdot|\le R}$
    \State Compute $k_{\mathrm{grid}}*\rho_m$ and $k_{\mathrm{grid}}*\sigma_m$ by FFT
    \State Compute $k_{\mathrm{grid}}'*\rho_m$ by spectral differentiation in Fourier space (applied to $k_{\mathrm{grid}}*\rho_m$)
    \State Interpolate $k_{\mathrm{grid}}*\rho_m,\ k_{\mathrm{grid}}*\sigma_m,$ and $\ k_{\mathrm{grid}}'*\rho_m$ back to $\{p^i_m\}$ to obtain approximations of $\sum_{j=1}^n k(p_m^i-p_m^j),\ \sum_{j=1}^n k(p_m^i-p_m^j)s_m^j,$ and $\sum_{j=1}^n k'(p_m^i-p_m^j)$
    \State Define $v_m$ as in~\eqref{eq:KDRW_velocity} with these approximations
\end{algorithmic}
\end{routine}

More precisely, if $\{z^{\ell}\}$ are the grid points, for each particle $p_m^i$, letting $\ell^i$ be the unique index such that $z^{\ell^i} \le p_m^i < z^{\ell^i+1}$, we define linear weights $w_0^i$,$w_1^i\in[0,1]$ with $w_0^i+w_1^i=1$ so that $p_m^i=w_0^iz^{\ell^i} + w_1^iz^{\ell^i+1}.$  We then build $\rho_m$ and $\sigma_m$ by initializing both arrays to $0$, and then depositing mass and score according to $\rho_m[\ell^i] \mathrel{+}= w_0^i,\,\rho_m[\ell^i+1] \mathrel{+}= w_1^i,\,
\sigma_m[\ell^i] \mathrel{+}= w_0^i\,s_m^i$, and $\sigma_m[\ell^i+1] \mathrel{+}= w_1^i\,s_m^i$ for each $i$. After the FFT-based convolutions, we evaluate the resulting grid fields at particle locations using the same weights, e.g.\ $\sum_{j=1}^nk(p^i_m-p^j_m)\approx w_0^i\,k_{\mathrm{grid}}*\rho_m[\ell^i] + w_1^i\,k_{\mathrm{grid}}*\rho_m[\ell^i+1],$ and similarly for $\sum_{j=1}^nk(p^i_m-p^j_m)S(p_m^j)$ and $\sum_{j=1}^n k'(p^i_m-p^j_m)$.

As long as $k$ and $U$ are sufficiently smooth, the FFT approximation quickly converges to the true value as $L,M\rightarrow\infty$. This method thus presents a way to substantially accelerate the speed of the algorithm with only small divergence from the exact computations. See Figures~\ref{fig:banana_fft} to see the similarity between samples generated using Routine~\ref{rout:KDRW} and~\ref{rout:KDRW_fft}, and~\ref{fig:bw} to see nearly identical performance for modest values of $L$ and $M$.


\subsection{Interacting particle approximation of RRW flow}  \label{sec:IPSRRW}

For the RRW, given a time step $\tau$,~\eqref{eq:RRWgf} is naturally approximated by the stochastic system
\begin{equation}\label{eq:RRW_sgd}
    x_{m+1}^i = x_m^i - \tau \theta_m\!\int_\R \left(\frac{\sum_{j=1}^n k'(p_m^i-p -p_m^j) + \sum_{j=1}^n k(p_m^i-p -p_m^j)\nabla_{\theta_m} U(x_m^j)}{\sum_{j=1}^n k( p_m^i-p-p_m^j)+n\eps}\right)k(p)\,dp.
\end{equation}
where again $i=1,...,n$, $\theta_m$ is sampled uniformly on $\S^{d-1}$, and $p_m^i = x_m^i \cdot \theta_m$. Letting
\[\tilde{v}(p):=\frac{\sum_{j=1}^n k'(p -p_m^j)+k(p -p_m^j)\nabla_{\theta_m} U(x_m^j)}{\sum_{j=1}^n k(p-p_m^j)+n\eps},\]
this corresponds to Algorithm~\ref{algo:skeleton} with $v_m$ defined by $v_m^i:=k*\tilde v(p^i_m)$. Unlike for the approximate KDRW flow, there is no naive implementation to compute this scalar velocity due to the additional convolution in space by $k$. For this reason, we introduce two approximate solutions: one that uses a crude approximation of the convolution, and one that uses spatial discretizations and the FFT, similar to Routine~\ref{rout:KDRW_fft}.

First, as an alternative to exactly computing the convolution defining $v_m$, we replace it by a local kernel average of $\tilde{v}$ over the projected particle locations:
\[ k*\tilde v(p_m^i)\approx \frac{\sum_{j=1}^n k(p_m^i-p_m^j)\tilde{v}(p_m^j)}{\sum_{j=1}^n k(p_m^i-p_m^j)}.\]
This approximation can be interpreted as replacing $k*\tilde{v}$ by the ratio $(k*(\tilde{v}\rho_m))/(k*\rho_m)$ where 
$\rho_m=\sum_{j=1}^n\delta_{p_m^j}$, and is therefore biased when the projected particle density varies significantly on the scale of $k$.
Nevertheless, it reduces the extra convolution in~\eqref{eq:RRW_sgd} to particle sums involving only differences  $p_m^i-p_m^j$, leading to the routine below.

\begin{routine}[h]
\caption{RRW scalar velocity}\label{rout:RRW}
\textbf{Input:} Projected particles $\{p^i_m\}_{i=1}^n$, projected score vector $\{s^i_m\}_{i=1}^n$\\
\textbf{Output:} Scalar velocity $v_m$
\begin{algorithmic}[1]
    \State Compute ``kernel matrix'' $K_{m}$, whose entries are $K_m^{i,j} = k (p_m^i-p_m^j)$ and ``derivative kernel matrix''  $K'_m$, whose entries are $(K_m')^{i,j} = k'(p_m^i-p_m^j)$
    \State Compute row sums of $K_m$ and $K_m'$ and denote them by $\kappa_m$ and $\kappa'_m$
    \State Compute $\tilde{v}_m=(\kappa_m' - K_m s_m)\oslash(\kappa_m+n\eps)$
    \State Compute $v_m=K_m\tilde{v}_m\oslash\kappa_m$
\end{algorithmic}
\end{routine}

It is clear that this approximation is not very accurate as this will not give a good estimation of the scalar velocity for particles outside the bulk of the point cloud. Instead, by using spatial discretization and the FFT to compute the convolution, we can give a more accurate approximation of $v_m^i$.

\begin{routine}[h]
\caption{RRW FFT scalar velocity}\label{rout:RRW_fft}
\textbf{Input:} Projected particles $\{p^i_m\}_{i=1}^n$, projected score vector $\{s^i_m\}_{i=1}^n$, cutoff $L$, and discretization parameter $M$\\
\textbf{Output:} Scalar velocity $v_m$
\begin{algorithmic}[1]
    \State Set $R := Lb$, $h := b/M$, and $[a,c] := [\min_i p^i_m - 2R,\ \max_i p^i_m + 2R]$
    \State Build a uniform grid on $[a,c]$ with spacing $h$, and deposit $\{p^i_m\}$ and $\{s^i_m\}$ to obtain grid arrays
    $\rho_m\approx \sum_{i=1}^n \delta_{p^i_m}$ and $\sigma_m\approx \sum_{i=1}^n s^i_m\,\delta_{p^i_m}$
    \State Discretize the kernel on the same grid to obtain $k_{\mathrm{grid}} \approx k\,\mathbf{1}_{|\cdot|\le R}$
    \State Compute $k_{\mathrm{grid}}*\rho_m$ and $k_{\mathrm{grid}}*\sigma_m$ by FFT
    \State Compute $k_{\mathrm{grid}}'*\rho_m$ by spectral differentiation in Fourier space (applied to $k_{\mathrm{grid}}*\rho_m$)
    \State Compute $\tilde{v}_{m,\mathrm{grid}}=(k_{\mathrm{grid}}'*\rho_m-k_{\mathrm{grid}}*\sigma_m)/(k_{\mathrm{grid}}*\rho_m+n\eps) $
    \State Compute $v_{m,\mathrm{grid}}=k_{\mathrm{grid}}*\tilde{v}_{m,\mathrm{grid}}$ by FFT
    \State Interpolate $v_{m,\mathrm{grid}}$ back to $\{p^i_m\}$ to obtain $v_m$
\end{algorithmic}
\end{routine}

This is almost identical to the FFT implementation outlined in Routine~\ref{rout:KDRW_fft} except we perform an extra convolution before interpolating back from the discretized values and have padded $[a,c]$ by $2R$ instead of $R$ to eliminate wrap-around error due to the double convolutions. We emphasize that this gives a more faithful approximation of $v_m$ than Routine~\ref{rout:RRW}. This is consistent with our empirical results: the implementation with Routine~\ref{rout:RRW} generally performs worse than with Routine~\ref{rout:RRW_fft}, see Figure~\ref{fig:bw}.

\subsection{Complexity}\label{subsec:comp}

In this subsection, we compare the per-step computational complexity of Algorithm~\ref{algo:skeleton} across the scalar velocity subroutines in Routines~\ref{rout:KDRW}--\ref{rout:RRW_fft}.

First, computing one step of Algorithm~\ref{algo:skeleton} with scalar velocity determined by Routine~\ref{rout:KDRW} or~\ref{rout:RRW}---KDRW or RRW with direct convolution computations---takes $O(nd + n^2)$ operations. Computing the directional derivatives and projections has complexity $O(nd)$ while computing the scalar velocity takes has $O(n^2)$ operations. This complexity is already much smaller than the $O(dn^2)$ operations per step of  SVGD. Additionally, the complexity of the convolution steps can be reduced further: if $k$ has bounded support then one only needs to account for $k(p_m^i-p_m^j)$ when $p_m^i - p_m^j$ is in the support of $k$. If the radius of the support of $k$ is very small, there are few such neighbors, and thus the complexity is much lower. For example, if the radius of the support is taken to scale like $n^{-\alpha}$ for some $\alpha$ (corresponding, for example, to classical choices for kernel density estimation such as $\alpha=1/5$ \cite{Tsybakov09,Wasserman06}), then the Routines \ref{rout:KDRW} and \ref{rout:RRW} typically have complexity $O(n^{2-\alpha})$.

Each step of the FFT-based routines---Routine~\ref{rout:KDRW_fft} and~\ref{rout:RRW_fft}---takes $O(n+G\log G)$ operations where $G\approx (c-a)/h$ is the number of grid-points. If a hard cap $G\le G_{\max}$ is additionally enforced (for example, when the target measure has bounded support and $b$ is fixed), then the cost of computing the convolution using the FFT is uniformly bounded and the per-step complexity of the Routine is just $O(n)$. If $b$ is instead taken to scale like $n^{-\alpha}$, then the per-step complexity becomes $O(n+ n^{\alpha}\log n)$. Thus if $\alpha<1$ and $c-a$ stays bounded, then the overall complexity of a full step of Algorithm~\ref{algo:skeleton} with velocities computed according to Routines \ref{rout:KDRW_fft} (KDRW with the FFT) or Routine \ref{rout:RRW_fft} (RRW with the FFT) is $O(nd)$.

In Subsection~\ref{subsec:obs_comp} we numerically investigate the time it takes to compute the update in practice.

\subsection{Alternative acceleration for Laplace kernel}  \label{sec:IPSextensions}

If we choose $k$ to be the Laplace kernel $k(p) = e^{-|p|}$ then there is an alternative routine to compute the velocity field with complexity $O(nd+n \log n)$ that does not use the FFT. This relies on ideas that go back to the Fast Multipole Method of Greengard and Rokhlin \cite{greengard87}. The particular recurrence relations relevant to Laplace kernel, for integrals on the line, appear in the work of Yarvin and Rokhlin \cite{Yarvin1999SOE}. This approach to computing convolutions can be extended, by approximation, to general kernels; it was already used in \cite{Yarvin1999SOE} and is generalized by Zhang, Zhuang, and Jiang \cite{Zhang21SOE}.

For brevity, we only present the algorithm for evaluating $k*\rho$ for $\rho=\frac{1}{n} \sum_{i=1}^n \delta_{p^i}$. The convolutions $k' * \rho$ and $k*(\rho\nabla_\theta U)$ are evaluated analogously. 

The first step is to sort the points so that $p^1\leq p^2 \cdots \leq p^n$. This takes $O(n \log n)$ operations. We then note that
\[ \sum_{i=1}^n e^{-|p^j-p^i|} = \sum_{i\le j} e^{-(p^j-p^i)} + \sum_{i\ge j} e^{-(p^i-p^j)} -1.
\]
Define the left and right sums $
L^j := \sum_{i\le j} e^{-(p^j-p^i)}$ and
$ R^j := \sum_{i\ge j} e^{-(p^i-p^j)}$ so that $L^1=R^n=1$ and
\begin{align} \label{eq:recrelLap}
\begin{split}
L^j &= 1 + e^{-(p^j-p^{j-1})}\,L^{j-1},\quad j\geq 2,\\
R^j &= 1 + e^{-(p^{j+1}-p^j)}\,R^{j+1},\quad j\leq n-1.
\end{split}
\end{align}
Thus $L^j$ and $R^j$  can be computed for all $j$ in $O(n)$ operations. We conclude by observing that
\[
(k * \rho)(p^j) = \frac{1}{n}\left(L^j + R^j - 1\right),
\]
hence $k*\rho(p^j)$ can be computed for all $j$ in $O(n\log n)$ operations.

\section{Experimental Results} \label{sec:experiments}

We performed numerous experiments using Algorithm~\ref{algo:skeleton} with Routines~\ref{rout:KDRW}-\ref{rout:RRW_fft} for sampling from  distributions in dimensions ranging from 2 to 2048. The corresponding Python code is available at~\cite{HCSX26Git}. Overall, we observe that these algorithms are able to quickly capture the large-scale structure of target distributions and converge efficiently for distributions which have a single well, even if their geometry is nontrivial.

Throughout this section and within our experiments, we  use the following conventions. KDRW and RRW respectively refer to Algorithm~\ref{algo:skeleton} with the scalar velocity determined by Routine~\ref{rout:KDRW} and Routine~\ref{rout:RRW}. KDRW\!\textunderscore fft and RRW\!\textunderscore fft respectively refer to the Algorithm~\ref{algo:skeleton} with the scalar velocity determined by Routine~\ref{rout:KDRW_fft} and Routine~\ref{rout:RRW_fft}. We take the regularizing kernel $k$ to be the Gaussian probability density with spatial bandwidth $b$, that is $k(p)\propto \exp({-|p|^2/2b^2}).$ We set the small parameter $\eps$ to equal $0.01/n$ or $0.02/n$ where $n$ is the number of particles. For the FFT based algorithms, we set the cuttoff parameter to be $L=5$, and the number of grid points per bandwidth to be $M=8$ or $10$.

We use two target distributions, standard Gaussians and Rosenbrock ``banana" distributions~\cite{pagani2022n}. For the latter we always take the potential function to equal
\begin{equation} \label{eq:banana_potential}
  U(x) = \frac{1}{2}\sum_{i=1}^{d-1} ( x^i)^2
+ 8 \left(x_d - \frac{1}{2\sqrt{d}} \left(\sum_{i=1}^{d-1} (x^i)^2 - (d-1)\right) \right)^2.  
\end{equation}
The banana distributions thus has a single energy well, but is not log-concave. 

To measure how close the generated samples are to the target density we rely on the $\MMD$ distance with gaussian kernel with spacial bandwidth $\sqrt{d}$. We use this distance since we can exactly compute the MMD between a discrete measure (the empirical measure of the samples, $\mu^n$) and a standard Gaussian. Additionally, the $\MMD$ distance is meaningful in high dimensions, as the expected error for $n$ i.i.d.\ samples scales like $\frac{1}{\sqrt n}$ \cite{sriperumbudur2016optimal}. To measure the error for banana target distributions $\pi$ we use an explicitly invertible transformation $T$ that pushes forward the standard normal distribution in $\R^d$ to $\pi$. That is $\pi = T\# \gamma$ where $\gamma$ is a standard normal. The error we report in the figures below is then $\MMD_T^2(\mu^n, \pi) = \MMD^2(T^{-1}\# \mu^n, \gamma)$. This allows us to explicitly measure the error for high-dimensional targets that are not log concave.

\begin{figure}[ht]
\centering
\begin{subfigure}{0.45\textwidth}
  \centering
  \includegraphics[width=\linewidth]{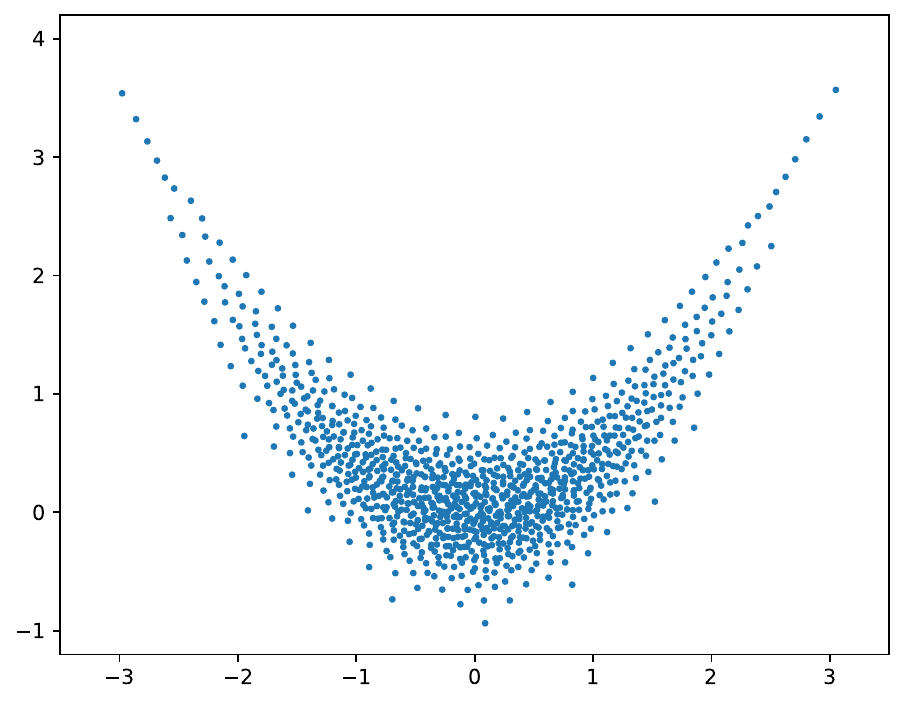}
  \caption{KDRW, $b=0.1$}
\end{subfigure}\hspace*{10pt}
\begin{subfigure}{0.45\textwidth}
  \centering
  \includegraphics[width=\linewidth]{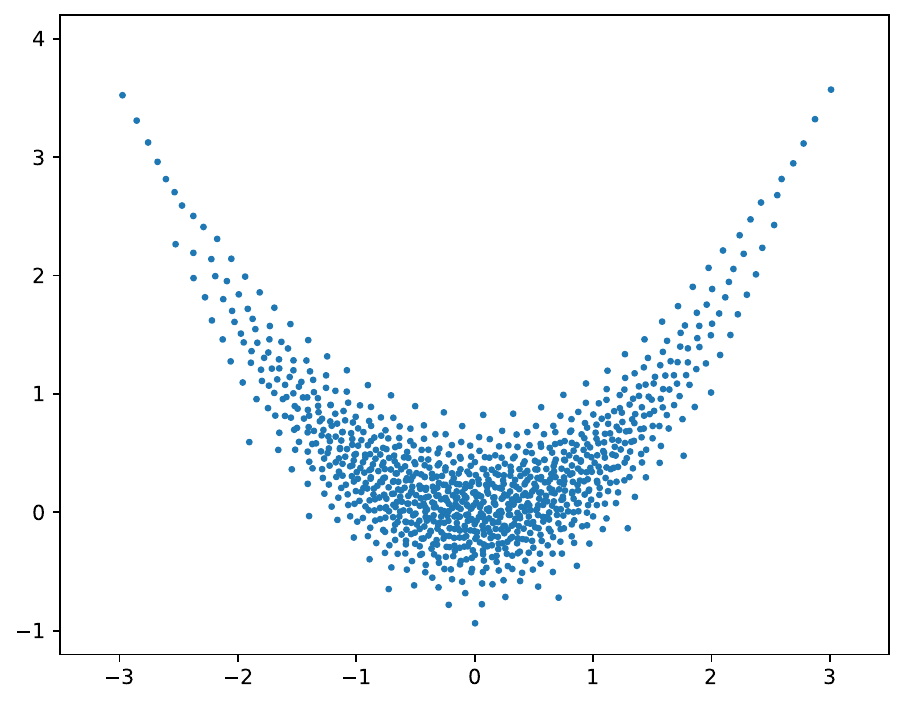}
  \caption{KDRW\!\textunderscore fft, $b=0.1$}
\end{subfigure}

\begin{subfigure}{0.45\textwidth}
  \centering
  \includegraphics[width=\linewidth]{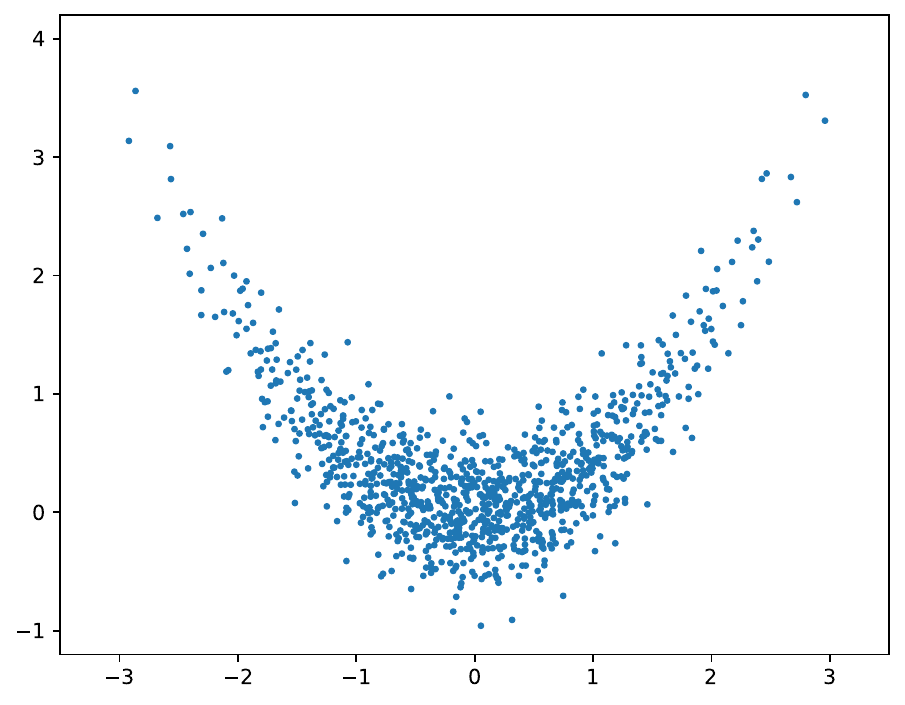}
  \caption{KDRW, $b=0.5$}
\end{subfigure}\hspace*{10pt}
\begin{subfigure}{0.45\textwidth}
  \centering
  \includegraphics[width=\linewidth]{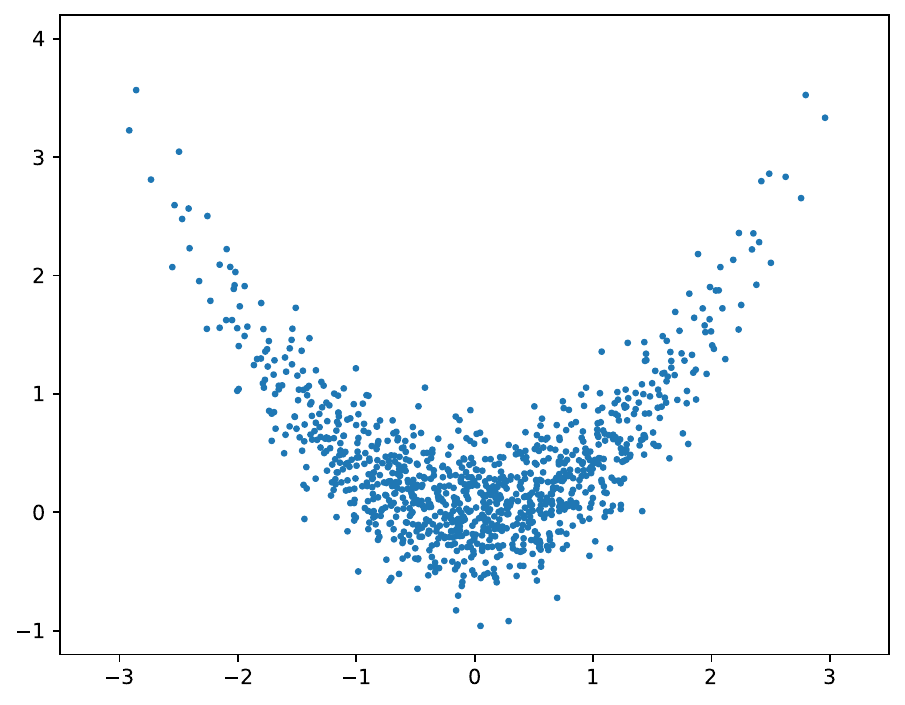}
  \caption{KDRW\!\textunderscore fft, $b=0.5$}
\end{subfigure}
     \caption{Samples from a banana distribution with potential~\eqref{eq:banana_potential} generated by both KDRW and KDRW\!\textunderscore fft with different bandwidths $b$ and $n=1024$ particles. Initial particles were sampled with the same i.i.d.\ sample from a standard Gaussian.}
\label{fig:banana_fft}
\end{figure}

\subsection{Influence of the bandwidth and performance of FFT-based algorithms}\label{sec:bandwidths}

Theoretically, for KDRW/KDRW\!\textunderscore fft, the bias of the final state decreases as the bandwidth $b$ is reduced. In practice, however, how small $b$ can be taken is limited by the need to accurately approximate the score of the projected measures. By contrast, RRW/RRW\!\textunderscore fft are unbiased. That said, when approximating the continuum flow by discrete measures we do expect the approximation error  to depend on $b$. For both dynamics, high-frequency information is dampened by the regularizing kernel, leading to a loss of resolution.

To investigate this effect, Figure~\ref{fig:bw} reports the $\mathrm{MMD}^2$ error for samples generated across a wide range of bandwidths using $1024$ particles, for Gaussian and banana target distributions in $d=2$ and $d=256$. Each run was initialized with i.i.d.\ samples from a standard Gaussian with mean shifted by 2 in the first coordinate, and was run for $50,000$ steps with step size $0.01$ when $d=2$ and $0.1$ when $d=256$. Results for the RW algorithms are averaged over 8 trials. For each target distribution the i.i.d.\ baseline reports the average 
$\mathrm{MMD}^2$ error of the empirical measure of $1024$ i.i.d.\ samples from the target distribution over 50 trials.

\begin{figure}[!t]
\centering
\captionsetup{font=small}

\begin{subfigure}{0.45\textwidth}
  \centering
  \includegraphics[width=\linewidth]{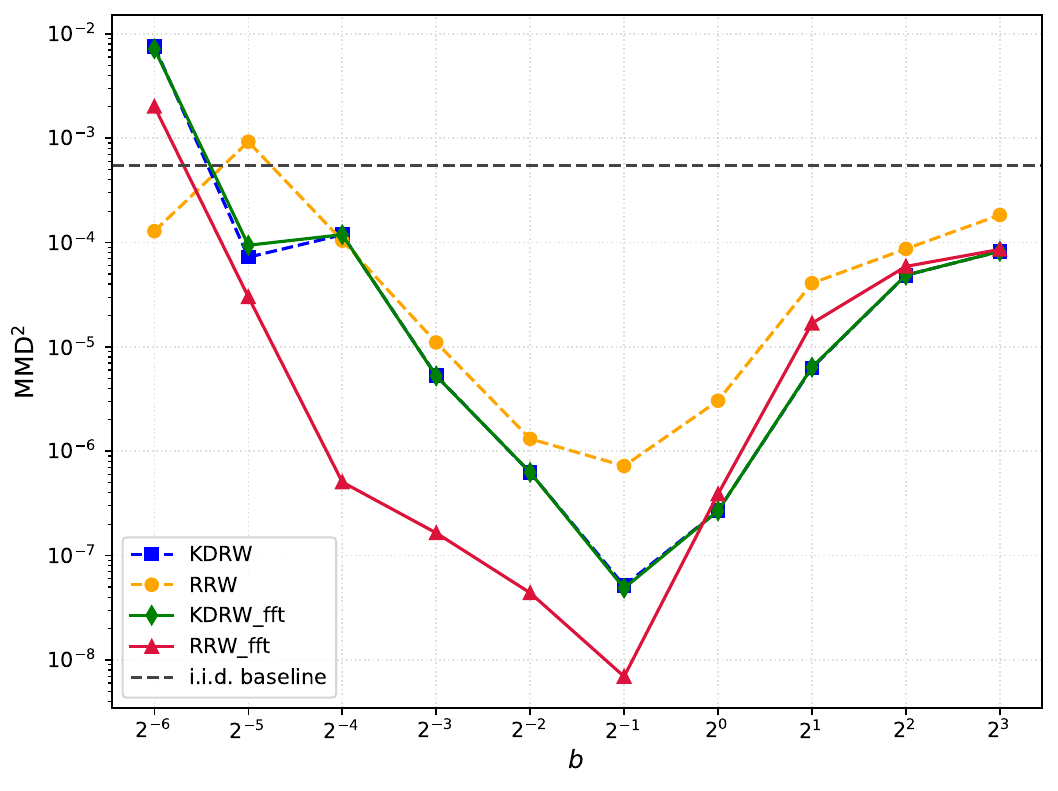}
  \caption{$d=2$, Gaussian target}
  \label{fig:bw_gauss_d2}
\end{subfigure}\hspace*{10pt}
\begin{subfigure}{0.45\textwidth}
  \centering
  \includegraphics[width=\linewidth]{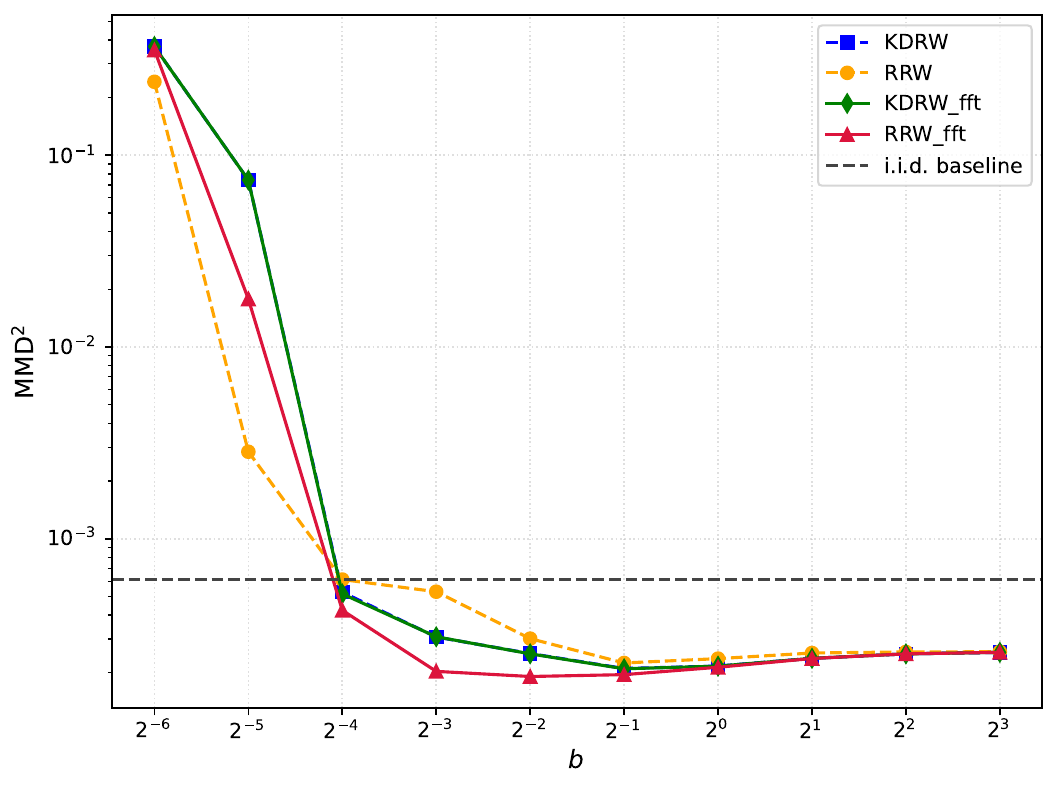}
  \caption{$d=256$, Gaussian target}
  \label{fig:bw_gauss_d256}
\end{subfigure}

\begin{subfigure}{0.45\textwidth}
  \centering
  \includegraphics[width=\linewidth]{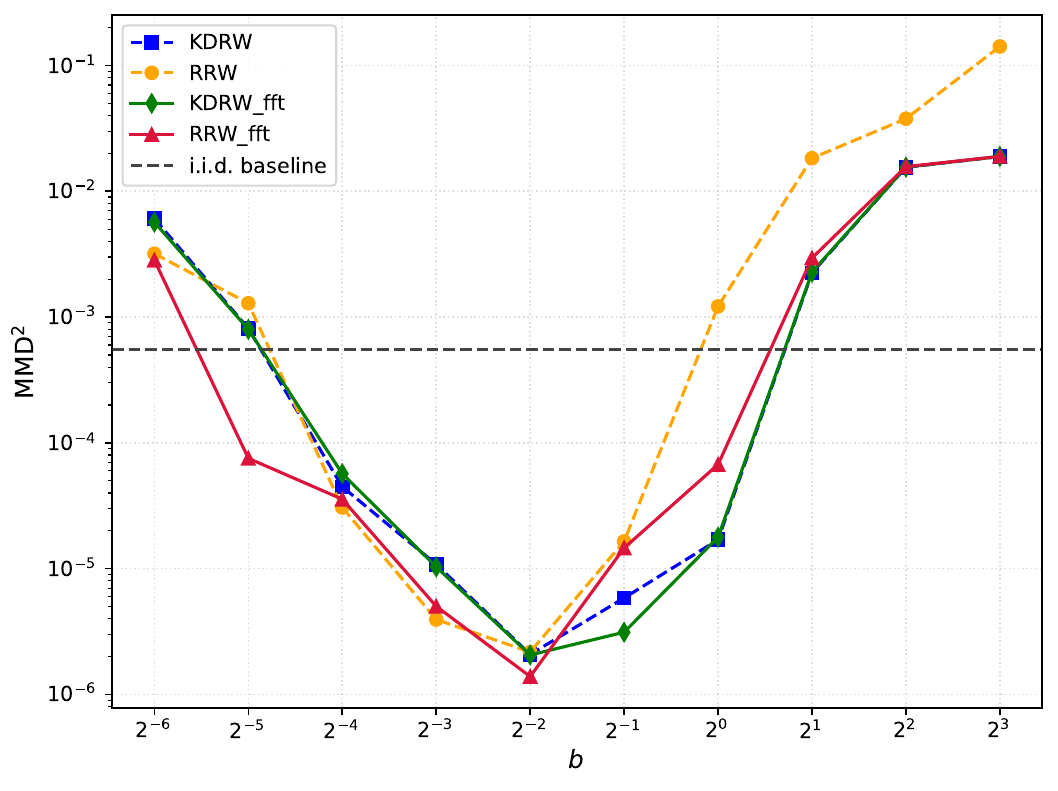}
  \caption{$d=2$, banana target}
  \label{fig:bw_banana_d2}
\end{subfigure}\hspace*{10pt}
\begin{subfigure}{0.45\textwidth}
  \centering
  \includegraphics[width=\linewidth]{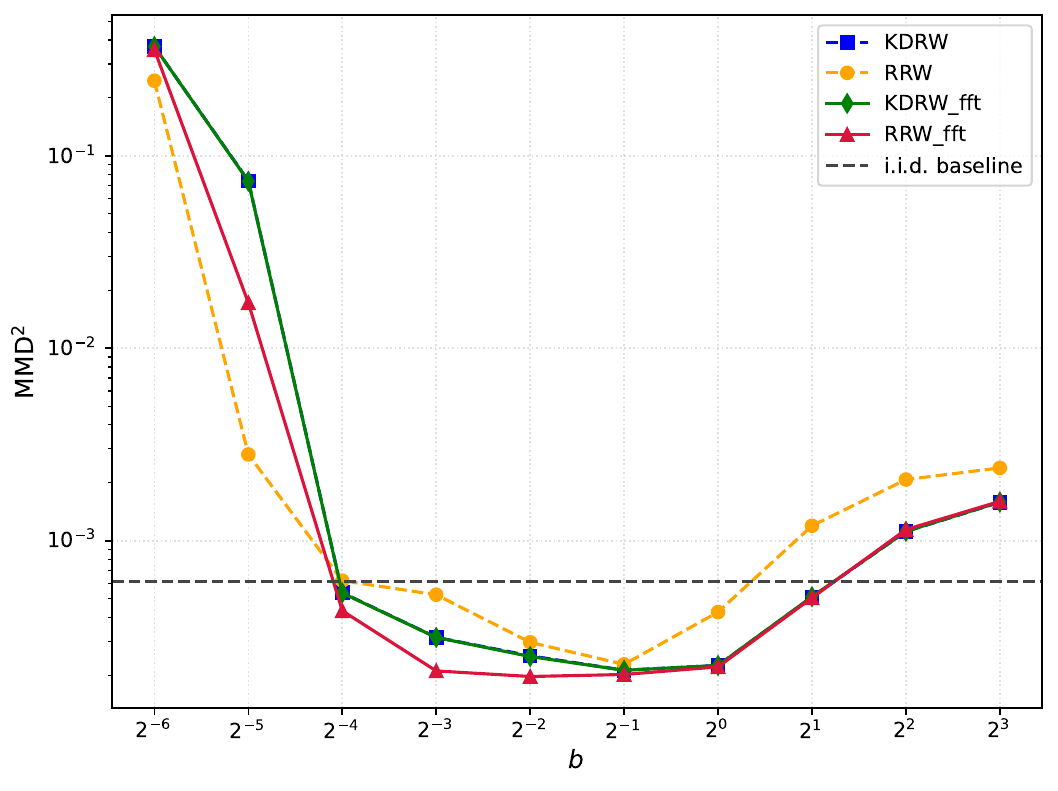}
  \caption{$d=256$, banana target}
  \label{fig:bw_banana_d256}
\end{subfigure}

\caption{$\MMD^2$ error for samples generated with different bandwidths $b$ and $n=1024$ particles. The target distribution is a standard Gaussian in the top row, and a Rosenbrock ``banana'' distribution with potential~\eqref{eq:banana_potential} in the bottom row.}
\label{fig:bw}
\end{figure}

We observe a characteristic $U$-shaped dependence on $b$: as $b$ decreases, performance initially improves, but once $b$ becomes too small the error increases rapidly. In higher dimensions the curve is noticeably flatter, with a broad range of $b$ yielding similar performance. Moreover, across a broad range of bandwidths, the RW algorithms return samples with lower $\mathrm{MMD}^2$ error than an i.i.d.\ sample with the same number of particles.

The performance of KDRW\!\textunderscore fft is nearly identical to that of the full KDRW scheme, even in high dimensions. This close agreement is also evident in the qualitative comparisons in Figure~\ref{fig:banana_fft}. In contrast, RRW\!\textunderscore fft consistently outperforms RRW. This is unsurprising as RRW uses a cruder approximation of the convolution with $k$. Because the FFT-based implementations are substantially faster (and more accurate for RRW), we use only the FFT-based versions of the algorithms in the remaining experiments.

\subsubsection{Bandwidth selection in general setting}

In addition to the experiments of Figure \ref{fig:bw}, we considered how the optimal bandwidth changes with the number of particles $n$. We found that the following simple scaling laws were nearly optimal for KDRW\!\textunderscore fft and RRW\!\textunderscore fft in dimensions 2, 32, and 256 for a standard Gaussian target 
\begin{equation} \label{bw:scaled_with_n}
    b_{\text{KDRW}} = 2 n^{-1/5} \quad \text{ and } \quad b_{\text{RRW}} = n^{-1/5}.
\end{equation}
We note that this scaling agrees with the optimal scaling for kernel density estimation in one dimension \cite{Tsybakov09}.

The observed scaling above motivates the following proposal for adaptive bandwidths for general target measures with arbitrary variance.  Given a configuration of particles $x^1, \dots, x^n$ and angle $\theta \in \S^{d-1}$, then letting 
$p^i = x^i \cdot \theta$, $\overline p = \frac{1}{n}\sum_{i=1}^n p^i$
and $\sigma = \left( \frac{1}{n} \sum_{i=1}^n (p^i - p)^2 \right)^{1/2}$, we propose the adaptive bandwidths
\begin{equation} \label{bw:adaptive}
    b_{ada, \text{KDRW}} = 2 \sigma  n^{-1/5} \quad \text{ and } \quad b_{ada,\text{RRW}} = \sigma  n^{-1/5}.
\end{equation}
As the banana distribution~\eqref{eq:banana_potential} has $\sigma\approx 0.8$ when $d=2$ and $\sigma\approx 1.0$ when $d=256$, these recommendations are consistent with Figure~\ref{fig:bw}. In particular, the slight shift between the curves in~\ref{fig:bw_gauss_d2} and~\ref{fig:bw_banana_d2}.

\subsection{Observed computational complexity of algorithms}\label{subsec:obs_comp}

As discussed in Subsection~\ref{subsec:comp}, the complexity of each step of KDRW\!\textunderscore fft and RRW\!\textunderscore fft is $O(dn)$. In practice, the fixed cost of some operations, memory management, and parallelism in modern CPUs and GPUs affect this scaling. Figure~\ref{fig:tps} reports the actual time to compute a single step of RRW\!\textunderscore fft for a banana target distribution using our code on a 2024 MacBook Pro over various particle numbers and dimensions. Although these values are dependent on the hardware, implementation, and target measure used, we believe they still present a valuable indication of the expected performance. In particular, they suggest that these algorithms remain feasible for large particle counts and high dimensions.

\begin{figure}[!t]
  \centering
  \begin{subfigure}{0.49\textwidth}
    \centering
    \includegraphics[width=\linewidth]{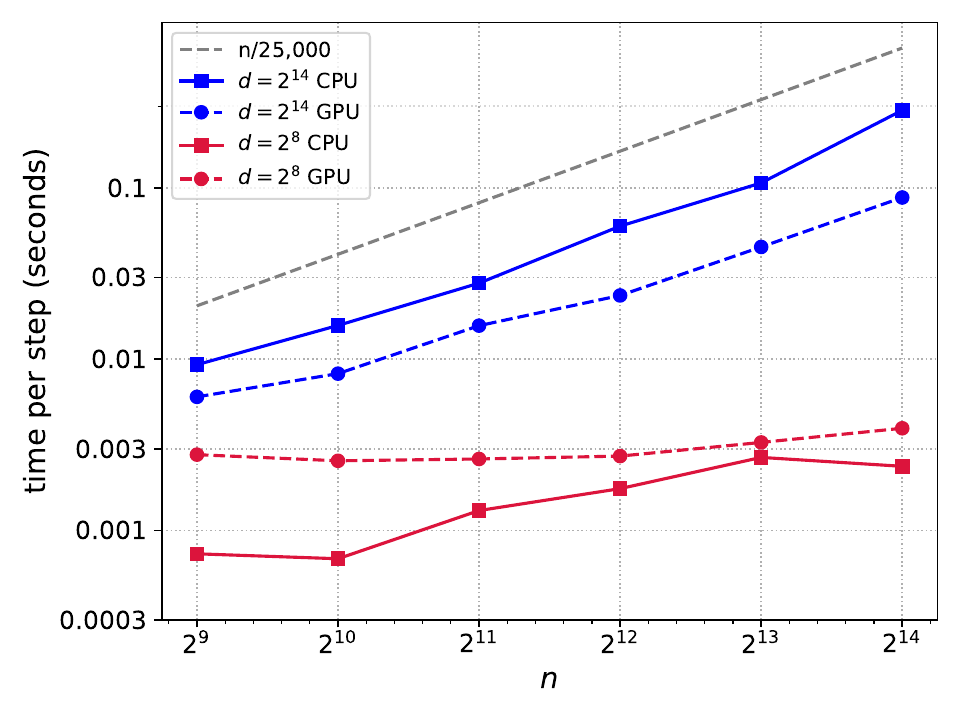}
    \caption{Time per step as a function of $n$}
  \end{subfigure}
  \begin{subfigure}{0.49\textwidth}
    \centering
    \includegraphics[width=\linewidth]{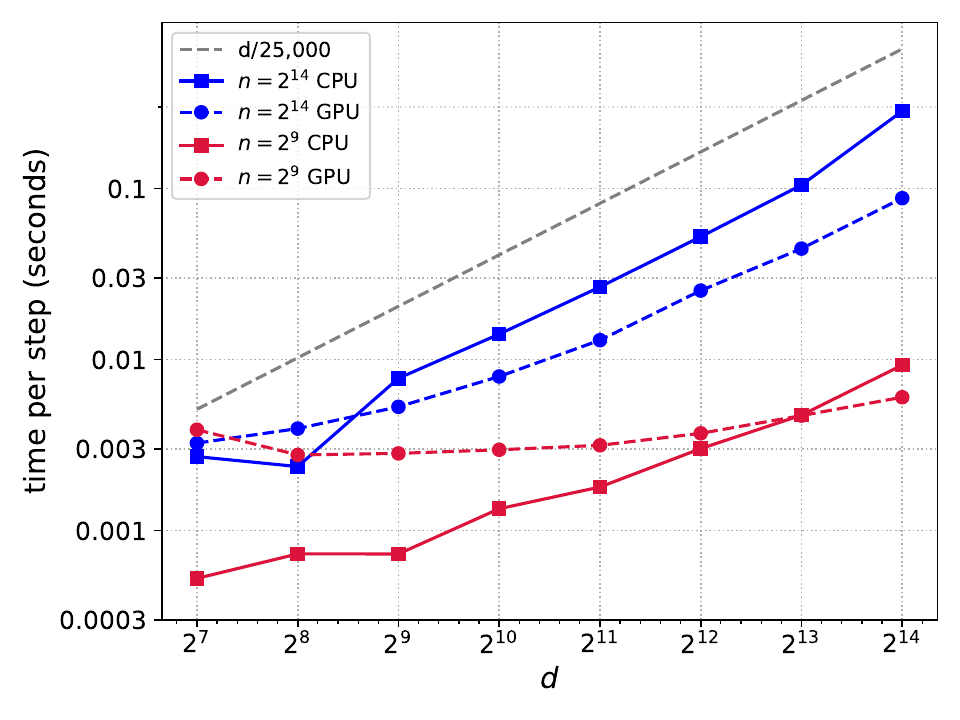}
    \caption{Time per step as a function of $d$}
  \end{subfigure}
  \medskip
  \caption{Time per step (seconds) for RRW\!\textunderscore fft on CPU vs.\ GPU for different particle numbers $n$ and dimensions $d$.} \label{fig:tps}
\end{figure}

\subsection{Observed Convergence in time of the flows}\label{sec:convergence}

To study the convergence behavior of the algorithms over time, in Figure~\ref{fig:convergencebanana} we plot the $\mathrm{MMD}^2$ error versus equation time, i.e. the cumulative step size $t$, for KDRW\!\textunderscore fft, RRW\!\textunderscore fft and SVGD for varying particle numbers $n$ and dimensions $d$. In all cases, the target density is the banana distribution, and each run was initialized with i.i.d.\ samples from a Gaussian with mean shifted by 1 in the first coordinate and covariance matrix $0.25I_d$.

After an initial warm-up period (with smaller step sizes), the RW algorithms and SVGD were run with respective step sizes equal to $0.005$ and $0.1$ when $d=2$, $0.1$ and $0.2$ when $d=32$, and identical step size of $0.1$ or $0.2$ in all other dimensions. The adaptive bandwidths~\eqref{bw:adaptive} were used for KDRW\!\textunderscore fft, RRW\!\textunderscore fft, while the median bandwidth $b=\text{median}_{i<j}|x_i-x_j|$ was used for SVGD. Each algorithm was run a single time for each value of $n$ and $d$, and the i.i.d.\ baselines report the $\mathrm{MMD}^2$ error of the empirical measure of a single i.i.d.\ sample of $n$ particles from the target distribution.

\begin{figure}[!htbp]
  \centering
  \begin{subfigure}{0.4\textwidth}
    \centering
    \includegraphics[width=\linewidth]{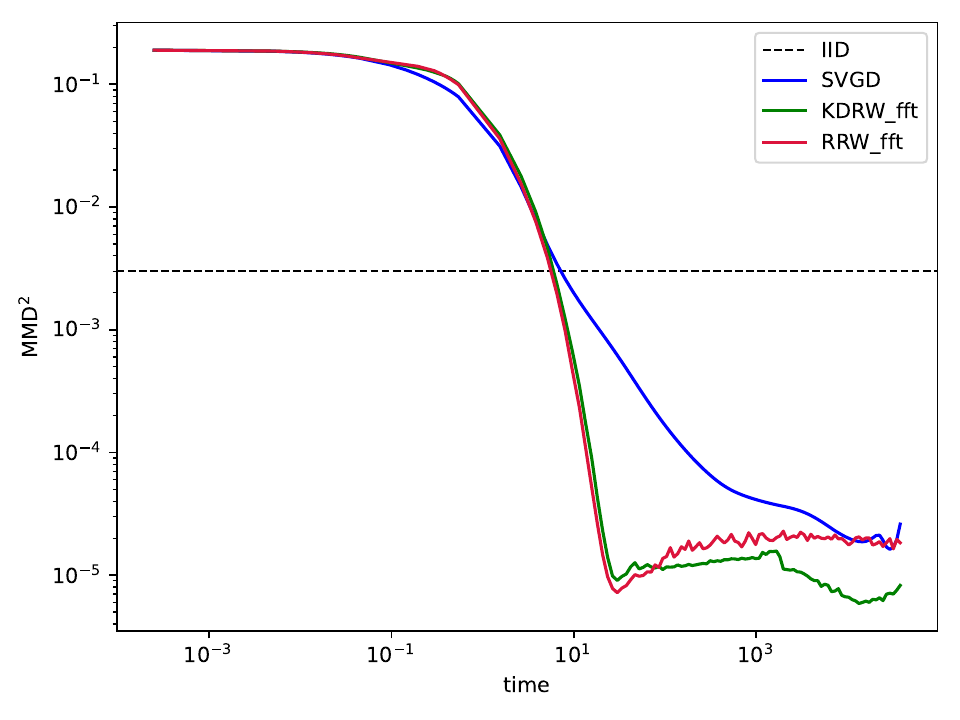}
    \caption{$d=2$, $n=256$}
    \label{fig:convergence2}
  \end{subfigure}
  \hspace*{16pt}
  \begin{subfigure}{0.4\textwidth}
    \centering
    \includegraphics[width=\linewidth]
    {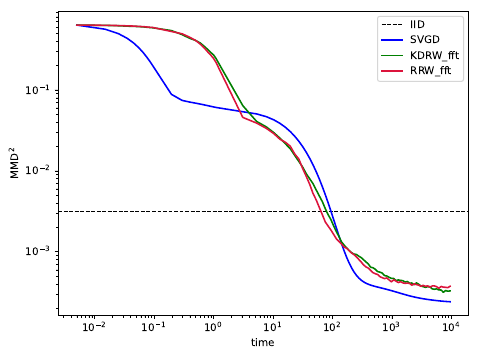}
    \caption{$d=32$, $n=256$}
    \label{fig:convergence32}
  \end{subfigure}
  \medskip

  \begin{subfigure}{0.4\textwidth}
    \centering
    \includegraphics[width=\linewidth]
    {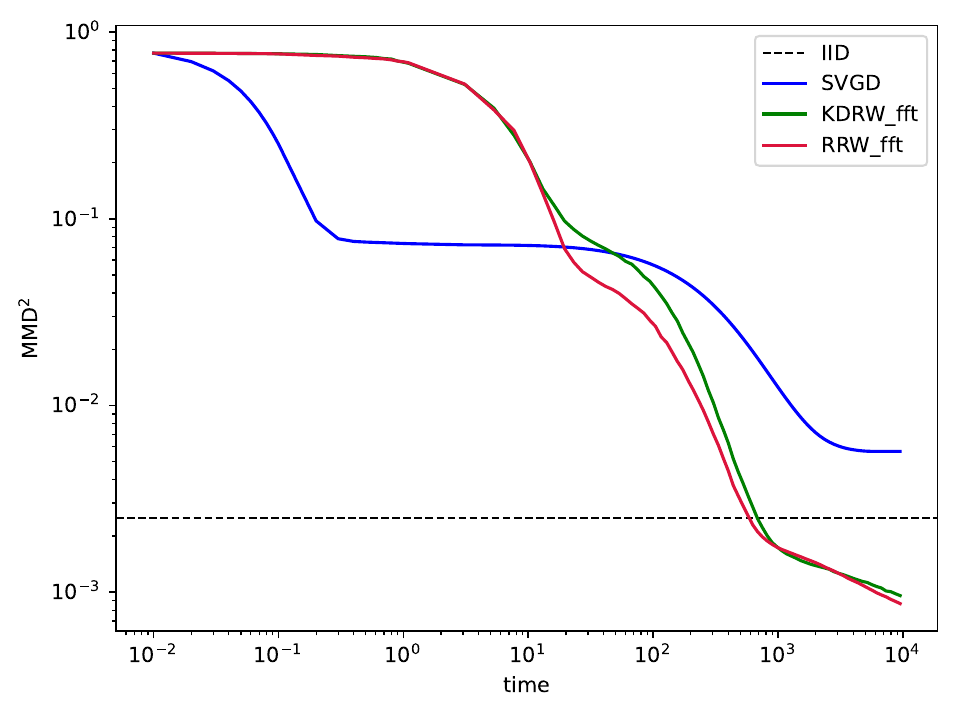}
    \caption{$d=256$, $n=256$}
    \label{fig:convergence256256}
  \end{subfigure}
  \hspace*{16pt}
  \begin{subfigure}{0.4\textwidth}
    \centering
    \includegraphics[width=\linewidth]
    {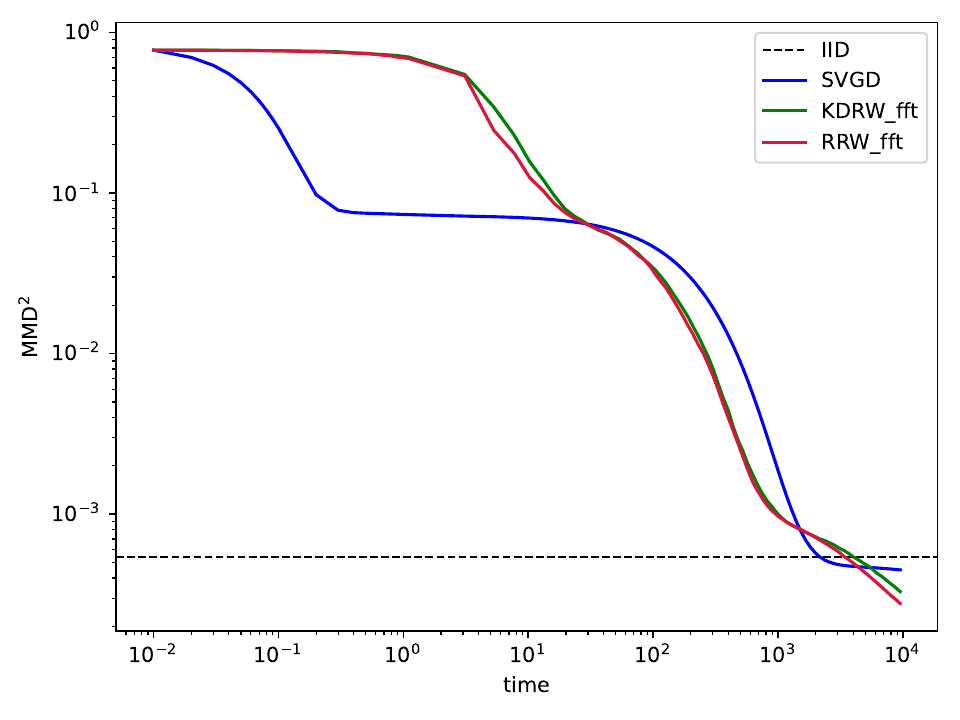}
    \caption{$d=256$, $n=1024$}
    \label{fig:convergence2561024}
  \end{subfigure}
   \medskip

  \begin{subfigure}{0.4\textwidth}
    \centering
    \includegraphics[width=\linewidth]
    {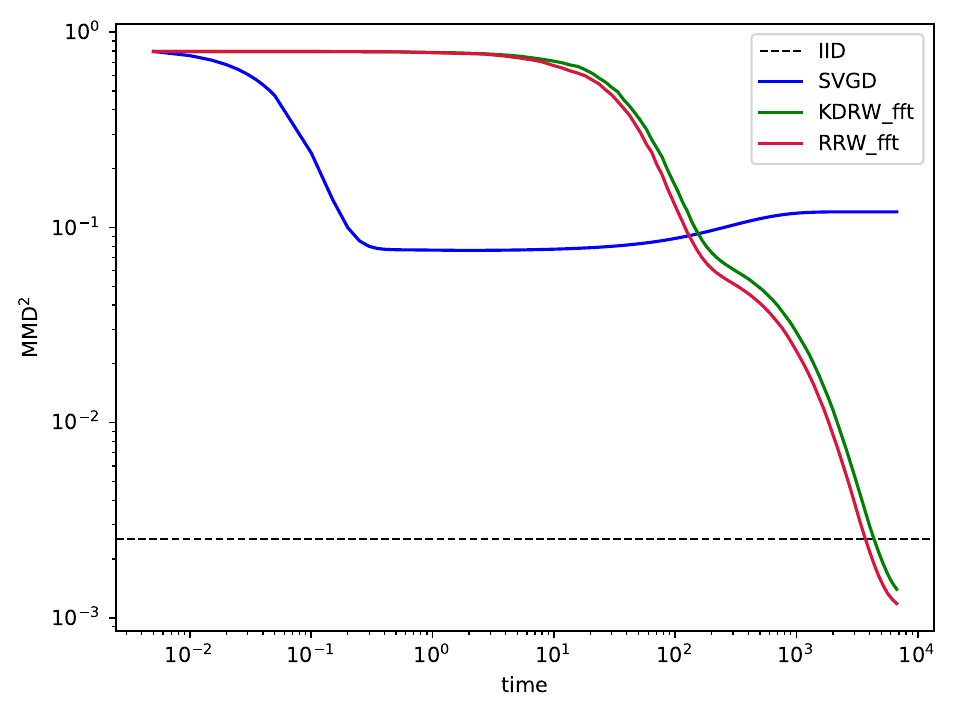}
    \caption{$d=2048$, $n=256$}
    \label{fig:convergence2048256}
  \end{subfigure}
  \hspace*{16pt}
  \begin{subfigure}{0.4\textwidth}
    \centering
    \includegraphics[width=\linewidth]
    {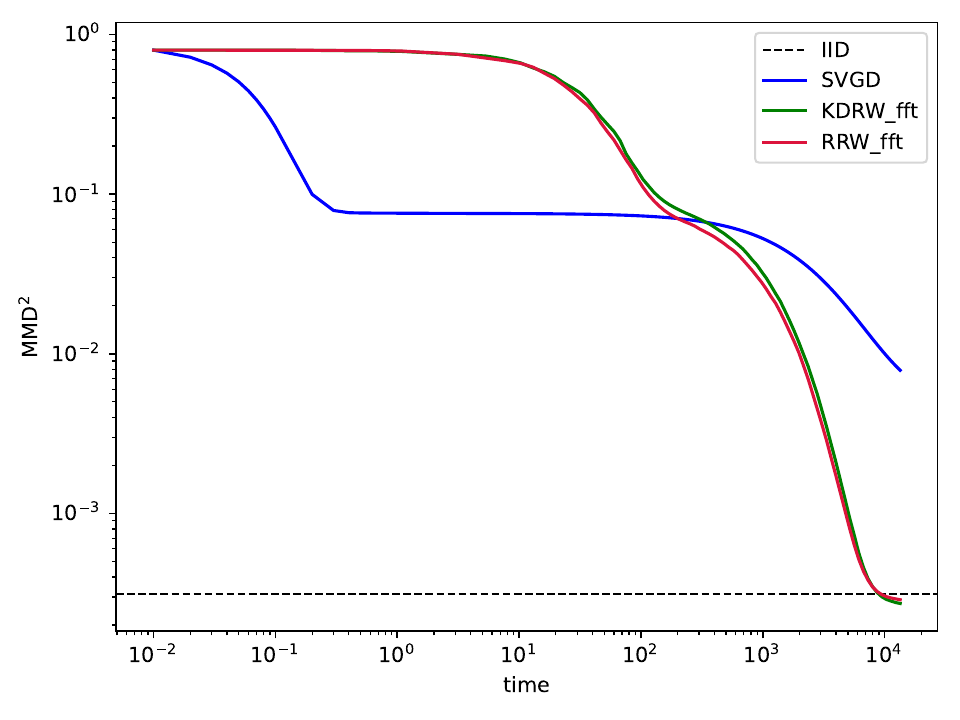}
    \caption{$d=2048$, $n=2048$}
    \label{fig:convergence20482048}
  \end{subfigure}

  \caption{$\MMD^2$ error versus algorithmic time for a Rosenbrock ``banana'' target distribution with potential~\eqref{eq:banana_potential} for different particle numbers $n$ and dimensions $d$.}
  \label{fig:convergencebanana}
\end{figure}

We see that KDRW\!\textunderscore fft and RRW\!\textunderscore fft perform well in all settings: they reach configurations that approximate the target measure as well as or better than the i.i.d.\ samples in a time that grows very mildly with the dimension. In lower dimensions they produce samples that are much better approximations than i.i.d.\ This is consistent with the experiments discussed in the next subsection. SVGD performs well in low and moderate dimensions when $n$ is substantially larger than $d$, while in Figure~\ref{fig:convergence2048256}, the final state of SVGD is far from the target distribution and the points have much smaller variance. We observed this to be typical of SVGD when $n<d$ and $d$ is large, with this being a manifestation of the variance collapse phenomenon discussed in~\cite{ba2021understanding}. In contrast, the RW algorithms perform well even when $n$ is substantially smaller than $d$.

Additionally, with the exception of $d=2$, where KDRW\!\textunderscore fft and RRW\!\textunderscore fft require a much smaller step size than SVGD, the RRW, KDRW and SVGD algorithms were stable for very similar step sizes. We note that in all dimensions (including $d=2$), the RW algorithms produced samples that are as good as, and often much better than, those of SVGD after performing approximately the same number of steps. When $d=2$, KDRW and RRW already produce excellent samples by $t=100$. Since each RW update costs $O(nd)$ whereas an SVGD update costs $O(dn^2)$, this implies a substantial computational advantage for the RW methods.

We also note that the step sizes under which the RW schemes remain stable do not appear to decrease with the dimension. This is consistent with our theoretical results in Subsection~\ref{subsec:spec-stoch_conv}, where we prove a dimension-free stochastic-approximation error bound between the exact ODE solution and the trajectories produced by our randomized schemes.

\subsection{Approximation error of final states}\label{sec:quantization}

Finally, we study how the long-time error scales with particle number $n$, i.e.\ the quantization error. As observed in Figure~\ref{fig:convergencebanana}, the outputs of KDRW\!\textunderscore fft and RRW\!\textunderscore fft often approximate the target measure better than the empirical measure of $n$ i.i.d.\ samples from the target distribution. A similar phenomenon was reported for SVGD in~\cite{xu2022accurate}. This improvement is specific to the interacting discrete dynamics and reflects the use of global information through particle interactions when steering the empirical measure toward the target.

First, in Figure~\ref{fig:quantization_MMD} we compare the $\MMD^2$ error of the long-time outputs of KDRW\!\textunderscore fft, RRW\!\textunderscore fft, SVGD, and i.i.d.\ sampling for a Gaussian target distribution across different particle numbers $n$ and dimensions $d$. Each run was initialized with i.i.d.\ samples from a standard Gaussian and was run until algorithmic time $t=10{,}000$, with respective step sizes for the RW algorithms and SVGD given by $0.1$ and $0.2$ when $d=2$, $0.2$ and $0.4$ when $d=32$, $0.5$ and $0.5$ when $d=256$, and $1.0$ and $1.0$ when $d=2048$. Results are averaged over $5$ trials when $d=2048$ and $10$ trials for all other dimensions.

The bandwidths~\eqref{bw:scaled_with_n} were used for KDRW\!\textunderscore fft and RRW\!\textunderscore fft. Setting the bandwidth for SVGD was a challenge. Namely, the adaptive bandwidth set according to the median trick,
\[b^2 = \frac{1}{2 \ln (n+1)} \text{median}_{i<j}|x^i - x^j|^2,\]
as is typical in the literature~\cite{liu2016stein,liu2018stein}, performs very poorly: the outputs remain far from the target distribution, especially in large dimensions. We found that taking a larger bandwidth, namely $b=\sqrt{2d}$, yielded substantially better results.

\begin{figure}[t]
\centering
\captionsetup{font=small}

\begin{subfigure}{0.45\textwidth}
  \centering
  \includegraphics[width=\linewidth]{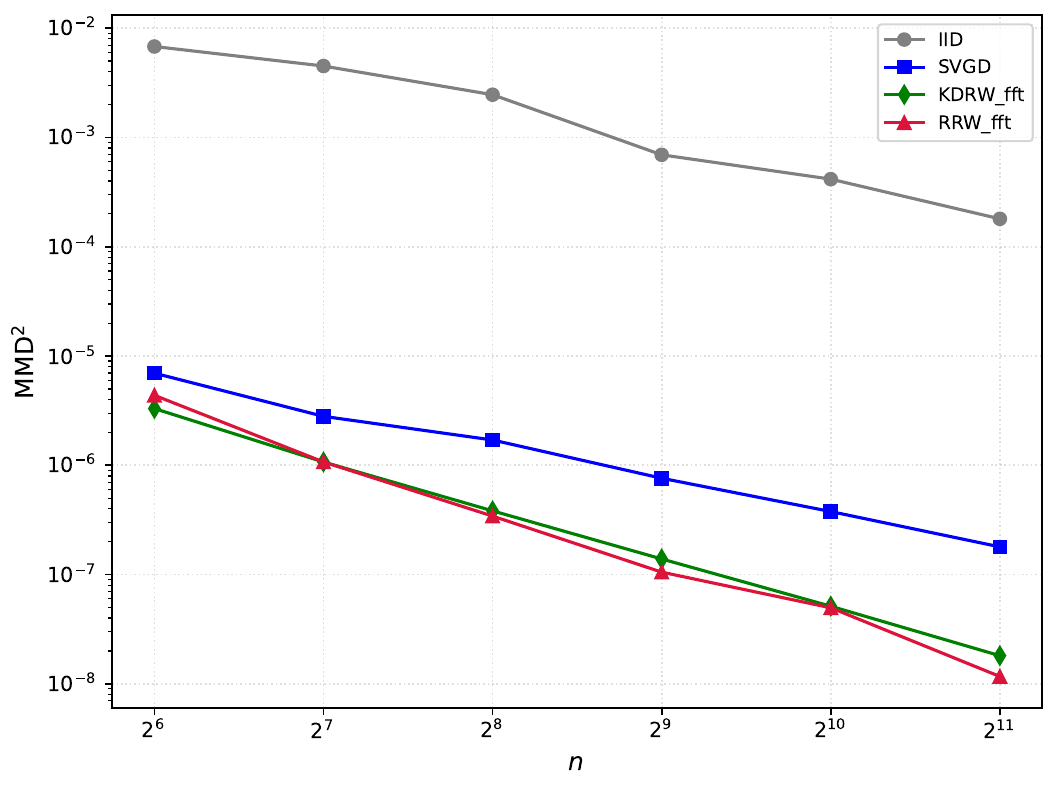}
  \caption{$d=2$}
  \label{fig:MMD_dim_2}
\end{subfigure}\hspace*{10pt}
\begin{subfigure}{0.45\textwidth}
  \centering
  \includegraphics[width=\linewidth]{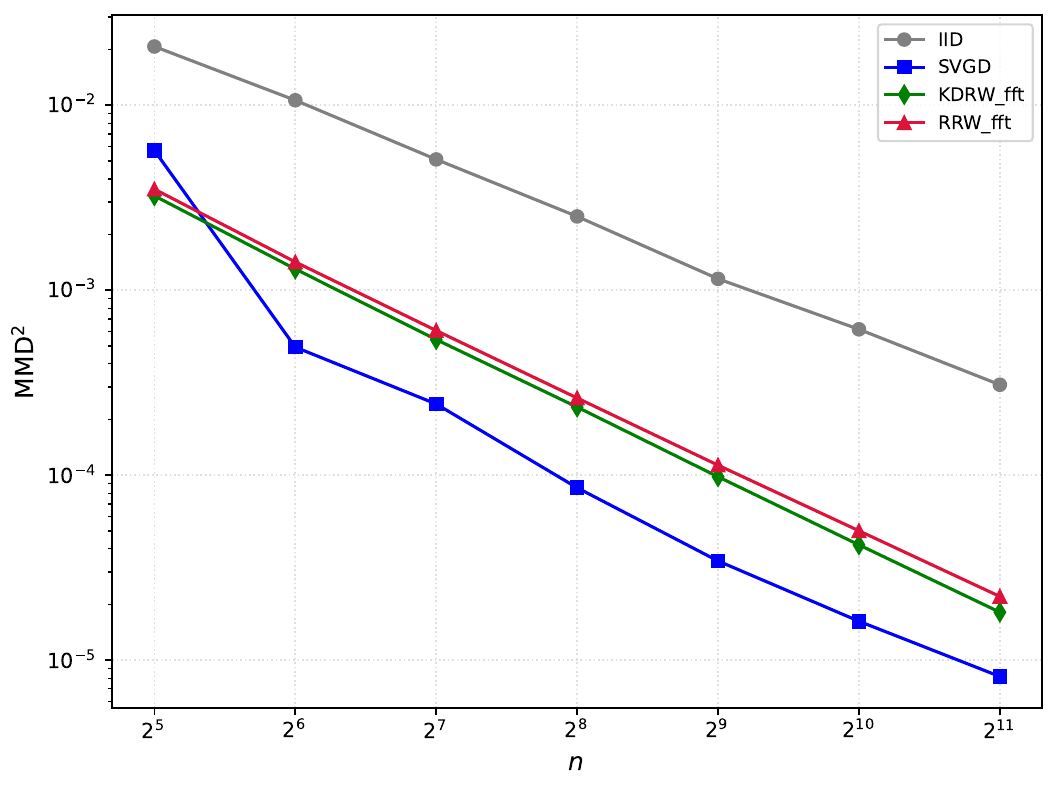}
  \caption{$d=32$}
  \label{fig:MMD_dim_32}
\end{subfigure}

\begin{subfigure}{0.45\textwidth}
  \centering
  \includegraphics[width=\linewidth]{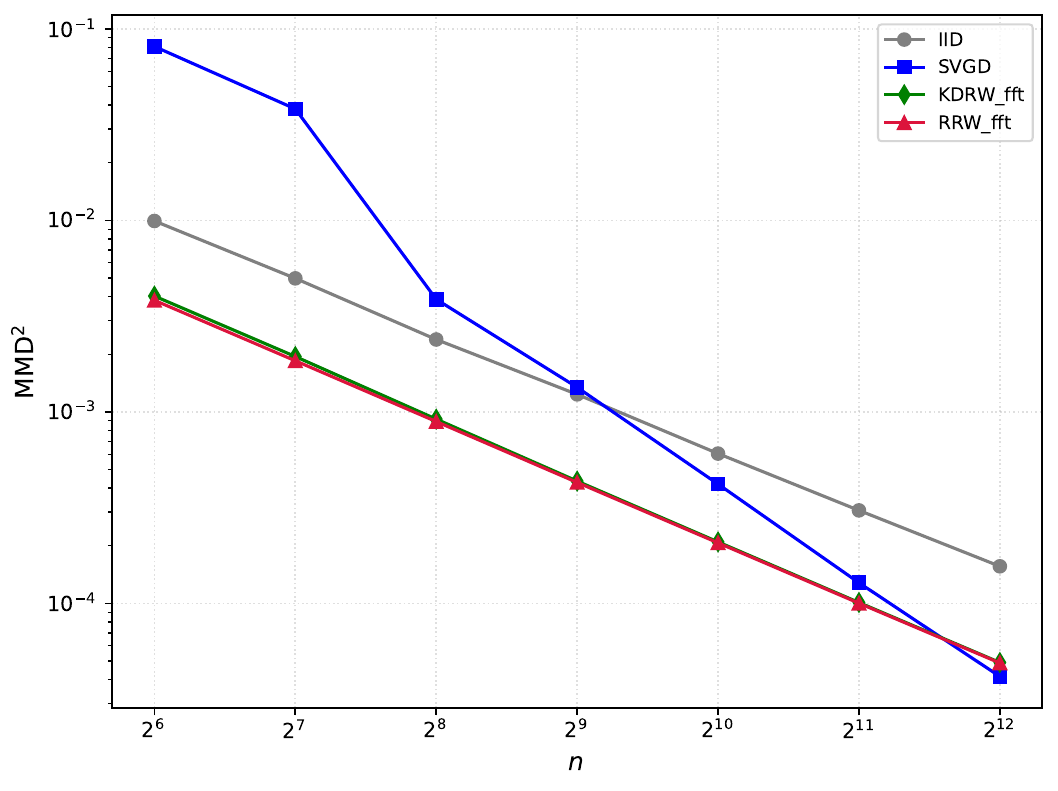}
  \caption{$d=256$}
  \label{fig:MMD_dim_256}
\end{subfigure}\hspace*{10pt}
\begin{subfigure}{0.45\textwidth}
  \centering
  \includegraphics[width=\linewidth]{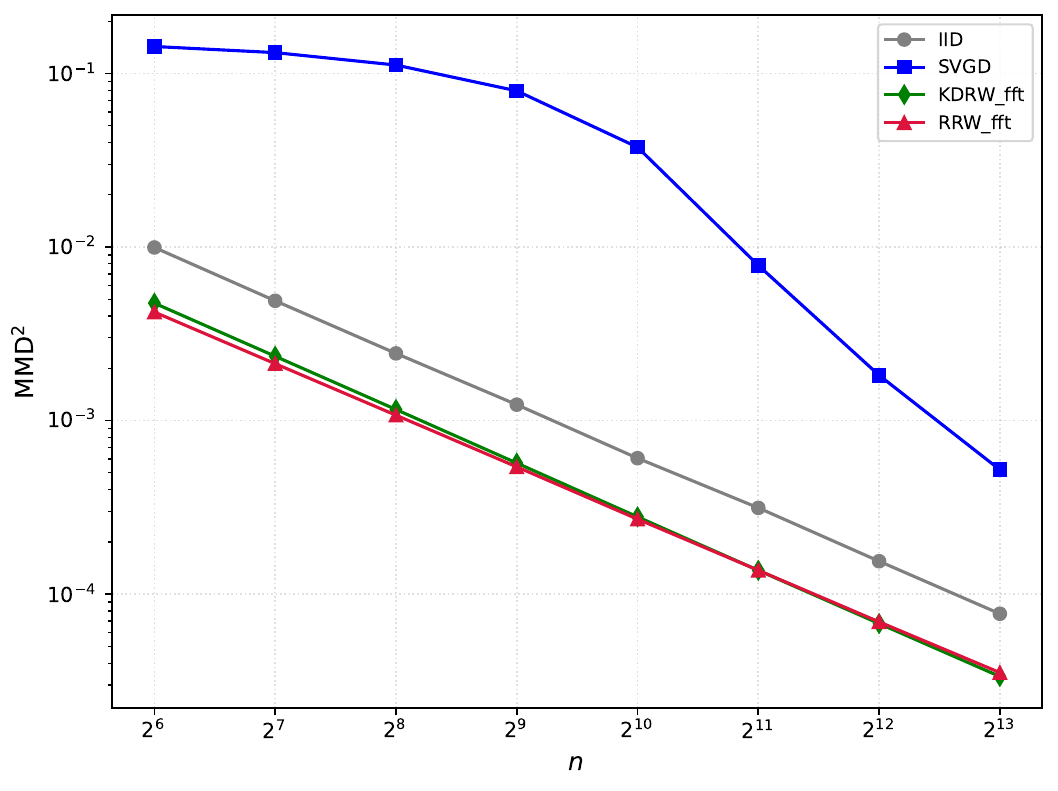}
  \caption{$d=2048$}
  \label{fig:MMD_dim_2048}
\end{subfigure}
     \caption{MMD$^2$ error versus particle number $n$ for a standard normal target distribution over different dimensions $d$.}\label{fig:quantization_MMD}
\end{figure}

We again observe that KDRW\!\textunderscore fft and RRW\!\textunderscore fft produce samples that are significantly closer to the target distribution in MMD distance than i.i.d.\ samples. SVGD also produces close samples when $n\geq d$. We also note that the error of the RW algorithms generally decreases with particle number faster than for i.i.d.\ samples. In particular, the slopes of the lines in the log--log plots of Figure~\ref{fig:quantization_MMD} (estimated by linear fits) are summarized in Table~\ref{tab:slopes}.

\begin{table}[b]
\centering
\begin{tabular}{lccc}
\hline
Dimension $d$ & i.i.d. & KDRW\!\textunderscore fft & RRW\!\textunderscore fft \\
\hline
$2$    & $-1.10$ & $-1.49$ & $-1.65$ \\
$32$   & $-1.02$ & $-1.23$ & $-1.20$ \\
$256$  & $-0.99$ & $-1.06$ & $-1.05$ \\
$2048$ & $-1.00$ & $-1.02$ & $-0.99$ \\
\hline
\end{tabular}
\caption{Slopes of the log--log error curves of the Radon--Wasserstein based algorithms in Figure~\ref{fig:quantization_MMD}.}
\label{tab:slopes}
\end{table}

Second, under the same experimental setup as in Figure~\ref{fig:quantization_MMD}, in Figure~\ref{fig:quantization_Mean} we plot the error of the sample mean
\begin{equation} \label{eq:mean_err}
    \| x \|_1 = \frac{1}{d} \sum_{i=1}^d  |x^i|,
\end{equation}
for each algorithm.

\begin{figure}[t]
\centering
\captionsetup{font=small}

\begin{subfigure}{0.45\textwidth}
  \centering
  \includegraphics[width=\linewidth]{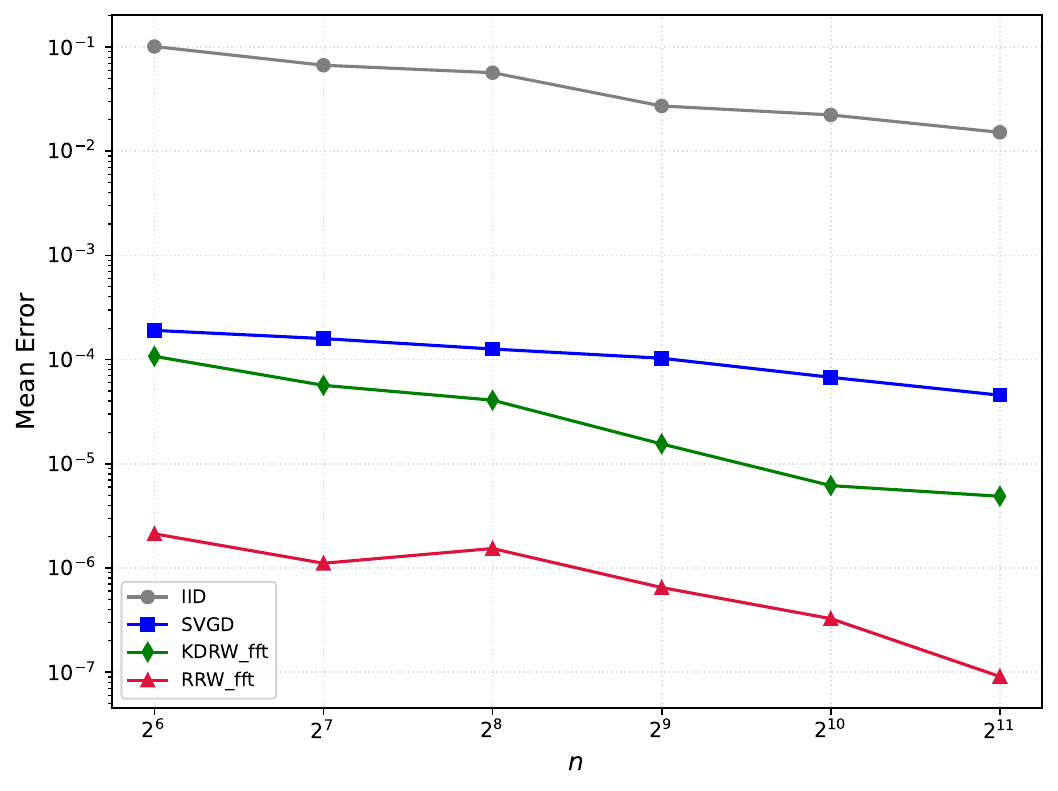}
  \caption{$d=2$}
  \label{fig:mean_dim_2}
\end{subfigure}\hspace*{10pt}
\begin{subfigure}{0.45\textwidth}
  \centering
  \includegraphics[width=\linewidth]{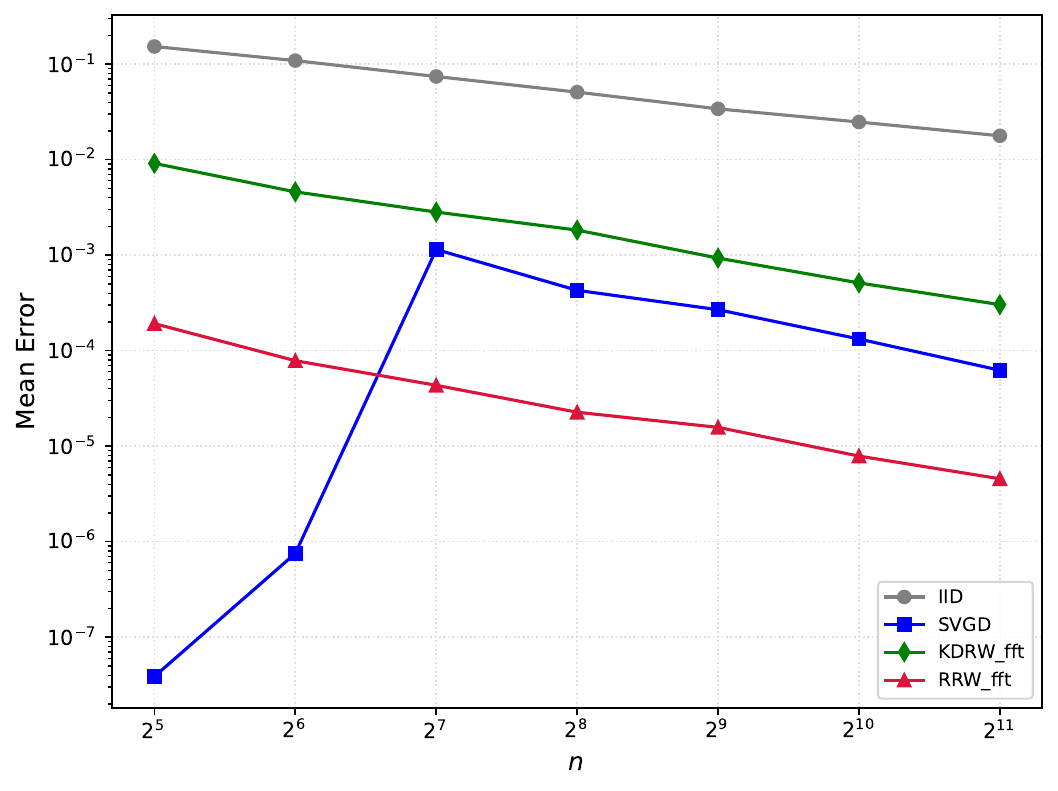}
  \caption{$d=32$}
  \label{fig:mean_dim_32}
\end{subfigure}

\begin{subfigure}{0.45\textwidth}
  \centering
  \includegraphics[width=\linewidth]{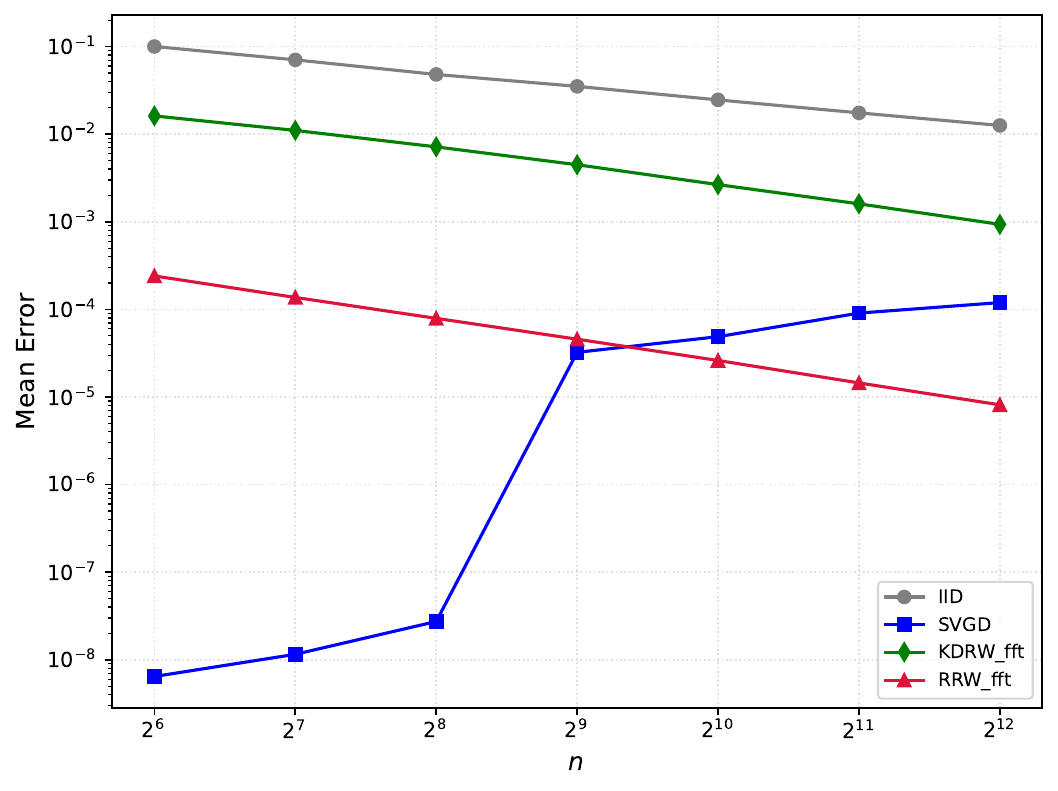}
  \caption{$d=256$}
  \label{fig:mean_dim_256}
\end{subfigure}\hspace*{10pt}
\begin{subfigure}{0.45\textwidth}
  \centering
  \includegraphics[width=\linewidth]{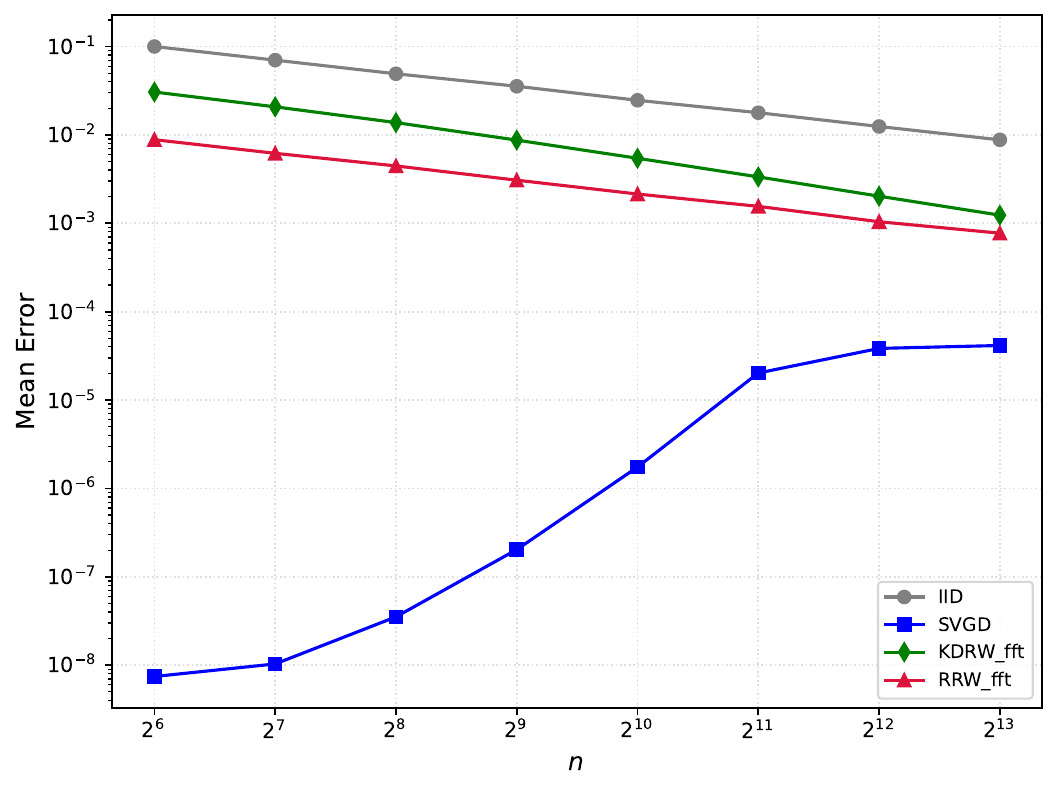}
  \caption{$d=2048$}
  \label{fig:mean_dim_2048}
\end{subfigure}
     \caption{Empirical mean~\eqref{eq:mean_err} error versus particle number $n$ for a standard normal target distribution over different dimensions $d$.}
\label{fig:quantization_Mean}
\end{figure}

We note that KDRW\textunderscore fft, RRW\textunderscore fft, and SVGD also have substantially smaller sample mean error than i.i.d.\ samples. Moreover, these results are largely consistent with the $\MMD^2$ error in Figure~\ref{fig:quantization_MMD}. The apparent oddity that SVGD produces the best approximation of the mean when $n\le d$, despite having large $\MMD^2$ error, is not due to a close approximation of the target measure, but rather to the variance collapse described in~\cite{ba2021understanding}: the particles are overall much closer to the origin than in i.i.d.\ sampling.

These experiments show that KDRW\!\textunderscore fft and RRW\!\textunderscore fft decrease the MMD distance in all dimensions and for all particle numbers, with the amount of improvement decreasing with dimension for all models. On the other hand, SVGD can increase the distance to the target (even though we start from i.i.d.\ samples of the target) when $n<d$. This degradation can be catastrophic: the final state of SVGD may fail to resemble the target distribution. This problem appears to worsen as the dimension increases. While we report SVGD experiments with bandwidth $\sqrt{2d}$, the other SVGD bandwidth choices we tested exhibited similar or worse behavior. Finally, we note that for all of these systems, convergence toward the final state is not exponential in the corresponding continuum model. Thus, while the evolution slows down with time, we cannot reliably predict what would happen for times well beyond the $t=10{,}000$ considered here.

\section{Theoretical Results} \label{sec:theory}

In this section, we study the theoretical properties of the flows proposed above, including the well-posedness of the continuous flow/particle schemes, the convergence of the stochastic descent scheme to the particle schemes, and the long-time behavior of the solutions. As many of the proofs of these facts are straightforward adaptions of well-known results, we defer most of the proofs to Appendix~\ref{app:theory}. The only exception is the proof of the qualitative long-time convergence of the RRW flow to the target measure, which we include in Subsection~\ref{subsec:long_time_convergence}.

Before continuing, we define the class of target measures for which these results apply. Given a target measure $\pi$ on $\R^d$ that is proportional to $e^{-U}$, we will always assume that  $U:\R^d\rightarrow \R$ satisfies the following conditions:

\begin{assumption}\label{ass:potential} There exists $L,a,c>0$ such that
\begin{enumerate}
    \item $\|\nabla U\|_{\text{Lip}}\leq L,$
    \item $U(x)\geq a|x|^2-c$.
\end{enumerate}   
\end{assumption}

That is, $\nabla U$ is globally Lipschitz and $U$ is quadratically confining. This assumption thus includes all Gaussian distributions and Gaussian mixtures.

Additionally, we will often require that the regularizing kernel $k$ used to define~\eqref{eq:KDRWgf} and~\eqref{eq:RRWgf} is sufficiently regular.

\begin{condition}\label{cond:k_conditions}
$k\in W^{1,1}(\R)\cap C_b^2(\R)$ and $k$ is positive and even.
\end{condition}

\subsection{Well-posedness}\label{subsec:spec-wp}

We start by stating the well-posedness of the KDRW and the  RRW flows. We first clarify our definition of a solution to a continuity equation.

\begin{definition}
If $\rho\in C([0,\infty),\Pc(\R^d))$ and $v(t,x)\in L^1_{loc}([0,\infty),L^1_{loc}(\rho_t))$, then we say $\rho$ is a solution to
\begin{equation}\label{eq:continuity}
\begin{cases}
    \partial_t\rho+\nabla\cdot(\rho v)=0\\
    \rho|_{t=0}=\nu
\end{cases}
\end{equation}
if
\[\int_0^\infty \int_{\R^d}(\partial_t \phi(t,x)+\nabla\phi(t,x)\cdot v(t,x)\,d\rho_t(x)\,dt+\int_{\R^d}\phi(0,x)\,d\nu(x)=0\]
for all $\phi\in C_c^\infty([0,\infty)\times \R^d)$.
\end{definition}

Our first result is then, when $\eps>0$ and $k$ is sufficiently regular, that there exists a unique solution to the KDRW and RRW flows in the class of $C([0,\infty),\Pc_2(\R^d))$ solutions for all initial conditions.

\begin{theorem}\label{thm:specific_existence}
If $\eps>0$ and $k$ satisfies Condition~\ref{cond:k_conditions}, then~\eqref{eq:KDRWgf} and~\eqref{eq:RRWgf} admit unique solutions in $C([0,\infty),\Pc_2(\R^d))$ for any initial condition $\nu$ in $\Pc_2(\R^d)$. More so, letting $\rho$ be the solution in either case, there exists $C(L,|\nabla U(0)|,\|k\|_{C^2_b},\eps^{-1})>0$ such that $\|\rho_t\|_{\Pc_2}\leq e^{Ct}(\|\nu\|_{\Pc_2}+1)$ for all $t\geq 0$.
\end{theorem}

The method of proof is classical: the velocity field, viewed as a non-local function, enjoys sufficient Lipschitz regularity under these conditions so as to apply a fixed point argument. That is, Theorem~\ref{thm:specific_existence} holds as a consequence of a general fixed point theorem and the following lemma.

\begin{proposition}\label{prop:specific_velocity}
Let $V:\R^d\times\mathcal{P}_2(\R^d)\rightarrow \R^d$ be defined by
\begin{equation}\label{eq:KDRW_vec}
V(x,\mu):=-\fint_{\S^{d-1}}\theta \left(\frac{k'*R^\theta\mu+k*R^\theta(\mu\nabla_\theta U)}{k*R^\theta\mu+\eps}\right)(x\cdot \theta)\,d\theta
\end{equation}
or
\begin{equation}\label{eq:RRW_vec}
V(x,\mu):=-\fint_{\S^{d-1}}\theta k*\left(\frac{k'*R^\theta\mu+k*R^\theta(\mu\nabla_\theta U)}{k*R^\theta\mu+\eps}\right)(x\cdot \theta)\,d\theta.
\end{equation}
Then, given the conditions of Theorem~\ref{thm:specific_existence}, there exists $M(L,|\nabla U(0)|,\|k\|_{C^2_b},\eps^{-1})>0$ so that for all $x,y\in\R^d$ and $\nu,\mu\in \Pc_2(\R^d)$
\begin{enumerate}
    \item $|V(x,\mu)|\leq M(1+\|\mu\|_{\Pc_2})$,
    \item $|V(x,\mu)-V(y,\nu)|\leq M(1+\|\mu\|_{\Pc_2})(|x-y|+\mathcal{W}(\mu,\nu))$.
\end{enumerate}
\end{proposition}

We note that~\eqref{eq:KDRWgf} and~\eqref{eq:RRWgf} correspond to the continuity equation~\eqref{eq:continuity} with $v(t,x)=V(x,\rho_t)$ where $V$ is respectively defined by~\eqref{eq:KDRW_vec} and~\eqref{eq:RRWgf}. Proposition~\ref{prop:specific_velocity} then states that these velocity fields grow linearly with respect to the second moment of the measure and are locally Lipschitz in both space and measure coordinates.

Although the positivity of $\eps$ is in general essential to these bounds, for special choices of $k$ (such as Gaussian) and under stronger hypothesis on the initial conditions, this assumption can be removed via a more complicated fixed point argument that uses the exact decay of the kernel and its derivatives. Since taking $\eps>0$ is always natural for applications, we do not include this proof.

Proposition~\ref{prop:specific_velocity} follows immediately from Lemma~\ref{lem:theta_vector_bounds} in Subsection~\ref{subsec:abstract} of the Appendix. Theorem~\ref{thm:specific_existence} is then a direct application of an abstract well-posedness theorem for non-local velocity fields satisfying conditions guaranteed by Proposition~\ref{prop:specific_velocity}---Theorem~\ref{thm:transport_wp} in Subsection~\ref{subsec:wp}. The proof of Theorem~\ref{thm:transport_wp} is a straightforward modification of a classical fixed point argument (for example, see~\cite[Theorem~4.21]{CarmonaDelarue2018I}). As we could not find an exact result that applies to the velocity fields considered here in the literature, we include the proof for completeness.

Continuing, it is a straightforward corollary of the abstract well-posedness theorem that the corresponding finite particle systems are well-posed.

\begin{corollary}\label{cor:specific_ode_wp}
Under the conditions of Theorem~\ref{thm:specific_existence}, for all initial conditions $\ux_0^n\in(\R^d)^n$, there exists unique solutions to the system of ODEs defined by
\begin{equation}\label{eq:discrete_KDRWgf}
\frac{d}{dt}x_t^i=-\fint_{\S^{d-1}}\theta\left(\frac{\sum_{j=1}^n k'((x_t^i-x_t^j)\cdot\theta)+\sum_{j=1}^n k((x_t^i-x_t^j)\cdot\theta)\nabla_\theta U(x_t^j)}{\sum_{j=1}^n k((x_t^i-x_t^j)\cdot\theta)+n\eps}\right)\,d\theta,
\end{equation}
or
\begin{equation}\label{eq:discrete_RRWgf}
\frac{d}{dt}x_t^i=-\fint_{\S^{d-1}}\theta\int_{\R}k(x_t^i\cdot\theta-p)\left(\frac{\sum_{j=1}^n k'(p-x_t^j\cdot\theta)+\sum_{j=1}^n k(p-x_t^j\cdot\theta)\nabla_\theta U(x_t^j)}{\sum_{j=1}^n k(p-x_t^j\cdot\theta)+n\eps}\right)\,dp\,d\theta.
\end{equation}
More so, letting  $\mu_t^n:=\frac{1}{n}\sum_{i=1}^n\delta_{x_t^i}$ in either case, there exists $C(L,|\nabla U(0)|,\|k\|_{C^2_b},\eps^{-1})>0$ such that $\|\mu_t^n\|_{\Pc_2}\leq e^{Ct}(\|\mu_0^n\|_{\Pc_2}+1)$.
\end{corollary}

We note that~\eqref{eq:KDRW_sgd} and~\eqref{eq:RRW_sgd} are respectively the discrete in time, stochastic descent versions of~\eqref{eq:discrete_KDRWgf} and~\eqref{eq:discrete_RRWgf}. Corollary~\ref{cor:specific_ode_wp} follows immediately from Proposition~\ref{prop:specific_velocity} and Corollary~\ref{cor:ode_wp} in Subsection~\ref{subsec:wp} of the Appendix.

\subsection{Stability and mean-field convergence} \label{subsec:MFconverg}

Next, we state a stability estimate for solutions to~\eqref{eq:KDRWgf} and~\eqref{eq:RRWgf}.

\begin{theorem}\label{thm:specific_stability}
Suppose $\eps>0$, $k$ satisfies Condition~\ref{cond:k_conditions}, and $\rho,\rho'$ are both solutions to~\eqref{eq:KDRWgf} or~\eqref{eq:RRWgf} with respective initial conditions given by $\nu,\mu\in\Pc_2(\R^d)$. Then, for all $T>0$ there exists $C(L,|\nabla U(0)|,\|k\|_{C^2_b},\eps^{-1},T,\|\nu\|_{\Pc_2})>0$ such that
\[\sup_{t\in[0,T]}\mathcal{W}(\rho_t,\rho'_t)\leq C\mathcal{W}(\nu,\mu).\]
\end{theorem}

Theorem~\ref{thm:specific_stability} implies that the systems~\eqref{eq:discrete_KDRWgf} and~\eqref{eq:discrete_RRWgf} mean-field converge to~\eqref{eq:KDRWgf} and~\eqref{eq:RRWgf} respectively. That is, if $\ux_t^n$ and $\rho$ are the unique solutions to~\eqref{eq:discrete_KDRWgf} and~\eqref{eq:KDRWgf}, with initial conditions $\ux^n_0\in(\R^d)^n$ and $\nu\in\Pc_2(\R^d)$, then since $\mu_t^n:=\frac{1}{n}\sum_{i=1}^n\delta_{x_t^i}$ is also a solution to~\eqref{eq:KDRWgf} with initial condition $\mu_0^n$, there exists $C>0$ such that
\[\sup_{t\in[0,T]}\mathcal{W}(\rho_t,\mu^n_t)\leq C\mathcal{W}(\nu,\mu^n_0).\]
This implies that if $\mu_0^n\rightarrow \nu$ in $\Pc_2(\R^d)$, then $\mu_t^n\rightarrow \rho_t$ for all $t>0$. An analogous statement holds for~\eqref{eq:RRWgf} and~\eqref{eq:discrete_RRWgf}. As the empirical measure of an i.i.d.\ sample from $\nu$ converges to $\nu$ in $\mathcal{P}_2(\R^d)$, the theorem thus verifies that discrete approximations of the continuum PDE are accurate for large particle numbers.

Theorem~\ref{thm:specific_stability} follows from Proposition~\ref{prop:specific_velocity} and Theorem~\ref{thm:stability} in Subsection~\ref{subsec:stability} of the Appendix. The latter theorem is essentially Dobrusin's coupling argument~\cite{MR541637}.

\subsection{Stochastic descent scheme convergence}\label{subsec:spec-stoch_conv}

In this subsection we establish that, as the time step $\tau$ goes to zero, the stochastic descent schemes~\eqref{eq:KDRW_sgd} and~\eqref{eq:RRW_sgd} respectively converge to the continuum ODE systems~\eqref{eq:discrete_KDRWgf} and~\eqref{eq:discrete_RRWgf}. Given $\tau>0$ and $x^i_m$ solving~\eqref{eq:KDRW_sgd} or~\eqref{eq:RRW_sgd}, we let $x_t^{i,\tau}$ denote the linear interpolation
\[x_t^{i,\tau}:=x_{m}^i+(t-\tau m)(x^i_{m+1}-x^i_m),\qquad t\in (\tau m,\tau(m+1)].\]
Our main result is then the following estimate.

\begin{theorem}\label{thm:specific_sgd_conv}
Suppose that $\eps>0$, $k$ satisfies Condition~\ref{cond:k_conditions}, and $\ux_t^n$ and $\ux_t^{n,\tau}$ are respectively defined by~\eqref{eq:discrete_KDRWgf} and~\eqref{eq:KDRW_sgd} or~\eqref{eq:discrete_RRWgf} and~\eqref{eq:RRW_sgd} with the same initial condition $\ux_0^n\in(\R^d)^n$. Then, for all $T\geq 0$ there exists $C(L,|\nabla U(0)|,\|k\|_{C^2_b},\eps^{-1},T,\|\mu_0^n\|_{\Pc_2})>0$ such that
\begin{equation}
\sup_{t\in[0,T]}\mathbb{E}\left[\frac{1}{n}\sum_{i=1}^n|x_t^i-x_t^{i,\tau}|^2 \right]\leq C\tau .
\end{equation}
\end{theorem}

Theorem~\ref{thm:specific_sgd_conv} implies that if $x_0^i$ are i.i.d.\ sampled from some distribution $\nu\in\Pc_2(\R^d)$, and $\ux_t^n$ and $\ux_t^{n,\tau}$ are initialized with $\ux_0^n$, then
\[\sup_{t\in[0,T]}\mathbb{E}\Big[\mathcal{W}(\mu_t^n,\mu_t^{n,\tau})^2\Big]\leq C\tau\]
where $\mu_t^n$ and $\mu_t^{n,\tau}$ are the empirical measures of the systems.

Notably, the constant $C$ in Theorem~\ref{thm:specific_sgd_conv} does not directly depend on the dimension $d$, which helps explain the dimension-independent stability observed in Subsection~\ref{sec:convergence}. This behavior is inherited from the single-direction estimators: $u(\theta,x\cdot\theta,\rho)$ and $k*u(\theta,x\cdot\theta,\rho)$ are unbiased estimators of the full velocity field $v(t,x)$, and Lemma~\ref{lem:theta_vector_bounds} implies variance bounds in $\theta$ that are uniform in $d$. These two inputs are precisely what enter the stochastic-approximation argument.

Theorem~\ref{thm:specific_sgd_conv} is a consequence of Lemma~\ref{lem:theta_vector_bounds} and Theorem~\ref{thm:sgd_conv} in Subsection~\ref{subsec:stoc_conv} of the Appendix. The proof follows the standard stochastic approximation convergence argument for stochastic gradient descent~\cite{RobMon51}.

\subsection{Long-time convergence}\label{subsec:long_time_convergence}

We conclude this section with a statement of the long-time convergence properties of the RRW gradient flow. In particular, that under mild conditions on $k$, if $\mathcal{F}(\nu)<\infty$ the flow converges weakly to the target measure as $t\rightarrow \infty$.

\begin{theorem}\label{thm:long_time_convergence}
   Suppose that $\eps>0$, $k$ satisfies Condition~\ref{cond:k_conditions}, $k$ has a nowhere vanishing Fourier transform, and $\mathcal{F}(\nu)<\infty$. If $\rho$ is the unique solution to~\eqref{eq:RRWgf} with initial condition $\nu$, then $\rho_t\rightarrow \pi$ weakly as $t\rightarrow\infty$.
\end{theorem}

We emphasize that the above theorem is non-quantitative. The proof is similar to that of qualitative long-time convergence of SVGD in~\cite{Korba2020NonAsymptoticSVGD}, and uses compactness plus an entropy dissipation identity.

Before proving Theorem~\ref{thm:long_time_convergence}, we establish some intermediate lemmas. First, we give a convenient condition on a measure $\rho$ being equal to $\pi$. This uses similar calculations as in~\cite[Theorem 2.1]{sharafutdinov2021radon}.

\begin{lemma}\label{lem:equil_char}
Suppose that $k$ satisfies Assumption~\ref{cond:k_conditions}, $k$ has a nowhere vanishing Fourier transform, and $\rho\in L^1(\R^d)\cap \Pc_2(\R^d)$. If
\[k'*R\rho+k*R(\rho\nabla U)\cdot\theta=0,\]
almost everywhere on $\RD^d$, then $\rho=\pi.$
\end{lemma}

\begin{proof}
First, we note that $R(\rho)$ and $R(\rho\nabla U)$ are in $L^1(\RD^d)$, thus
\[g^\theta(p):=k'*R^\theta(\rho)(p)+k*(R^\theta(\rho\nabla U))(p)\cdot\theta\]
is an element of $L^2(\R)$ for all $\theta\in\S^{d-1}$. Additionally, since $\widehat{R^\theta f}(s)=\hat{f}(\theta s)$ for all $f\in L^1(\R^d)$, $g^\theta$  has Fourier transform (in $p$) given by
\begin{align*}
\hat{g^\theta}(s)&=2\pi is\hat{k}(s)\widehat{R^\theta(\rho)}(s)+\hat{k}(s)\widehat{R^\theta(\rho\nabla U)}(s)\cdot \theta
\\&=2\pi is\hat{k}(s)\widehat{\rho}(s\theta)+\hat{k}(s)\widehat{\rho\nabla U}(\theta s)\cdot \theta.
\end{align*}
The Plancherel theorem thus implies that
\begin{align*}
 0&=\int_{\RD^d}|k'*R^\theta(\rho)+k*(R^\theta(\rho\nabla U))\cdot\theta|^2\,dp\,d\theta
 \\&=\fint_{\S^{d-1}}\int_{\R}|\hat{k}( s)|^2|2\pi i s\hat{\rho}(\theta s)+\widehat{(\rho\nabla U)}(\theta s)\cdot\theta|^2\,d s\,d\theta.
\end{align*}
Changing coordinates, in total we have that
\[0=\int_{\R^d}|\zeta |^{-(d+1)}|\hat{k}(|\zeta|)|^2|2\pi i \zeta|^2\hat{\rho}(\zeta)+(2\pi i \zeta)\cdot \widehat{(\rho\nabla U)}(\zeta )|^2\,d\zeta.\]
Since $\hat{k}(s)\neq 0$ by assumption, this implies that
\[|2\pi i \zeta|^2\hat{\rho}(\zeta)+(2\pi i \zeta)\cdot \widehat{(\rho\nabla U)}(\zeta)=0\]
almost surely, and thus
\[\nabla\cdot (\nabla\rho+\rho\nabla U))=0\]
in distribution.

We thus find that $\rho_t=\rho$ is a stationary solution to the Fokker--Planck equation
\[\partial_t\rho_t=\Delta\rho_t+\nabla(\rho_t\nabla U),\]
and thus it must be the case that $\rho=\pi$~\cite[Proposition 4.6]{Pavliotis2014StochasticProcessesApplications}.
\end{proof}

Next, we recall the following entropy balance identity for continuity equations. For completeness, we include a self contained proof in Appendix~\ref{ap:continuity}.

\begin{lemma}\label{lem:entropy_dynamics}
Suppose that $\nu\in\mathcal{P}(\R^d)$ is such that $\mathcal{F}(\nu)<\infty$ and $v\in L^1_{loc}([0,\infty),W^{1,\infty}(\R^d))$. If $\rho$ is the unique solution to~\eqref{eq:continuity}, then
\[\mathcal{F}(\rho_t)+\int_0^t \int_{\R^d} \big(\nabla\cdot v(s,x)-\nabla U(x)\cdot v(s,x)\big)\rho_s(x)\,dx\,ds=\mathcal{F}(\nu).\]
or all $t\geq0$.
\end{lemma}

We use the above lemma to prove the following entropy dissipation equation for solutions to ~\eqref{eq:RRWgf}.

\begin{proposition}\label{prop:balance}
Suppose that $\eps>0$, $k$ satisfies Condition~\ref{cond:k_conditions}, and $\mathcal{F}(\nu)<\infty$. If $\rho$ is the unique solution to~\eqref{eq:RRWgf} with initial condition $\nu$, then it holds that
\[ \mathcal{F}(\rho_t)+\int_0^t \int_{\mathbb{D}^d} \frac{|k'*R\rho+k*R(\rho\nabla_\theta  U)|^2}{k*R\rho+\eps}\,d\theta dp\,dt= \mathcal{F}(\nu).\]
for all $t\geq 0$.
\end{proposition}

\begin{proof}
Let
\[v(t,x):=-\fint_{\S^{d-1}}\theta  k*\left(\frac{k'*R^\theta\rho_t+k*R^\theta(\rho_t\nabla_\theta U)}{k*R^\theta\rho_t+\eps}\right)(x\cdot\theta)\,d\theta.\]
so that $\rho_t$ solves~\eqref{eq:continuity} with drift $v$. Then, since $\|\rho_t\|_{\Pc_2}\leq e^{Ct}(\|\nu\|_{\Pc_2}+1)$, Proposition~\ref{prop:specific_velocity} implies that $v\in C([0,\infty),W^{1,\infty}(\R^d))$, thus
\[\mathcal{F}(\rho_t)+\int_0^t \int_{\R^d} (\nabla\cdot v(s,x)-\nabla U(x)\cdot v(s,x))\rho_s(x)\,dx\,ds=\mathcal{F}(\nu)\]
by Lemma~\ref{lem:entropy_dynamics}. We then note that
\[\nabla\cdot v=-R^*\left( k'*\left(\frac{k'*R^\theta\rho_t+k*R^\theta(\rho_t\nabla_\theta U)}{k*R^\theta\rho_t+\eps}\right)\right).\]
With Proposition~\ref{prop:adjoint} we thus find that
\begin{align*}
&\int_{\R^d} (\nabla\cdot v-v\cdot\nabla U)\rho\,dx
\\&\qquad= \int_{\R^d} - R^*\left(k'*\frac{k'*R\rho+R(\rho\nabla_\theta  U)}{k*R\rho+\eps}\right)\rho\,dx+\int_{\R^d} R^*\left(\theta k*\frac{k'*R\rho+R(\rho\nabla_\theta  U)}{K*R\rho+\eps}\right) \cdot\nabla U\rho\,dx
\\&\qquad=\int_{\mathbb{D}^d} \frac{|k'*R\rho+R(\rho\nabla_\theta  U)|^2}{k*R\rho+\eps}\,d\theta dp,
\end{align*}
concluding the claim.
\end{proof}

We now have the requisite results to prove the theorem.

\medskip

\begin{proof}[Proof of Theorem~\ref{thm:long_time_convergence}]

First we note that Proposition~\ref{prop:balance} implies that $\mathcal{F}(\rho_t)\leq \mathcal{F}(\nu)<\infty$ for all $t\geq 0$. Since the KL divergence has weakly compact sub-level sets, this implies that $\{\rho_t\}_{t\geq 0}$ is precompact in the weak topology. Additionally, the Donsker--Varadhan variational formula implies that 
\[\int_{\R^d} \lambda |x|^2\,d\rho_t \leq \mathcal{F}(\rho_t)+\int_{\R^d} e^{\lambda |x|^2}\,d\pi\leq \mathcal{F}(\nu)+\int_{\R^d} e^{\lambda |x|^2}\,d\pi<\infty \]
given that $\lambda<a$. The family $\{\rho_t\}_{t\geq 0}$ thus have uniformly bounded second moments.

Proposition~\ref{prop:balance} also implies that
\begin{align*}
&\int_0^\infty\int_{\mathbb{D}^d} |k'*R\rho_{t_\ell}+k*R(\rho_{t_{k_\ell}}\nabla_\theta  U)|^2 \,d\theta dp\,dt
\\&\qquad\leq (\|k\|_{L^\infty}+\eps)\int_0^\infty\int_{\mathbb{D}^d} \frac{|k'*R\rho_{t_{k_\ell}}+k*R(\rho_{t_{k_\ell}}\nabla_\theta  U)|^2}{k*R\rho_{t_{k_\ell}}+\eps}\,d\theta dp\,dt
\\&\qquad\leq (\|k\|_{L^\infty}+\eps)\mathcal{F}(\nu)<\infty.  
\end{align*}
Thus, to show that
\[\lim_{t\rightarrow \infty}\int_{\mathbb{D}^d} |k'*R\rho_t+k*R(\rho_t\nabla_\theta  U)|^2 \,d\theta dp=0\]
it suffices to prove that the function in the limit above is Lipschitz continuous in $t$.

To this end, let
\begin{align*}
h_t(\theta,p):=k'*R^\theta\rho_t(p)+k*R^\theta(\rho_t\nabla_\theta  U)(p)=\int_{\R^d} \big(k'(p-x\cdot\theta)+k(p-x\cdot\theta)\nabla_\theta U(x)\big)\rho_t(x)\,dx,
\end{align*}
so that
\[\int_{\mathbb{D}^d} |k'*R\rho_t+k*R(\rho_t\nabla_\theta  U)|^2 \,d\theta dp=\int_{\RD^d} h_t(p,\theta)^2\,d\theta dp.\]
We then have that for every $t$ and  $\theta$
\begin{align*}
 \int_\R |h_t(\theta,p)|\,dp&\leq \int_{\R}|k'*R^\theta\rho_t|+|k*R(\rho_t\nabla_\theta U)|\,dp\leq \|k'\|_{L^1(\R)}+\|k\|_{L^1(\R)}\|R(\rho_t\nabla U)\|_{L^1(\R)}.
\end{align*}
Since
\begin{align*}
\|R(\rho_t\nabla_\theta U)\|_{L^1(\R)}\leq \int_{\R^d} |\nabla_\theta U(x)|\rho_t(x)\,dx&\leq |\nabla U(0)|+\|\nabla U\|_{\text{Lip}}\int_{\R^d}|x|\rho_t(x)\,dx
\\&\leq C(1+\|\rho_t\|_{\Pc_2})    
\end{align*}
for some $C>0$, we find that $\|h_t(\theta,\cdot)\|_{L^1(\R)}$ is uniformly bounded over $t$ and $\theta$. Since $\rho_t$ is a solution to a continuity equation, $h_t(p,\theta)$ is absolutely continuous in $t$ and for almost every $t$
\begin{align*}
\frac{d}{dt} h_t(p,\theta)&=\int_{\R^d} \big(-k''(p-x\cdot\theta)-k'(p-x\cdot\theta)\nabla_\theta U(x)\big)\theta\cdot v(t,x)\rho_t(x)\,dx
\\&\quad+\int_{\R^d}k(p-x\cdot\theta)\big(\text{Hess}\ U(x)\theta\big)\cdot v(t,x)\rho_t(x)\,dx.
\end{align*}
Bounding each term individually, in an analogous way as to above, we conclude that
\begin{align*}
\Big| \frac{d}{dt} h_t(p,\theta)\Big|&\leq \|k''\|_{L^\infty}\|v\|_{L^\infty}+\|k'\|_{L^\infty}\|v\|_{L^\infty}(|\nabla U(0)|+\|\nabla U\|_{\text{Lip}}\|\rho_t\|_{\mathcal{P}_2})+\|\text{Hess}\ U\|_{L^\infty}\|k\|_{L^\infty}\|v\|_{L^\infty}
\\&\leq M
\end{align*}
for some constant $M>0$ that is independent of $t$, $p$ and $\theta$ since $v(t,x)$ is uniformly bounded  in space and time by Proposition~\ref{prop:specific_velocity}. Since
\[\Big||h_t(p,\theta)|^2-|h_s(p,\theta)|^2\Big|=|h_t(p,\theta)-h_s(p,\theta)||h_t(p,\theta)+h_s(p,\theta)|\leq M|t-s|(|h_t(p,\theta)|+|h_s(p,\theta)|),\]
combining the above we indeed find that $t\mapsto \|h_t\|_{L^2(\RD^d)}^2$ is Lipschitz as desired.

We can now conclude the theorem. Let $t_k$ be any sequence such that $t_k\rightarrow \infty$. Then, by compactness, we can extract a subsequence $t_{k_\ell}$ such that $\rho_{t_{k_\ell}}\rightarrow \rho_\infty$ weakly where $\rho_\infty$ is some probability measure such that $\mathcal{F}(\rho_\infty)<\infty$. In particular, this implies that $\rho_\infty$ has a density. Since
\begin{align*}
k'*R^\theta\rho_{t_{k_\ell}}+k*R^\theta(\rho_{t_{k_\ell}}\nabla_\theta  U)&=\int_{\R^d} \big(k'(\cdot-x\cdot \theta)+k(\cdot-x\cdot\theta)\nabla_\theta U(x)\big)\rho_{t_{k_\ell}}(x)\,dx
\end{align*}
and the integrand (with respect to $\rho_{t_{k_\ell}}$) above is continuous and bounded by $C(1+|x|)$ for some $C>0$, the weak convergence of $\rho_{t_{k_\ell}}$ to $\rho$ and the uniform second moment bounds imply that
\[k'*R\rho_{t_{k_\ell}}+k*R(\rho_{t_{k_\ell}}\nabla_\theta  U)\rightarrow k'*R\rho_{\infty}+k*R(\rho_{\infty}\nabla_\theta  U)\]
pointwise in $\RD^d$. Fatou's Lemma thus implies that
\[\int_{\mathbb{D}} |k'*R\rho_\infty+k*R(\rho_\infty\nabla_\theta  U)|^2\leq \lim_{\ell\rightarrow \infty} \int_{\mathbb{D}} |k'*R\rho_{t_{k_\ell}}+k*R(\rho_{t_{k_\ell}}\nabla_\theta U)|^2=0,\]
hence 
\[k'*R\rho_\infty+k*R(\rho_\infty\nabla_\theta  U)=0,\]
almost surely. We thus conclude by Lemma~\ref{lem:equil_char} that $\rho_\infty=\pi$. Since $t_k$ was arbitrary, in total we have found that $\rho_t\rightarrow \pi$ weakly as claimed.
\end{proof}
\smallskip

\subsection*{Acknowledgments}
The authors are grateful to the National Science Foundation for the support under grants  DMS-220606, DMS-2342349, DMS-2407166, and DMS-2511684. In addition,
the first author is grateful to the Simons Laufer Mathematical Sciences Institute (supported by NSF grant DMS-2424139) where they were in residence during the Fall 2025 semester. The authors would also like to thank Patrick Flynn for illuminating discussions.

\appendix

\section{Optimal Transport and Gradient Flows}\label{appendix:Otto}

In this appendix we review the interpretation of the Wasserstein distance as the geodesic distance with respect to appropriate metric tensor on the space of probability measures as in \cite{jordan1998variational,BenBre00,otto2001geometry}. Furthermore, we show how to derive the Wasserstein-gradient-flow structure of the Fokker--Planck equation by using the Rayleigh functional. This motivates the derivation of the gradient flows for the Radon--Wasserstein and Regularized Radon--Wasserstein metric tensors in Section~\ref{sec:defs}.

We first recall the Benamou--Brenier characterization of the Wasserstein distance.

\begin{theorem}[\cite{BenBre00}]\label{thm:BenBre}
Consider $\mu, \nu \in \Pc_2(\R^d)$. Then
\[ \mathcal{W}^2(\mu, \nu) = \min_{(\rho,v) \in \mathcal A (\mu, \nu) } \int_0^1 \int_{\R^d} |v(x,t)|^2 d \rho_t(x) dt,\]
where $\mathcal A (\mu, \nu)$ is the set of all admissible paths between $\mu$ and $\nu$. That is, the set of all pairs $(\rho, v)$ where $\rho \in \mathcal{AC}([0,1], \Pc_2(\R^d))$ (absolutely continuous curves with respect to Wasserstein distance) is a solution to the continuity equation:
\begin{align*}
\begin{cases}
 \partial_t \rho_t + \nabla \cdot (v \rho_t) = 0,&\text{ on } \R^d \times [0,1] \\
 \rho_0  = \mu,\quad\rho_1 = \nu.   
\end{cases}
\end{align*}
\end{theorem}

Furthermore, the Benamou--Brenier theorem implies  that $\mathcal{W}$ is (formally) the Riemannian distance for the metric $\overline g_\rho$ introduced in \eqref{eq:Wass_tensor}.

Next, recall the derivation the Fokker--Planck equation as the gradient flow of the Kullback--Leibler divergence with respect to the Wasserstein metric tensor of Jordan, Kinderlehrer and Otto \cite{jordan1998variational}. We first note that in any Riemannian manifold the negative gradient is the minimizer of the Rayleigh functional over the tangent space. For the Wasserstein space with the tangent vectors in the density form~\eqref{eq:Wass_tensor}, the Rayleigh functional is
\begin{align*}
 \overline {\mathcal R}(s) & = \frac12 \overline g_\rho(s,s) + \text{diff}|_{\rho} \mathcal F(s) \\
 &  = \frac12 \inf_{v \::\:  - \nabla\cdot(\rho v) = s} g_\rho(v,v) + \text{diff}|_{\rho} \mathcal F(-\nabla\cdot(\rho v) )
\end{align*}
Minimizing $\overline {\mathcal{R} }(s)$ over $s$ is equivalent to minimizing the associated Lagrangian form of the  Rayleigh functional over all $v$, thus conveniently combining the two minimizations:
 \begin{equation} \label{eq:Ray}
  \mathcal R(v) = \frac12 g_\rho(v,v) + \text{diff}|_{\rho} \mathcal F(-\nabla \cdot(\rho v) )=\frac{1}{2}\int |v|^2\,d\rho+\int(\nabla\rho+\rho\nabla U)\cdot v.
\end{equation}
We readily see that the gradient velocity field is
\[-\text{grad}_g\mathcal{F}(\rho)=-\left(\frac{\nabla\rho}{\rho}+\nabla U\right),\]
thus the gradient flow is given by the Fokker--Planck equation
\[\partial_t\rho=\nabla\cdot(\rho\,\text{grad}_g\mathcal{F})=\Delta\rho+\nabla\cdot(\rho\nabla U)\]
as claimed.

\section{Entropy balance identity} \label{ap:continuity}

In this appendix we prove Lemma~\ref{lem:entropy_dynamics}. The proof is straightforward, essentially only using the flow map characterization of $\rho$ and changes of variables. Note that we do not assume any \textit{a priori} bounds on the solution $\rho$ such as finiteness of the Fisher information as in~\cite[Proposition 2.15 (v)]{Gianazza2009Wasserstein}.

\begin{proof}[Proof of Lemma~\ref{lem:entropy_dynamics}]
First we note that if $\mathcal{F}(\rho)<\infty$ then $\rho$ has a density, $\rho\in\mathcal{P}_2(\R^d)$, and $\int_{\R^d} |\rho(x)\log(\rho(x))|\,dx<\infty$. The absolute continuity is immediate and the second moment bound follows from the Donsker-Varadhan variational formula. The $L\log(L)$ integrability then follows after expanding out the definition of $\mathcal{F}(\rho)$.

By the regularity of $v$, it also holds that $\rho_t\in C([0,\infty),\mathcal{P}_2(\R^d))$ and $\rho_t=X_t\#\nu$ where $X_t$ is the flow map solving
\[\frac{d}{dt}X_t(x)=v(t,X_t(x)),\qquad X_0(x)=x.\]
Changing variables, 
\[\rho_t(x)= \frac{\nu\circ X_t^{-1}(x)}{\det\big(\nabla X_t\circ X_t^{-1}(x)\big)},\]
where $\det(\nabla X_t)>0$. This implies that
\begin{align}\label{eq:entropy_1}
&\notag\int_{\R^d} \rho_t(x)\log(\rho_t(x))\,dx
\\&\notag\qquad=\int_{\R^d} \nu\circ X_t^{-1}(x)\log(\nu\circ X_t^{-1}(x))\det\big(\nabla X_t\circ X_t^{-1}(x)\big)^{-1}\,dx
\\\notag&\quad\qquad-\int_{\R^d}\nu\circ X_t^{-1}(x)\log\big(\det\big(\nabla X_t\circ X_t^{-1}(x)\big)\big)\det\big(\nabla X_t\circ X_t^{-1}(x)\big)^{-1}\,dx
\\&\qquad=\int_{\R^d} \nu(y)\log(\nu(y))\,dy-\int_{\R^d} \nu(y)\log\big(\det(\nabla X_t(y))\big)\,dy,
\end{align}
where we have additionally used the change of variables $y=X_t^{-1}(x)$. The equalities above are all valid since the functions in the last equality are integrable. Indeed, by the Liouville formula for determinants it holds that
\[\log(\det(\nabla X_t))(x)=\int_0^t\nabla\cdot v(s,X_s(x))\,ds,\]
thus, since $\nabla\cdot v$ is boudned uniformly in space and time
\begin{equation}\label{eq:entropy_2}\int_{\R^d} \nu(y)\log(\det(\nabla X_t(y))\,dy=\int_0^t\ \int_{\R^d} \nu(y)\nabla\cdot v(s,X_s(y))\,dy\,ds=\int_0^t\ \int_{\R^d} \rho_s(y)\nabla\cdot v(s,y)\,dy\,ds
\end{equation}
by the Tonelli/Fubini theorems.

Next, for all $x$ and $t$
\[U(X_t(x))=U(x)+\int_0^t \nabla U(X_s(x))\cdot v(s,X_s(x))\,ds\]
by the Fundamental Theorem of Calculus. It thus holds that
\begin{align}\label{eq:entropy_3}
 \notag \int_{\R^d} U(x)\rho_t(x)\,dx-\int_{\R^d} U(x)\nu(x)\,dx&=\int_0^t \int_{\R^d} \nabla U(X_s)\cdot v(s,X_s)\nu(x)\,dx\,ds\\
 &=\int_0^t \int_{\R^d} \nabla U(y)\cdot v(s,y)\rho_s(y)\,dy\,ds.
\end{align}
Here all equalities hold due to the conditions on $U$ and $v$ and the fact that $\rho_t$ has uniformly bounded second moments locally in time.

Combining~\eqref{eq:entropy_1}-\eqref{eq:entropy_3}, in total we have found that
\begin{align*}
    \mathcal{F}(\rho_t)&=\int_{\R^d}  \rho_t(x)\log(\rho_t(x))\,dx+\int_{\R^d} U(x)\rho_t(x)\,dx+C
    \\&=\int_{\R^d} \nu(y)\log(\nu(y))\,dy+\int_{\R^d} U(y)\nu(y)\,dy+C
    \\&\qquad-\int_0^t \int_{\R^d} (\nabla\cdot v(s,x)-\nabla U(y)\cdot v(s,y))\rho_s(y)\,dy\,ds
    \\&=\mathcal{F}(\nu)-\int_0^t \int_{\R^d} \big(\nabla\cdot v(s,y)-\nabla U(y)\cdot v(s,y)\big)\rho_s(y)\,dy\,ds,
\end{align*}
as claimed.
\end{proof}

\section{Proofs of theoretical results}\label{app:theory}

This appendix collects (essentially) standard well-posedness and stability results for nonlinear continuity equations, stated under a simple local Lipschitz growth assumption on the velocity field. We first introduce the abstract assumptions and verify that the specific velocity fields defined in~\eqref{eq:KDRWgf} and~\eqref{eq:RRWgf} satisfy them. We then prove global well-posedness and moment bounds in Subsection~\ref{subsec:wp},  Wasserstein stability with respect to initial data in Subsection~\ref{subsec:stability}, and error bounds for deterministic Euler discretizations and stochastic schemes in Subsection~\ref{subsec:stoc_conv}.

\subsection{Abstract setting and Lemma~\ref{lem:theta_vector_bounds}}\label{subsec:abstract}

We work with nonlinear continuity equations on $\R^d$ driven by a nonlinear velocity field
$V:\R^d\times\Pc_2(\R^d)\to\R^d$. The results in Subsections~\ref{subsec:wp}--\ref{subsec:stoc_conv}
will be proved under the following linear-growth and Lipschitz assumption on $V$.

\begin{condition}\label{cond:velocity_field}
There exists $M>0$ so that for all $x,y\in\R^d$ and $\nu,\mu\in \Pc_2(\R^d)$
\begin{enumerate}
    \item\label{item:bounded} $|V(x,\mu)|\leq M(1+\|\mu\|_{\Pc_2})$,
    \item\label{item:lipschitz} $|V(x,\mu)-V(y,\nu)|\leq M(1+\|\mu\|_{\Pc_2})(|x-y|+\mathcal{W}(\mu,\nu))$.
\end{enumerate}
\end{condition}

Proposition~\ref{prop:specific_velocity} in the main text thus states that the explicit velocity fields
associated with~\eqref{eq:KDRWgf} and~\eqref{eq:RRWgf} satisfy Condition~\ref{cond:velocity_field}.
To verify this it is convenient to use a spherical-averaging representation of the velocity field. This will also be convenient for showing convergence of the stochastic descent scheme in Subsection~\ref{subsec:stoc_conv}.

We will therefore consider velocity fields $V$ of the form
\begin{equation}\label{eq:V_spherical_average}
V(x,\mu)\;:=\;\fint_{\S^{d-1}} u\bigl(\theta,x\!\cdot\!\theta,\mu\bigr)\,d\theta
\end{equation}
where $u:\RD^d\times\Pc_2(\R^d)\to\R^d$.
The velocity fields~\eqref{eq:KDRW_vec} and~\eqref{eq:RRW_vec} are of this form, respectively corresponding to the
kernels $u(\theta,p,\mu)$ and $k*u(\theta,p,\mu)$ defined by~\eqref{eq:spec_integral_kern}.

The following kernel-level condition is a uniform-in-$\theta$ analogue of
Condition~\ref{cond:velocity_field}, and implies Condition~\ref{cond:velocity_field} for the averaged field
$V$ in~\eqref{eq:V_spherical_average} by a direct argument.

\begin{condition}\label{cond:theta_velocity_field}
There exists $M>0$ so that for all $\theta\in\S^{d-1}$, $p,q\in\R$, and $\nu,\mu\in \Pc_2(\R^d)$
\begin{enumerate}
    \item\label{item:theta_bounded} $|u(\theta,p,\mu)|\leq M(1+\|\mu\|_{\Pc_2})$,
    \item\label{item:theta_lipschitz} $|u(\theta,p,\mu)-u(\theta,q,\nu)|\leq M(1+\|\mu\|_{\Pc_2})(|p-q|+\mathcal{W}(\mu,\nu))$.
\end{enumerate}
\end{condition}

With this reduction, Proposition~\ref{prop:specific_velocity} follows from the next lemma.

\begin{lemma}\label{lem:theta_vector_bounds}
Let $u:\RD^d\times \Pc_2(\R^d)\to\R^d$ be defined by~\eqref{eq:spec_integral_kern}, $\eps>0$, and $k$ satisfy Condition~\ref{cond:k_conditions}. Then there exists $M(L,|\nabla U(0)|,\|k\|_{C^2_b},\eps^{-1})>0$ such that $u$ and $k*u$ satisfy Condition~\ref{cond:theta_velocity_field} with constant $M$.
\end{lemma}

\begin{proof}
For convenience we let
\[
f(\theta,p,\mu):=(k'*R^\theta\mu)(p) + (k*R^\theta(\mu\nabla_\theta U))(p)
\qquad\text{and}\qquad
g(\theta,p,\mu):=(k*R^\theta\mu)(p)+\eps
\]
so that
\[
u(\theta,p,\mu)= -\theta\frac{f(\theta,p,\mu)}{g(\theta,p,\mu)}.
\]
Then
\begin{align*}
  |f(\theta,p,\mu)|
  &\leq \|k'*R^\theta\mu\|_{L^\infty(\R)}+\|k*R^\theta(\mu\nabla_\theta U)\|_{L^\infty(\R)}
  \\&\leq\|k'\|_{L^\infty(\R)}+\|k\|_{L^\infty(\R)}\int_{\R^d} |\nabla U|(x)\,d\mu(x)
  \\&\leq \|k'\|_{L^\infty(\R)}+\|k\|_{L^\infty(\R)}(|\nabla U(0)|+\|\nabla U\|_{\text{Lip}}\|\mu\|_{\Pc_2}),
\\&\leq C(1+\|\mu\|_{\Pc_2}),
\end{align*}
for some $C>0$ where we have applied H\'older's inequality in the second to last inequality. By almost identical computations we also have that
\[
\Big|\frac{\partial}{\partial p}f(\theta,p,\mu)\Big|
\leq \|k''\|_{L^\infty(\R)}+\|k'\|_{L^\infty(\R)}\int_{\R^d} |\nabla U|(x)\,d\mu(x)
\leq C(1+\|\mu\|_{\Pc_2}),
\]
for some $C>0$.

Next, letting \(\gamma\in\Gamma(\mu,\nu)\) be optimal,
\begin{align*}
|(k'*R^\theta\mu)(p)-(k'* R^\theta\nu)(p)|
&\leq \int_{\R^d\times\R^d} |k'(p-x\cdot\theta)-k'(p-y\cdot\theta)|\,d\gamma(x,y)\\
&\leq \|k''\|_{L^\infty(\R)}\int_{\R^d\times \R^d}|x-y|\,d\gamma(x,y)
\leq \|k''\|_{L^\infty(\R)}\mathcal{W}(\mu,\nu),
\end{align*}
where we used \(|(x-y)\cdot\theta|\le |x-y|\) and H\"older's inequality in the last step. Similarly,
\begin{align*}
&|(k*R^\theta (\mu\nabla_\theta U))(p)-(k*R^\theta(\nu\nabla_\theta U))(p)|\\
&\qquad\leq \int_{\R^d\times\R^d}| k(p-x\cdot\theta)\nabla U(x)-k(p-y\cdot\theta)\nabla U(y)|\,d\gamma(x,y)\\
&\qquad\leq \|k'\|_{L^\infty}\int_{\R^d\times\R^d}|x-y||\nabla U(x)|\,d\gamma(x,y)
+\|k\|_{L^\infty}\|\nabla U\|_{\text{Lip}}\int_{\R^d\times\R^d}|x-y|\,d\gamma(x,y).
\end{align*}
H\"older's inequality implies that
\begin{align*}
\int_{\R^d\times\R^d}|x-y||\nabla U(x)|\,d\gamma(x,y)
&\leq \left(\int_{\R^d} |\nabla U(x)|^2\,d\mu(x)\right)^{1/2}
\left(\int_{(\R^d)^2} |x-y|^2\,d\gamma(x,y)\right)^{1/2}\\
&\leq C(1+\|\mu\|_{\Pc_2})\mathcal{W}(\mu,\nu),
\end{align*}
thus
\[
|(k*R^\theta (\mu\nabla_\theta U))(p)-(k*R^\theta(\nu\nabla_\theta U))(p)|
\leq C(1+\|\mu\|_{\Pc_2})\mathcal{W}(\mu,\nu),
\]
for some \(C>0\).  Altogether, the displays above imply that
\[
|f(\theta,p,\mu)-f(\theta,p,\nu)|
\leq C(1+\|\mu\|_{\Pc_2})\mathcal{W}(\mu,\nu).
\]

By very similar, but more straightforward computations we have
\[
\eps\leq g(\theta,p,\mu)\leq \|k\|_{L^\infty(\R)}+\eps,
\quad
\Big|\frac{\partial}{\partial p} g(\theta,p,\mu)\Big|\leq \|k'\|_{L^\infty(\R)},\]
and
\[|g(\theta,p,\mu)-g(\theta,p,\nu)|\leq \|k'\|_{L^\infty(\R)}\,\mathcal{W}(\mu,\nu).\]
Together, these bounds imply that
\[
\bigg|\theta\frac{f(\theta,p,\mu)}{g(\theta,p,\mu)}\bigg|
\leq \frac{C}{\eps}(1+\|\mu\|_{\Pc_2}),
\]
and, using the bounds on $\frac{\partial}{\partial p}f$ and $\frac{\partial}{\partial p} g$,
\[
|f(\theta,p,\mu)-f(\theta,q,\mu)|\le C(1+\|\mu\|_{\Pc_2})|p-q|,
\qquad
|g(\theta,p,\mu)-g(\theta,q,\mu)|\le C|p-q|.
\]
Therefore, using the decomposition
\[
\frac{f(\theta,p,\mu)}{g(\theta,p,\mu)}-\frac{f(\theta,q,\nu)}{g(\theta,q,\nu)}
=
\frac{f(\theta,p,\mu)-f(\theta,q,\nu)}{g(\theta,p,\mu)}
+\frac{f(\theta,p,\mu)\big(g(\theta,q,\nu)-g(\theta,p,\mu)\big)}{g(\theta,p,\mu)\,g(\theta,q,\nu)},
\]
and the lower bound \(g\ge \eps\), we obtain
\[
\bigg|\theta\frac{f(\theta,p,\mu)}{g(\theta,p,\mu)}-\theta\frac{f(\theta,q,\nu)}{g(\theta,q,\nu)}\bigg|
\leq \frac{C}{\eps^2}(1+\|\mu\|_{\Pc_2})\big(|p-q|+\mathcal{W}(\mu,\nu)\big).
\]
Inspecting the proof, it is clear that $C>0$ depends on $L$ and $|\nabla U(0)|$ and $\|k\|_{C^2_b}$, thus the claim holds.
\end{proof}

\subsection{Well-posedness}\label{subsec:wp}

Under Condition~\ref{cond:velocity_field}, the nonlinear continuity equation can be solved by characteristics: for a candidate curve $t\mapsto \rho_t$, the field 
$x\mapsto V(x,\rho_t)$ generates a flow map, and a fixed-point argument yields a self-consistent pushforward solution. Our proof follows closely to~\cite{CarmonaDelarue2018I}, although here there is no noise, and we are no longer considering globally Lipschitz velocity fields.
    
\begin{theorem}\label{thm:transport_wp}
For all $V$ satisfying Condition~\ref{cond:velocity_field} and $\nu\in\Pc_2(\R^d)$ there exists a unique solution $\rho\in C([0,\infty),\Pc_2(\R^d))$ to
\begin{equation}\label{eq:non-linear_continuity_equation}
\begin{cases}
 \partial_t \rho+\nabla\cdot(\rho v)=0
 \\\rho|_{t=0}=\nu
\end{cases}
\end{equation}
with $v(t,x)=V(x,\rho_t)$. More so, there exists $C(M)>0$ so that for all $t\geq 0$
\begin{equation}\label{eq:moment_bound}
 \|\rho_t\|_{\Pc_2}\leq e^{Ct}(\|\nu\|_{\Pc_2}+1).   
\end{equation}
\end{theorem}

\begin{proof}
 We first establish local-in-time existence of the continuity equation before showing global-in-time existence.

\textit{Local-in-time existence:} Fix $T>0$ to be determined. Then, for any $\rho\in C([0,T],\Pc_2(\R^d))$, if $v^\rho(t,x):=V(x,\rho_t)$, then $v^\rho\in C([0,T],W^{1,\infty}(\R^d)).$ Thus, by standard ODE theory, there exists a flow map $X^\rho:[0,T]\times \R^d\rightarrow \R^d$ satisfying
\[
\begin{cases}
    \frac{d}{dt}X_t^\rho(x)=V(X_t^\rho(x),\rho_t)
    \\X_t^\rho(x)=x.
\end{cases}
\]

Next, let 
\[\mathcal{C}^\nu_T:=\Big\{\rho\in C([0,T],\Pc_2(\R^d)):\rho_0=\nu\text{ and }\sup_{t\in[0,T]}\|\rho_t\|_{\Pc_2}\leq (1+2\|\nu\|_{\Pc_2})\Big\}.\]
Then the function $\Phi:\mathcal{C}^\nu_T\to\mathcal{C}^\nu_T$ defined by
\[(\Phi(\rho))_t:=X_t^\rho\#\nu,\]
is well-defined if $T$ is sufficiently small. Indeed, $(\Phi(\rho))_0=\nu$ and Item~\ref{item:bounded} in Condition~\ref{cond:velocity_field} implies that for all $x$ and $t\in[0,T]$ and $\rho\in C^\nu_T$
\[|X_t^\rho(x)|\leq |x|+\int_0^t |V(X_s^\rho(x),\rho_s)|\,ds\leq |x|+2TM(1+\|\nu\|_{\Pc_2}).\]
This immediately implies that
\[\|X_t^\rho\#\nu\|_{\Pc_2}=\|X_t^\rho\|_{L^2(\nu)}\leq \|x\|_{L^2(\nu)}+2TM(1+ \|\nu\|_{\Pc_2})= \|\nu\|_{\Pc_2}+2TM(1+\|\nu\|_{\Pc_2}),\]
thus, if $T\leq \frac{1}{2M}$, then $\|X_t^\rho\#\nu\|_{\Pc_2}\leq 1+2\|\nu\|_{\Pc_2}$. That is, $\Phi(\rho)\in\mathcal{C}^\nu_T$.

Next, we will show that $\Phi$ defines a contraction on $\mathcal{C}^\nu_T$ if $T$ is sufficiently small, and thus has a fixed point. Indeed, suppose that $\rho,\rho'\in\mathcal{C}^\nu_T$. Then, for any $x$ and $t$, Item~\ref{item:lipschitz} in Condition~\ref{cond:velocity_field} implies that
\begin{align*}
  |X_t^\rho(x)-X_t^{\rho'}(x)|&\leq \int_0^t |V(X_s^\rho(x),\rho_s)-V(X_s^{\rho'}(x),\rho'_s)|\,ds
  \\&\leq 2M(1+\|\nu\|_{\Pc_2})\int_0^t |X_s^\rho(x)-X_s^{\rho'}(x)|+\mathcal{W}(\rho_s,\rho'_s)\,ds 
  \\&\leq 2M(1+\|\nu\|_{\Pc_2})\left(\int_0^t |X_s^\rho(x)-X_s^{\rho'}(x)|\,ds+T\sup_{s\in[0,T]}\mathcal{W}(\rho_s,\rho'_s)\right).
\end{align*}
Integrating the above over $\nu$ and applying Minkowski's inequality we find that
\[\|X_t^\rho-X_t^{\rho'}\|_{L^2(\nu)}\leq 2M(1+\|\nu\|_{\Pc_2})\left(\int_0^t \|X_s^\rho-X_s^{\rho'}\|_{L^2(\nu)}\,ds+T\mathcal{W}(\rho_s,\rho'_s)\right),\]
thus
\[\mathcal{W}((\Phi(\rho))_t,(\Phi(\rho'))_t)\leq \|X_t^\rho-X_t^{\rho'}\|_{L^2(\nu)}\leq 2TM(1+\|\nu\|_{\Pc_2})e^{2TM(1+\|\nu\|_{\Pc_2})} \sup_{s\in[0,T]}\mathcal{W}(\rho_s,\rho'_s)\]
by Gr\"onwall's inequality. If $T\leq (4M(1+\|\nu\|_{\Pc_2}))^{-1}$ then the prefactor is less than 1, and thus $\Phi$ is indeed a contraction.

As a consequence, the Banach fixed-point theorem implies that there exists a unique $\rho\in\mathcal{C}^\nu_T$ so that $\rho_t=(\Phi(\rho))_t$. Since this $\rho$ satisfies $\rho_t=X_t^\rho\#\nu$, $\rho_t$ is a solution~\eqref{eq:non-linear_continuity_equation} \cite[Lemma~8.1.6.]{ambrosio2005gradient}.

\textit{Global-in-time existence}: Given the restrictions on $T$ in the previous step, the local solution can clearly be extended to a global one as long as $\|\rho_t\|_{\Pc_2}$ remains finite. To this end, we note that
\[\|\rho_t\|_{\Pc_2}=\|X_t^\rho\|_{L^2(\nu)}\leq \|\nu\|_{\Pc_2}+M\int_0^t(1+\|\rho_s\|_{\Pc_2})\,ds,\]
thus Gr\"onwall's inequality implies that
\[\|\rho_t\|_{\mathcal{P}_2}\leq e^{Ct}(\|\nu\|_{\Pc_2}+1),\]
for some $C(M)>0$. We thus conclude both that a solution exists on $[0,\infty)$ and the inequality~\eqref{eq:moment_bound} holds.
\end{proof}

The well-posedness result above immediately yields existence and uniqueness for the associated
finite-$n$ particle system. Indeed, when the initial law is empirical, the unique solution of the
continuity equation is the pushforward of that empirical measure by the characteristic flow, and hence
remains empirical for all times with particle trajectories satisfying the corresponding ODE.

\begin{corollary}\label{cor:ode_wp}
    Let $V$ satisfy the conditions of Theorem~\ref{thm:transport_wp}. Then, for any initial condition $\ux_0^n\in(\R^d)^n$ there exists a unique solution to the ODE
    \begin{equation}\label{eq:ode}
    \frac{d}{dt} x_t^i= V(x_t^i,\mu_t^n),
    \end{equation}
    where $\mu_t^n:=\frac{1}{n}\sum_{i=1}^n\delta_{x_t^i}$. More so, there exists $C(M)>0$ such that $\|\mu_t^n\|_{\Pc_2}\leq e^{Ct}(\|\mu_0^n\|_{\Pc_2}+1)$.
\end{corollary}

\begin{proof} Let $\nu=\frac{1}{n}\sum_{i=1}^n\delta_{x_0^i}\in \Pc_2(\R^d)$ and $\rho$ be the (unique) solution of~\eqref{eq:non-linear_continuity_equation} with initial conditions $\nu$. Then, since $v\in C([0,\infty),W^{1,\infty}(\R^d))$, $\rho_t=X_t\#\nu$ where $X_t$ is the flow map
\[
\begin{cases}
   \frac{d}{dt}X_t(x)=v(t,X_t(x)),
   \\X_0(x)=x.
\end{cases}\]
Letting $x_t^i:=X_t(x_0^i)$ we thus find that $\rho_t=\frac{1}{n}\sum_{i=1}^n\delta_{x_t^i}$ and
\[\frac{d}{dt}x_t^i=v(t,x_t ^i)=V(x_t^i,\mu_t^n).\]
We have thus shown that there exists a solution to the ODE~\eqref{eq:ode}.

Suppose that $\uy_t^n$ is another solution of~\eqref{eq:ode}. Then $\rho_t':=\frac{1}{n}\sum_{i=1}^n\delta_{y_i^t}$ defines a solution to~\eqref{eq:non-linear_continuity_equation}, hence $\rho=\rho'$ by the uniqueness of solutions to~\eqref{eq:non-linear_continuity_equation}. This immediately implies that $x_t^i=y_t^i$.
\end{proof}

\subsection{Stability}\label{subsec:stability}

The next result quantifies continuous dependence on the initial law in $\mathcal{P}_2(\R^d)$. The proof is a standard Dobru\v{s}hin’s coupling argument (see~\cite{MR541637,MR3468297}), using the moment bound from Theorem~\ref{thm:transport_wp} to control the Lipschitz constant in time.

\begin{theorem}\label{thm:stability}
Suppose $V$ satisfies Condition~\ref{cond:velocity_field} and $\rho$ and $\rho'$ are the unique solutions to~\eqref{eq:non-linear_continuity_equation} with initial conditions respectively given by $\nu$ and $\mu$ in $\Pc_2(\R^d)$. Then, for all $T>0$ there exists $C(M,T,\|\nu\|_{\Pc_2})>0$ so that
\[\sup_{t\in[0,T]}\mathcal{W}(\rho_t,\rho'_t)\leq C\mathcal{W}(\nu,\mu).\]
\end{theorem}

\begin{proof}
Let $X_t^\rho$ and $X_t^{\rho'}$ be the flow maps so that $\rho_t=X_t^\rho\#\nu$ and $\rho'_{t}=X^{\rho'}_t\#\mu,$ and fix $T>0$.  Then, it holds that for all $t\leq T$
\begin{align*}
|X_t^\rho(x)-X_t^{\rho'}(y)|&\leq |x-y|+\int_0^t |V(X_s^\rho(x),\rho_s)-V(X_s^{\rho'}(y),\rho'_s)|\,ds.
\\&\leq |x-y|+M\int_0^t(1+\|\rho_s\|_{\Pc_2})\big(|X_s^\rho(x)-X_s^{\rho'}(y)|+\mathcal{W}(\rho_s,\rho'_s)\big)\,ds
\\&\leq |x-y|+C\int_0^t |X_s^\rho(x)-X_s^{\rho'}(y)|+\mathcal{W}(\rho_s,\rho'_s)\,ds,
\end{align*}
for some $C>0$. Thus, for any coupling $\gamma\in\Gamma(\mu,\nu)$, letting
\[D_t(\gamma):=\left(\int_{(\R^d)^2}|X_t^\rho(x)-X_t^{\rho'}(y)|^2\,d\gamma(x,y)\right)^{1/2},\]
Minkowski's inequality implies that
\begin{align*}
D_t(\gamma)\leq D_0(\gamma)+C\int_0^t\mathcal{W}(\rho_s,\rho'_s)+D_s(\gamma)\,ds.
\end{align*}
Applying Gr\"onwall's inequality we thus find that for all $t\in[0,T]$
\[ D_t(\gamma)\leq C\left(D_0(\gamma)+\int_0^t\mathcal{W}(\rho_s,\rho'_s)\,ds)\right).\]
Minimizing over all $\gamma$, we have that
\[\mathcal{W}(\rho_t,\rho'_t)\leq C\left(\mathcal{W}(\mu,\nu)+C\int_0^t\mathcal{W}(\rho_s,\rho'_s)\,ds\right),\]
and applying Gr\"onwall's inequality once more we conclude. By keeping track of the constants, we see that $C$ depends on $M$, $T$, and $\|\nu\|_{\Pc_2}$. 
\end{proof}

\subsection{Stochastic scheme convergence}\label{subsec:stoc_conv}

Finally, we prove the convergence of a stochastic descent schemes for velocity fields of the form~\eqref{eq:V_spherical_average}. For this purpose, we compare three dynamics: the exact particle system $\ux_t^n$ driven by $V$, its deterministic
time-discretization $\uy_t^{n,\tau}$, and the stochastic scheme $\ux_t^{n,\tau}$ obtained by replacing the spherical average defining $V$ with a single i.i.d.\
random direction at each step. Our goal is to control the cumulative error over a fixed horizon $T$
by bounding both the time-discretization error and the direction-sampling error.

We first define the (forward Euler) time discretization of~\eqref{eq:ode}

\begin{definition}
For $\tau\in(0,1]$ and initial conditions $\uy_0^n\in (\R^d)^n$, let $\uy_t^{n,\tau}$ be the process inductively defined by
\begin{equation}\label{eq:ODE_approx}
y_t^{i,\tau}=y_{m\tau}^{i,\tau}+(t-m\tau)V(y_{m\tau}^{i,\tau},\nu_{m\tau}^{n,\tau}),
\end{equation}
when $t\in(m\tau,(m+1)\tau]$ where $\nu_t^{n,\tau}:=\frac{1}{n}\sum_{i=1}^n\delta_{y_t^{i,\tau}}$.
\end{definition}

We then bound the error between the forward Euler approximation and the actual ODE solution. This is essentially standard, using the moment bound from Theorem~\ref{thm:transport_wp} to control the Lipschitz constant in time.

\begin{proposition}\label{prop:ODE_approx}
Suppose that $V$ satisfies Condition~\ref{cond:velocity_field}. Then, if
$\ux_t^n$ and $\uy_t^n$ are solutions to~\eqref{eq:ode} and~\eqref{eq:ODE_approx} with the same initial condition $\ux_0^n\in(\R^d)^n$, for all $T\geq 0$ there exists $C(M,T,\|\mu_0^n\|_{\Pc_2})>0$ so that
\[\sup_{t\in[0,T]}\left(\frac{1}{n}\sum_{i=1}^n|x_t^i-y_t^{i,\tau}|^2\right)^{1/2}\leq C\tau.\]
\end{proposition}

\begin{proof}
For convenience, let $t_m=\tau m$. Let $u^n:(\R^d)^n\rightarrow \R^d$ be the velocity field with $i$-th component given by
\[u^i(\ux^n)=V\left(x^i,\frac{1}{n}\sum_{j=1}^n\delta_{x^j}\right).\]
Then
\[\frac{d}{dt}x_t^i=u^i(\ux_t^n)\]
and
\[y^{i,\tau}_t=y^{i,\tau}_{t_m}+(t-t_m)u^i(\uy^{n,\tau}_{t_m}),\]
thus $y_t^{n,\tau}$ is the forward Euler scheme for $x_t^{n}$ with step size $\tau$.

Next, we note that since $\mu_t^n$ is the unique solution to~\eqref{eq:non-linear_continuity_equation} with initial condition $\nu=\frac{1}{n}\sum_{i=1}^n\delta_{x_0^i}$, Theorem~\ref{thm:transport_wp} implies that there exists $C>0$ so that $\sup_{t\in[0,T]}\|\mu_t^n\|_{\Pc_2}\leq C$. The definition of $u^i$ thus imply that
\begin{align*}
 |u^i(\ux_t^n)-u^i(\uy_t^{n,\tau})|&\leq C\left(|x_t^i-y_t^{i,\tau}|+\mathcal{W}(\mu_t^n,\nu_t^{n,\tau})\right)
 \\&\leq C\left(|x^i_t-y_t^{i,\tau}|+\left(\frac{1}{n}\sum_{j=1}^n|x_t^{j,\tau}-y_t^{j,\tau}|^2\right)^{1/2}\right),   
\end{align*}
hence
\[|u^n(\ux_t^n)-u^n(\uy_t^{n,\tau})|\leq C|\ux_t^n-\uy_t^{n,\tau}|.\]
By classical bounds on the error generated by forward Euler schemes (see, e.g., \cite[Chapter~I, Section~I.7]{HairerNorsettWanner1993}), we obtain
\[\sup_{t\in[0,T]}\left(\frac{1}{n}\sum_{i=1}^n|x_t^i-y_t^{i,\tau}|^2\right)^{1/2}\leq C\tau.\]
By keeping track of the constants, we see that $C$ only depends on $M$, the time horizon $T$, and the second moment of the initial conditions $\|\mu_0^n\|_{\Pc_2}$.
\end{proof}

Next, we define the corresponding stochastic approximation when $V$ is of the form~\eqref{eq:V_spherical_average}.

\begin{definition}
For $\tau\in(0,1]$ and initial conditions $\uy_0^N\in (\R^d)^N$, let $\uy_t^N$ be the process inductively defined by
\begin{equation}\label{eq:stoc_approx}
x_t^{i,\tau}=x_{m\tau}^{i,\tau}+(t-m\tau)u(\theta_m,x_{m\tau}^{i,\tau}\cdot\theta_m,\mu_{m\tau}^{n,\tau}),
\end{equation}
when $t\in(m\tau,(m+1)\tau]$ where $\mu_t^{n,\tau}:=\frac{1}{n}\sum_{i=1}^n\delta_{x_t^{i,\tau}}$ and $\{\theta_m\}_{m\geq 0}$ are i.i.d.\ uniformly sampled from $\S^{d-1}$.
\end{definition}

The following moment bounds for $\uy_t^{n,\tau}$ and $\ux_t^{n,\tau}$ will be necessary for the main result.

\begin{lemma}\label{lem:discrete_a_priori}
Suppose that $u$ satisfies Condition~\ref{cond:theta_velocity_field}, $V$ is defined by~\eqref{eq:V_spherical_average}, and $\uy_t^{n,\tau}$ and $\ux_t^{n,\tau}$ are respectively given by~\eqref{eq:ODE_approx} and~\eqref{eq:stoc_approx}. Then, there exists $C(M)>0$ so that for all $t\geq 0$
\[\|\nu_t^{n,\tau}\|_{\Pc_2}\leq e^{Ct}(\|\mu_0^{n}\|_{\Pc_2}+1)\quad\text{and}\quad \|\mu_t^{n,\tau}\|_{\Pc_2}\leq e^{Ct}(\|\mu_0^{n}\|_{\Pc_2}+1).\]

\end{lemma}

\begin{proof}
First, we note that by assumption,
\[|y_t^{i,\tau}|\leq |y_{t_m}^{i,\tau}|+\tau|u^i(\uy_{t_m}^{n,\tau})|=|y_{t_m}^{i,\tau}|+\tau |V(y_{t_m}^{i,\tau},\nu_{t_m}^{n,\tau})|\leq |y_{t_m}^{i,\tau}|+\tau C(1+\|\nu_{t_m}^{n,\tau}\|_{\Pc_2})\]
thus,
\[\|\nu_t^{n,\tau}\|_{\Pc_2}\leq (1+C\tau)\|\nu_{t_m}^{n,\tau}\|_{\Pc_2}+C\tau.\]
Iterating this bound, we find that 
\[\|\nu_t^{n,\tau}\|_{\Pc_2}\leq e^{Ct}(\|\nu_0^{n,\tau}\|_{\Pc_2}+1)\] 
for some $C(M)>0$ as claimed. The bound for $\mu_t^{n,\tau}$ follows almost identically.
\end{proof}

We now state the main result of this subsection.

\begin{theorem}\label{thm:sgd_conv}
Suppose that $u$ satisfies Condition~\ref{cond:theta_velocity_field}, $V$ is defined by~\eqref{eq:V_spherical_average}, and $\ux_t^n$ and $\ux_t^{n,\tau}$ are respectively solutions to~\eqref{eq:ode} and~\eqref{eq:stoc_approx} with the same initial condition $\ux_0^n\in(\R^d)^n$. Then, for all $T\geq 0$ there exists $C(M,T,\|\mu_0^n\|_{\Pc_2})>0$ such that
\begin{equation}
\sup_{t\in[0,T]}\mathbb{E}\left[\frac{1}{n}\sum_{i=1}^n|x_t^i-x_t^{i,\tau}|^2 \right]\leq C\tau .
\end{equation}
\end{theorem}

\begin{proof}
For all $T>0$, Proposition~\ref{prop:ODE_approx} implies that there exists $C>0$ so that
\begin{align*}
\sup_{t\in[0,T]}\mathbb{E}\left[\frac{1}{n}\sum_{i=1}^n|x_t^i-x_t^{i,\tau}|^2\right]&\leq 2\sup_{t\in[0,T]}\frac{1}{n}\sum_{i=1}^n|x_t^i-y_t^{i,\tau}|^2+2\sup_{t\in[0,T]}\mathbb{E}\left[\frac{1}{n}\sum_{i=1}^n|y_t^{i,\tau}-x_t^{i,\tau}|^2\right] \\&\leq C\tau^2+2\sup_{t\in[0,T]}\mathbb{E}\left[\frac{1}{n}\sum_{i=1}^n|y_t^{i,\tau}-x_t^{i,\tau}|^2\right].
\end{align*}
We thus only need to bound the last term in the last line.

By the definition of $V$,
\[y_t^{i,\tau}=y_{t_m}^{i,\tau}+(t-t_m)\fint_{\S^{d-1}}u(\theta,y_{t_m}^{i,\tau}\cdot\theta,\nu_{t_m}^{    n   ,\tau})\,d\theta\]
for all $t\in[t_m,t_{m+1}]$. Differentiating, we thus have that
\begin{align}\label{eq:sgd_conv_1}\notag \frac{d}{dt}|y_t^{i,\tau}-x_t^{i,\tau}|^2&=(y_t^{i,\tau}-x_t^{i,\tau})\left(\fint_{\S^{d-1}}u(\theta,y_{t_m}^{i,\tau}\cdot\theta,\nu_{t_m}^{n,\tau})\,d\theta-u(\theta_m,x_{t_m}^{i,\tau}\cdot\theta_m,\mu_{t_m}^{n,\tau})\right)
\\\notag & =(y_{t_m}^{i,\tau}-x_{t_m}^{i,\tau})\left(\fint_{\S^{d-1}}u(\theta,y_{t_m}^{i,\tau}\cdot\theta,\nu_{t_m}^{n,\tau})\,d\theta-u(\theta_m,x_{t_m}^{i,\tau}\cdot\theta_m,\mu_{t_m}^{n,\tau})\right)
\\&\quad+(t-t_m)\left(\fint_{\S^{d-1}}u(\theta,y_{t_m}^{i,\tau}\cdot\theta,\nu_{t_m}^{n,\tau})\,d\theta-u(\theta_m,x_{t_m}^{i,\tau}\cdot\theta_m,\mu_{t_m}^{n,\tau})\right)^2.
\end{align}
Let $\mathbb{E}_m$ denote taking expectation with respect to $\theta_m$. Then
\begin{align*}
&\mathbb{E}_m\left[(y_{t_m}^{i,\tau}-x_{t_m}^{i,\tau})\left(\fint_{\S^{d-1}}u(\theta,y_{t_m}^{i,\tau}\cdot\theta,\nu_{t_m}^{n,\tau})\,d\theta-u(\theta_m,x_{t_m}^{i,\tau}\cdot\theta_m,\mu_{t_m}^{n,\tau})\right)\right]
\\&\quad=(y_{t_m}^{i,\tau}-x_{t_m}^{i,\tau})\fint_{\S^{d-1}}\left(u(\theta,y_{t_m}^{i,\tau}\cdot\theta,\nu_{t_m}^{n,\tau})-u(\theta,x_{t_m}^{i,\tau}\cdot\theta,\mu_{t_m}^{n,\tau})\right)\,d\theta
\end{align*}
where we've used that $x_{t_m}^{i,\tau}$ is independent of $\theta_m$. Condition~\ref{cond:theta_velocity_field} and Lemma~\ref{lem:discrete_a_priori} imply that
\[\bigg|\fint_{\S^{d-1}}\left(u(\theta,y_{t_m}^{i,\tau}\cdot\theta,\nu_{t_m}^{n,\tau})-u(\theta,x_{t_m}^{i,\tau}\cdot\theta,\mu_{t_m}^{n,\tau})\right)\,d\theta\bigg|\leq C(|y_{t_m}^{i,\tau}-x_{t_m}^{i,\tau}|+\mathcal{W}(\nu_{t_m}^{n,\tau},\mu_{t_m}^{n,\tau})),\]
hence
\begin{multline}\label{eq:sgd_conv_2}
\bigg|\mathbb{E}_m\left[(y_{t_m}^{i,\tau}-x_{t_m}^{i,\tau})\left(\fint_{\S^{d-1}}u(\theta,y_{t_m}^{i,\tau}\cdot\theta,\nu_{t_m}^{n,\tau})\,d\theta-u(\theta_m,x_{t_m}^{i,\tau}\cdot\theta_m,\mu_{t_m}^{n,\tau})\right)\right]\bigg|\\\leq C|y_{t_m}^{i,\tau}-x_{t_m}^{i,\tau}|^2+C\frac{1}{n}\sum_{j=1}^n |y_{t_m}^{j,\tau}-x_{t_m}^{j,\tau}|^2.
\end{multline}
Condition~\ref{cond:theta_velocity_field} and Lemma~\ref{lem:discrete_a_priori} imply that
\begin{equation}\label{eq:sgd_conv_3}
\left(\fint_{\S^{d-1}}u(\theta,y_{t_m}^{i,\tau}\cdot\theta,\nu_{t_m}^{n,\tau})\,d\theta-u(\theta_m,x_{t_m}^{i,\tau}\cdot\theta_m,\mu_{t_m}^{n,\tau})\right)^2\leq C.
\end{equation}
Averaging over $i$ and combining~\eqref{eq:sgd_conv_1}-\eqref{eq:sgd_conv_3} we have thus found that
\[\sup_{t\in[t_m,t_{m+1}]}\mathbb{E}_m\left[\frac{1}{n}\sum_{i=1}^n|y_{t}^{i,\tau}-x_{t}^{i,\tau}|^2\right]\leq (1+\tau C)\frac{1}{n}\sum_{i=1}^n|y_{t_m}^{i,\tau}-x_{t_m}^{i,\tau}|^2+C\tau^2.\]
Iterating this bound, we have
\[\sup_{t\in[0,T]}\mathbb{E}\left[\frac{1}{n}\sum_{i=1}^n|y_{t_m}^{i,\tau}-x_{t_m}^{i,\tau}|^2\right]\leq (1+\tau C)^{T/\tau}\tau\leq C\tau,\]
thus we have shown the desired claim.
\end{proof}

{\small
\bibliographystyle{alpha}
\bibliography{references1,references2}
}

\end{document}